\documentclass{article}

\usepackage{microtype}
\usepackage{graphicx}
\usepackage{subfigure}
\usepackage{amssymb, amsmath, amsthm,amsbsy,bm}
\usepackage{booktabs} % for professional tables

\usepackage{todonotes}

\usepackage{hyperref}

\usepackage{hyperref}

\newcommand{\wg}[1]{{\color{black}#1}}

\definecolor{olivegreen}{rgb}{0.203, 0.450, 0.047}

\usepackage[accepted]{icml2021}

\icmltitlerunning{Active Slices for Sliced Stein Discrepancy }

\newtheorem{definition}{Definition}[section]
\newtheorem{theorem}{Theorem}
\newtheorem{prop}{Proposition}
\newtheorem{corollary}{Corollary}[theorem]
\newtheorem{lemma}{Lemma}

\begin{document}

\twocolumn[
\icmltitle{Active Slices for Sliced Stein Discrepancy}

% It is OKAY to include author information, even for blind
% submissions: the style file will automatically remove it for you
% unless you've provided the [accepted] option to the icml2021
% package.

% List of affiliations: The first argument should be a (short)
% identifier you will use later to specify author affiliations
% Academic affiliations should list Department, University, City, Region, Country
% Industry affiliations should list Company, City, Region, Country

% You can specify symbols, otherwise they are numbered in order.
% Ideally, you should not use this facility. Affiliations will be numbered
% in order of appearance and this is the preferred way.
%\icmlsetsymbol{equal}{*}

\begin{icmlauthorlist}
\icmlauthor{Wenbo Gong}{EngCAM}
\icmlauthor{Kaibo Zhang}{EngCAM}
\icmlauthor{Yingzhen Li}{IC}
\icmlauthor{Jos\'e Miguel Hern\'andez-Lobato}{EngCAM}
\end{icmlauthorlist}

\icmlaffiliation{EngCAM}{Department of Engineering, University of Cambridge, Cambridge, United Kingdom}
% \icmlaffiliation{MathCAM}{Department of Mathematics, University of Cambridge, Cambridge, United Kingdom}
\icmlaffiliation{IC}{Department of Computing, Imperial College London, London, United Kingdom}

\icmlcorrespondingauthor{Wenbo Gong}{wg242@cam.ac.uk}
\icmlcorrespondingauthor{Jos\'e Miguel Hern\'andez-Lobato}{jmh233@cam.ac.uk}

% You may provide any keywords that you
% find helpful for describing your paper; these are used to populate
% the "keywords" metadata in the PDF but will not be shown in the document
\icmlkeywords{Machine Learning, ICML}

\vskip 0.3in
]
\printAffiliationsAndNotice{}

\begin{abstract}

Sliced Stein discrepancy (SSD) and its kernelized variants have demonstrated promising successes in goodness-of-fit tests and model learning in high dimensions. Despite their theoretical elegance, their empirical performance depends crucially on the search of optimal slicing directions to discriminate between two distributions. Unfortunately, previous gradient-based optimisation approaches for this task return sub-optimal results: they are computationally expensive, sensitive to initialization, and they lack theoretical guarantees for convergence. We address these issues in two steps. First, we provide theoretical results stating that the requirement of using optimal slicing directions in the kernelized version of SSD can be relaxed, validating the resulting discrepancy with \emph{finite} random slicing directions. Second, given that good slicing directions are crucial for practical performance, we propose a fast algorithm for finding such slicing directions based on ideas of active sub-space construction and spectral decomposition. Experiments on goodness-of-fit tests and model learning show that our approach achieves both improved performance and faster convergence. Especially, we demonstrate a 14-80x speed-up in goodness-of-fit tests when comparing with gradient-based alternatives.

\end{abstract}

\section{Introduction}
\label{sec: Introduction}
% Discrepancy measures between two distributions have been the cornerstones in modern statistical machine learning. Among many possible choices, \textit{Stein discrepancy} (SD) \citep{gorham2015sd}, a special case of \textit{integral probability metric} (IPM)\citep{sriperumbudur2009integral}, requires finding optimal test functions within a given function family. Its optimality is analytic when the function family is give by a \textit{reproducing kernel Hilbert space} (RKHS). In this case it is called kernelized Stein discrepancy (KSD). Due to its special form that only requires the samples from one distribution and the score (gradient of log probability) from the other, it has been extensively used for \textit{goodness-of-fit} (GOF) testing \citep{liu2016kernelized,chwialkowski2016kernel,huggins2018random,jitkrittum2017linear,gorham2017measuring} and model learning \citep{hu2018stein,grathwohl2020cutting,liu2016stein,pu2017vae}.

% Despite its theoretical elegance, \textit{KSD} often suffers from the curse-of-dimensionality, which significantly deteriorates its GOF testing power \citep{chwialkowski2016kernel,huggins2018random} and causes mode collapse when \textit{KSD} is used for particle-based inference \citep{liu2016stein,zhuo2018message,chen2019projected}.

Discrepancy measures between two distributions are critical tools in modern statistical machine learning. Among them,  \textit{Stein discrepancy} (SD) and its kernelized version, kernelized Stein discrepancy (KSD), have been extensively used for \textit{goodness-of-fit} (GOF) testing \citep{liu2016kernelized,chwialkowski2016kernel,huggins2018random,jitkrittum2017linear,gorham2017measuring} and model learning \citep{liu2016stein,pu2017vae,hu2018stein,grathwohl2020cutting}. Despite their recent success, applications of Stein discrepancies to high-dimensional distribution testing and learning remains an unsolved challenge. 

These ``curse of dimensionality'' issues have been recently addressed by the newly proposed Sliced Stein discrepancy (SSD) and its kernelized variants SKSD \cite{gong2021sliced}, which have demonstrated promising results in both high dimensional GOF tests and model learning. They work by first projecting the score function and the test inputs across two slice directions $\bm{r}$ and $\bm{g}_r$ and then comparing the two distributions using the resulting one dimensional slices. The performance of SSD and SKSD crucially depends on choosing  slicing directions that are highly discriminative. Indeed, \citet{gong2021sliced} showed that such discrepancy can still be valid 
despite the information loss caused by the projections,
if \emph{optimal slices} -- directions along which the two distributions differ the most -- are used.
% \citet{gong2021sliced} proposed a novel discrepancy family called \textit{sliced Stein discrepancy} (SSD), which works by projecting the score function and the available samples along one-dimensional slices defined by directions $\bm{r}$ and $\bm{g}_r$, respectively. Analogous to \textit{KSD}, the optimal test function in SSD is tractable when using the kernel trick, resulting in the \textit{sliced kernelized Stein discrepancy} (SKSD).
% Therefore one must resort to using \emph{finite} number of slices; in such case choosing \emph{optimal} slices is critical to both the validity of the resulting divergence (maxSKSD) and practical performances. 
%that maximize the SKSD (called maxSKSD).
%Despite \textit{maxSKSD} is often preferred in practice, one often has to seek the help from gradient-based optimization to search optimal slices. 
%\textit{maxSKSD} is implemented using gradient-based optimization of the slices. 
Unfortunately, gradient-based optimization for searching such optimal slices often suffers from slow convergence and sub-optimal solutions. In practice, many gradient updates may be required to obtain a reasonable set of slice directions \citep{gong2021sliced}. 

% For example, \textit{maxSKSD} requires to use a \todo{Not clear to the reader what this is}{diversified initialization} and to run the optimization process for nearly $1200$ epochs when using it in simple GOF testing benchmarks \citep{gong2021sliced}. The diversified initialization results in a $O(M^2)$ memory and time cost, where $M$ is the number of \todo{verify that this is correct}{slice directions used} and often equals the problem dimension.
We aim to tackle the above practical challenges by proposing an efficient algorithm to find \emph{good} slice directions with statistical guarantees. Our contributions are as follows:

\vspace{-1em}
\begin{itemize}
\setlength\itemsep{0em}
    %\item \todo{The reader has not yet been guided to understand why this is a relevant contribution.}{We relax the restrictive constraint of having to use optimal slices in \textit{SKSD}
    %and show that the \textit{SKSD} validity can be preserved even with finite randomly-drawn slice directions.}
    
    \item We propose a computationally efficient variant of SKSD using a \emph{finite number of random slices}. This relaxes the restrictive constraint of having to use  \emph{optimal} slices, with the consequence that convergence during optimisation to a global optimum is no longer required.
    
    %\item \todo{Same as in the previous point}{We theoretically show the relationship between \textit{SKSD} and its non-kernelized version \textit{SSD}.}
    \item Given that \emph{good} slices are still preferred in practice, we propose surrogate optimization tasks to find such directions. These are called \emph{active slices} and have analytic solutions that can be computed very efficiently. 

    \item Experiments on GOF test benchmarks (including testing on restricted Boltzmann machines)
    show that our algorithm outperforms alternative gradient-based approaches 
    while achieving at least a 14x speed-up.
    
    \item In the task of learning high dimensional \textit{independent component analysis} (ICA) models
    \citep{comon1994independent}, our algorithm converges much faster and to significantly better solutions than other baselines.
\end{itemize}
\paragraph{Road map:} First, we give a brief background for \textit{SD}, \textit{SKSD} and its relevant variants (Section \ref{sec: Background}).
% Next, we show how to relax the previously required optimality constraints to using finite random slice directions (Section \ref{subsec: validity of active slices}). 
\wg{Next, we show that the optimality of slices are not necessary. Instead, finite random slices are enough to ensure the validity of SKSD (\ref{subsec: validity of active slices}).} \wg{Despite that relaxing the optimality constraint gives us huge freedom to select slice directions, choosing an appropriate objective for finding slices is still crucial. Unfortunately, analysing SKSD in RKHS is challenging.}
\wg{We thus propose to analyse SSD as a surrogate objective by showing SKSD can be well approximated by SSD (Section \ref{subsec: Relationship between SSD and SKSD})}. Lastly, by analyzing SSD, we propose algorithms to find active slices for SKSD (Sections \ref{sec: Active slice direction g}, \ref{sec: active slice direction r}, \ref{sec: In practice}), and demonstrate the efficacy of our proposal in the experiments (Section \ref{sec: Experiments}). Assumptions and proofs of  theoretical results as well as the experimental settings can be found in the appendix. 

% \begin{figure}[t]
%     \centering
%     \includegraphics[scale=0.9]{Methodology/Practice/Divergence_Relationship.pdf}
%     \caption{The relationship between different SSD discrepancies, where \textcolor{olivegreen}{green texts indicates our contributions}, \textcolor{red}{red symbols indicates valid discrepancies}, $\mathcal{H}_{rg_r}$ is the RKHS induced by kernel $k_{rg_r}$.}
%     \label{fig: divergence relationship}
% \end{figure}
% mentioning KSD and its problems
% mentione SKSD, advantages and its problems
% contributions

\section{Background}
\label{sec: Background}
% \yl{TODO: make $f$ boldface, fix reference to appendix}

For a distributions $p$ on $\mathcal{X} \subset \mathbb{R}^D$ with differentiable density, we define its score function as $\pmb{s}_p(\pmb{x}) = \nabla_{\pmb{x}} \log p(\pmb{x})$. We also define the Stein operator $\mathcal{A}_p$ for distribution $p$ as 
\begin{equation}
    \mathcal{A}_p \bm{f}(\pmb{x}) = \pmb{s}_p(\pmb{x})^T \bm{f}(\pmb{x}) + \nabla_{\pmb{x}}^T \bm{f}(\pmb{x})  \,,
    \label{eq: Stein operator}
\end{equation}
where $\bm{f}: \mathcal{X} \rightarrow \mathbb{R}^D$ is a test function. 
%
% Now let $\mathcal{X} = \mathbb{R}^D$. The \textit{Stein class} $\mathcal{F}_q$ of $q$ is defined as the set of test function $f_0$ satisfying the \textit{Stein identity} \citep{stein1972bound, liu2016kernelized} $\mathbb{E}_q[\mathcal{A}_q \bm{f}(\pmb{x})] = 0$. It can be generalised when $\bm{f}$ is a vector function by insisting all components of $\bm{f}$ satisfy the \textit{Stein identity}. Hence a valid discrepancy called 
Then the \textit{Stein discrepancy} (SD) \citep{gorham2015sd} between two distributions $p, q$ with differentiable densities on $\mathcal{X}$ is 
\begin{equation}
    \begin{split}
        D_{SD}(q,p) &= \sup_{\bm{f} \in \mathcal{F}_q} \mathbb{E}_q[\mathcal{A}_p \bm{f}(\pmb{x})]\,,
    \end{split}
    \label{eq: SD}
\end{equation}
where $\mathcal{F}_q$ is the Stein's class of $q$ that contains test functions satisfying $\mathbb{E}_q[\mathcal{A}_q \bm{f}(\pmb{x})] = 0$ (also see Definition \ref{def: Stein class} in appendix \textcolor{red}{\ref{appendix: definition and assumption}}).
The supremum can be obtained by choosing $\bm{f^*} \propto \pmb{s}_p(\pmb{x}) - \pmb{s}_q(\pmb{x})$ if $\mathcal{F}_q$ is rich \citep{hu2018stein}.

\citet{chwialkowski2016kernel, liu2016kernelized} further restricts the test function space $\mathcal{F}_q$ to be a unit ball in an RKHS induced by a $c_0-$universal kernel $k:\mathcal{X}\times\mathcal{X}\rightarrow \mathbb{R}$. This results in the 
% By writing 
% $u_p(\bm{x},\bm{x}')=s_p(\bm{x})^T k(\bm{x},\bm{x}')s_p(\bm{x'})
%       +s_p(\bm{x})^T \nabla_{\bm{x}'}k(\bm{x},\bm{x}')
%       +s_p(\bm{x}')^T \nabla_{\bm{x}}k(\bm{x},\bm{x}')
%       +\operatorname{Tr}(\nabla_{\bm{x}, \bm{x}'} k(\bm{x}, \bm{x}'))$,
\textit{kernelized Stein discrepancy} (KSD), \wg{which can be computed analytically}:
\begin{equation}
    \begin{split}
        D^2(q,p)
        &= \left( \sup_{\bm{f} \in \mathcal{H}_k, ||f||_{\mathcal{H}_k} \leq 1} \mathbb{E}_q[\mathcal{A}_p \bm{f}(\pmb{x})] \right)^2 \\
        &= ||\mathbb{E}_{q} [\bm{s}_p(\bm{x})k(\bm{x},\cdot)+\nabla_{\bm{x}}k(\bm{x},\cdot)]||^2_{\mathcal{H}_k}\,,
    \end{split}
\end{equation}
where $\mathcal{H}_k$ is the $k$ induced RKHS with norm $||\cdot||_{\mathcal{H}_k}$.

\subsection{Sliced kernelized Stein discrepancy}
Despite the theoretical elegance of \textit{KSD}, it often suffers from the curse-of-dimensionality in practice. To address this issue, \citet{gong2021sliced} proposed a divergence family called \textit{sliced Stein discrepancy} (SSD) and its kernelized variants, \wg{under mild assumptions on the regularity of probability densities (\textbf{Assumptions 1-4} in appendix \ref{appendix: definition and assumption}) and the richness of the kernel (\textbf{Assumption 5} in appendix \ref{appendix: definition and assumption})}. The key idea is to compare the distributions on their one dimensional slices by projecting the score $\bm{s}_p$ and test input $\bm{x}$ with two directions $\bm{r}$ and its corresponding $\bm{g}_r$, respectively. Readers are referred to appendix \ref{appendix: Detailed Background} for details. \wg{Despite that one cannot access all the information possessed by $\bm{s}_p$ and $\bm{x}$ due to the projections, the validity of the discrepancy can be ensured by using an orthogonal basis for $\bm{r}$ along with the corresponding most discriminative $\bm{g}_r$ directions.} 
% The loss of information due to this projection operation can be mitigated by using an orthogonal basis for $\bm{r}$ along with the corresponding $\bm{g}_r$ directions with largest discriminative power \citep{gong2021sliced}.
% but computing SSD requires integrating out all possible slicing directions for both $\bm{r}$ and $\bm{g}_r$ (Integrated SSD) which is computationally intractable.
The resulting valid discrepancy is called \textit{maxSSD-g}, which uses a set of orthogonal basis $\bm{r}\in O_r$ and their corresponding optimal $\bm{g}_r$ directions:
\begin{equation}
\begin{split}
    S_{\text{max}_{g_r}}(q,p)&=\sum_{\bm{r}\in O_r}{\sup_{\substack{h_{rg_r}\in\mathcal{F}_q\\\bm{g}_r\in\mathbb{S}^{D-1}}}{\mathbb{E}_q[{s}^r_p(\bm{x})h_{rg_r}(\bm{x}^T\bm{g}_r)+}}\\
    &{{\bm{r}^T\bm{g}_r\nabla_{\bm{x}^T\bm{g}_r}h_{rg_r}(\bm{x}^T\bm{g}_r)]}}\,,
    \label{eq: maxSSD-g}
\end{split}
\end{equation}
where $h_{rg_r}:\mathcal{K}\subseteq\mathbb{R}\rightarrow\mathbb{R}$ is the test function,  $\mathbb{S}^{D-1}$ is the $D$-dimensional unit sphere and $s_p^r(\bm{x})=\bm{s}_p(\bm{x})^T\bm{r}$ is the projected score function. Under certain scenarios \citep{gong2021sliced}, i.e. GOF test, one can further improve the performance of \textit{maxSSD-g} by replacing 
$\sum_{\bm{r}\in O_r}$ with the optimal $\sup_{\bm{r}\in\mathbb{S}^{D-1}}$ in Eq.\ref{eq: maxSSD-g}, resulting in another variant called \textit{maxSSD-rg} ($S_{\text{max}_{rg_r}}$). This increment in performance is due to the higher discriminative power provided by the optimal $\bm{r}$.  
% Other valid member of the \textit{SSD} family include \textit{orthogonal SSD} (Eq.\ref{eq: orthogonal SSD} in Appendix \textcolor{red}{\ref{appendix: Detailed Background}}) which integrates $\bm{g}_r$ with distribution $p_{g}$ inside $S_{\text{max}_{g_r}}$ instead of $\sup_{\bm{g}_r}$.

However, the optimal test functions $h^*_{rg_r}$ in \emph{maxSSD-g} (or \emph{-rg}) are intractable in practice. \citet{gong2021sliced} further proposed kernelized variants to address this issue by letting $\mathcal{F}_q$ to be in a unit ball of an RKHS induced by a $c_0-$universal kernel $k_{rg_r}$. With
\begin{equation}
\begin{split}
    \xi_{p,r,g_r}(\bm{x},\cdot)&=s_p^r(\bm{x})k_{rg_r}(\bm{x}^T\bm{g}_r,\cdot)+\\
    &\quad\bm{r}^T\bm{g}_r\nabla_{\bm{x}^T\bm{g}_r}k_{rg_r}(\bm{x}^T\bm{g}_r,\cdot)\,,
\end{split}
    \label{eq: SKSD test function}
\end{equation}
the \textit{maxSKSD-g} (the kernelized version of maxSSD-g) is 
\begin{equation}
    SK_{\text{max}_{g_r}}(q,p)=\sum_{\bm{r}\in O_r}{\sup_{\bm{g}_r\in\mathbb{S}^{D-1}}{||\mathbb{E}_q[\xi_{p,r,g_r}(\bm{x})]||^2_{\mathcal{H}_{rg_r}}}}\,,
    \label{eq: maxSKSD-g}
\end{equation}
where $\mathcal{H}_{rg_r}$ is the RKHS induced by $k_{rg_r}$ with the associated norm $||\cdot||_{\mathcal{H}_{rg_r}}$. Similarly, a kernelized version of maxSSD-rg, denoted by \textit{maxSKSD-rg} ($SK_{\text{max}_{rg_r}}$), is obtained by replacing 
$\sum_{\bm{r}\in O_r}$ with $\sup_{\bm{r}\in\mathbb{S}^{D-1}}$ in Eq.\ref{eq: maxSKSD-g}.

Despite that \emph{maxSKSD-g} (or \emph{-rg}) addresses the tractability of test functions, the practical challenge of computing them is the computation of the optimal slice directions $\bm{r}$ and $\bm{g}_r$. Gradient-based optimization \citep{gong2021sliced} for such computation suffers from slow convergence; even worse, it is sensitive to initialization and returns sub-optimal solutions only. In such case, it is unclear whether the resulting discrepancy is still valid, making the correctness of GOF test unverified. Therefore, the first important question to ask is: are the optimality of slices a necessary condition for the validity of \emph{maxSKSD-g} (or \emph{-rg})? Remarkably, we show that the answer is \textbf{No} with mild assumptions on the kernel (\textbf{Assumption 5-6} in appendix \ref{appendix: definition and assumption}). 

\wg{As the sliced Stein discrepancy defined previously involves a $\sup$ operator, making them difficult to analyze, we need to define notations for their ``sub-optimal" versions. For example, \emph{maxSSD-g} (Eq.\ref{eq: maxSSD-g}) involves a $\sup$ operator over slices $\bm{g}_r$. We thus define \emph{SSD-g} ($S_{g_r}$) as Eq.\ref{eq: maxSSD-g} with a given $\bm{g}_r$ instead of the $\sup$:
\begin{equation}
\begin{split}
        S_{g_r} &= \sum_{\bm{r}\in O_r}\sup_{h_{rg_r}\in\mathcal{F}_q}\mathbb{E}_{q}[s_p^r(\bm{x})h_{rg_r}(\bm{x}^T\bm{g}_r)+\\
        &\bm{r}^T\bm{g}_r\nabla_{\bm{x}^T\bm{g}_r}h_{rg_r}(\bm{x}^T\bm{g}_r)]
\end{split}
\label{eq: SSD-g}
\end{equation}
Following similar logic, we define the ``sub-optimal" version for each of the discrepancy mentioned in this section as table \ref{tab: notations} and appendix \ref{appendix: Notations}.
\begin{table*}[!h]
\centering
\caption{Notations for ``sub-optimal" versions of SSD \& SKSD.  }
\label{tab: notations}
\begin{tabular}{l|llll}
\toprule
Optimal form & \emph{maxSSD-g} ($S_{\text{max}_{g_r}}$)                             & \emph{maxSSD-rg} ($S_{\text{max}_{rg_r}}$)                                          & \emph{maxSKSD-g} ($SK_{\text{max}_{g_r}}$) & \emph{maxSKSD-rg} ($SK_{\text{max}_{rg_r}}$) \\
\midrule
Modifications            & Change $\sup_{\bm{g}_r}$  & Change $\sup_{\bm{r},\bm{g_r}}$ to given & Same as \emph{maxSSD-g}.             & Same as \emph{maxSSD-rg}             \\
&to given $\bm{g}_r$ in Eq.\ref{eq: maxSSD-g}& $\bm{r}$, $\bm{g}_r$ in Eq.\ref{eq: maxSSD-rg} (App. \ref{appendix: Detailed Background})&in Eq.\ref{eq: maxSKSD-g}&in Eq.\ref{eq: maxSKSD-rg} (App. \ref{appendix: Detailed Background}) \\
\midrule
``sub-optimum"  & \emph{SSD-g} ($S_{g_r}$)                                            & \emph{SSD-rg} ($S_{rg_r}$)                                                          & \emph{SKSD-g} ($SK_{g_r}$)                 & \emph{SKSD-rg} ($SK_{rg_r}$)                \\ \bottomrule
\end{tabular}
\vspace{-1em}
\end{table*}}
\begin{figure}[t]
    \centering
    \includegraphics[scale=0.9]{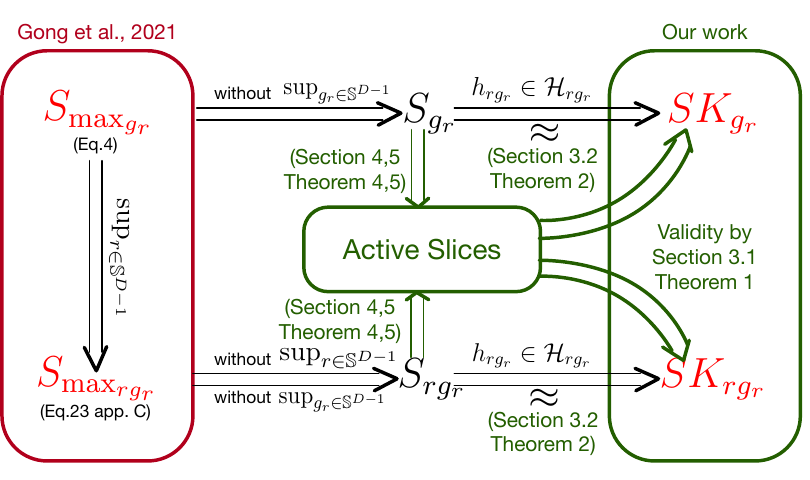}
    \caption{The relationship between different SSD discrepancies, where \textcolor{olivegreen}{green texts indicate our contributions}, \textcolor{red}{red symbols indicate valid discrepancies} and $\mathcal{H}_{rg_r}$ is the RKHS induced by kernel $k_{rg_r}$. The \textcolor{red}{leftmost part} are the discrepancies proposed by \citet{gong2021sliced}, whereas the \textcolor{olivegreen}{rightmost part + central ``Active Slices"} are our contributions. The arrows $\Rightarrow$ indicate   the connections between \citet{gong2021sliced} and our work. }
    \label{fig: divergence relationship}
\end{figure}

% \section{Methodology}
% \input{Methodology/Practice/sim_evalues}
\section{Relaxing constraints for the SKSD family}
\label{sec: relaxing constraints for kernelized SSD family}
% In order to find good slices, one should consider the following two questions: (i) What are the requirements for the slice directions to define a valid discrepancy? (ii) Within those slice directions supporting a valid discrepancy, how to find good ones in practice? We address these two questions in the following, with all the proofs provided in appendix \ref{appendix: validity of maxSKSD with active slices}.
% \subsection{Assumptions}
% \label{subsec: Assumptions}
% \wg{Before we jump into the theorems, we briefly explain the main assumptions used in the following. The details of the assumptions can be found in appendix \ref{appendix: definition and assumption}. \textbf{Assumption 1-4} regularise the properties of the probability densities and test functions, which are needed for defining valid \emph{sliced Stein discrepancy} family. \textbf{Assumption 5} further restricts the probability density to be log concave, which is an essential requirement for an important inequality used in the following: Poincar\'e inequality. \textbf{Assumption 6} is assumed to define a valid Poincar\'e inequality. \textbf{Assumption 7} ensures the richness of the RKHS induced by the kernel. \textbf{Assumption 8} further regularizes the kernel to be real-analytic, where we can use the nice properties of real-analytic functions to relax the optimality constraints of the slice directions.} 

% \yl{consider explaining assumptions in place where you first introduce them.}

\subsection{Is optimality necessary for validity?}
\label{subsec: validity of active slices}
% As mentioned, the validity of \textit{SKSD} variants require either integrating all possible directions of $\bm{g}_r$ with integrated or orthogonal $\bm{r}\in O_r$ (\textit{integrated} or \textit{orthogonal SKSD}), or using the optimal slices ($SK_{\text{max}_{g_r}}$ and $SK_{\text{max}_{rg_r}}$).
As mentioned before, the discrepancy validity of {max SKSD} requires the optimality of slice directions, which restricts its application in practice. 
In the following, we show that these restrictions can be much relaxed with mild assumptions on the kernel. All proofs can be found in Appendix \ref{appendix: validity of maxSKSD with active slices}.
% \paragraph{Key Idea:} It is trivial that when $p=q$, the $SK_{{rg_r}}=0$ with given $\bm{r}$ and $\bm{g}_r$. Now we need to make sure when $p\neq q$, $SK_{{rg_r}}>0$ for slices $\bm{r}$ and $\bm{g}_r$. Inspired by the work about linear-time test\citep{chwialkowski2015fast,jitkrittum2016interpretable,jitkrittum2016adaptive,jitkrittum2017linear}, if we could show that $SK_{{rg_r}}$ is real analytic w.r.t. both $\bm{r}$ and $\bm{g}_r$. Then, we can use a nice property about real analytic function \citep{mityagin2015zero}: if a real analytic function $SK_{{rg_r}}$ is not $0$ identically, then the set of roots $\{(\bm{g}_r,\bm{r})|SK_{{rg_r}}=0\}$ has measure $0$. This means we can draw $\bm{g}_r,\bm{r}$ from some distribution $\eta$, and if $p\neq q$, then $SK_{{rg_r}}>0$ almost $\eta$ surely. 
% % Although the above holds with assumption that $\bm{r}$, $\bm{g}_r$ are in the connected open set, one can easily generalize it to $\mathbb{S}^{D-1}$ because every vector $\bm{v}$ can be written as $\bm{v}=c\bm{v}'$ where $c$ is the normalizing constant and $\bm{v}'\in\mathbb{S}^{D-1}$. 
% As we already know when $p=q$, $SK_{rg_r}=0$ for all possible $\bm{r}$, $\bm{g}_r$, we thus only need to show when $p\neq q$, $SK_{rg_r}>0$ for discrepancy validity.

The key idea is to use kernels such that the corresponding term $SK_{rg_r}$ is \emph{real analytic} w.r.t.~both $\bm{r}$ and $\bm{g}_r$, \wg{which is detailed by \textbf{Assumption 6} (Appendix \ref{appendix: definition and assumption}).} A nice property of any real analytic function is that, unless it is a constant function, otherwise the set of its roots has zero Lebesgue measure. This means the possible valid slices are almost everywhere in $\mathbb{R}^D$, giving us huge freedom to choose slices without worrying about violating validity.
\begin{theorem}[Conditions for valid slices]
Assuming assumptions {\hyperref[assumption 1]{1-4} (density regularity), \hyperref[assumption 5]{5} (richness of kernel) and \hyperref[assumption 6]{6} (real analytic kernel) in Appendix \ref{appendix: definition and assumption}}, let {$\bm{g}_r\sim \eta_g$} for each $\bm{r}\sim \eta_r$, where $\eta_g$, $\eta_r$ are distributions on $\mathbb{R}^D$ with a density, then {$SK_{{rg_r}}(q,p)=0$} iff.~$p=q$ almost surely.
\label{thm: validity of maxSKSD}
\end{theorem}
The above theorem tells us that a \emph{finite} number of \emph{random} slices is enough to make $SK_{rg_r}$ valid without the need of using optimal slices (c.f.~$SK_{\text{max}_{rg_r}}$).
In practice, we often consider $\bm{r},\bm{g}_r\in\mathbb{S}^{D-1}$ instead of $\mathbb{R}^D$. Fortunately, one can easily transform arbitrary slices to $\mathbb{S}^{D-1}$ without violating the validity. For any $\bm{r},\bm{g}_r$, we (i) add Gaussian noises to them, and (2) re-normalize the noisy $\bm{r},\bm{g}_r$ to unit vectors. We refer to corollary \ref{coro: validity of maxSKSD normalized} in appendix \ref{appendix subsec: validity w.r.t g_r} for details.

\subsection{Relationship between SSD and SKSD}
\label{subsec: Relationship between SSD and SKSD}
Theorem \ref{thm: validity of maxSKSD} allows us to use random slices. However, it is still beneficial to find good ones in practice. Unfortunately, $SK_{rg_r}$ is not a suitable objective for finding good slice directions. This is because, unlike the test function in a general function space ($h_{rg_r}\in\mathcal{F}_q$), the optimal kernel test function ($\mathbb{E}_q[\xi_{p,r,g_r}(\bm{x},\cdot)]$) can not be easily {analyzed for finding good slices due to its restriction in RKHS}. 

Instead, we propose to use $S_{g}$ (or $S_{rg_r}$) as the optimization objective. To justify $S_{rg_r}$ as a good replacement for $SK_{g_r}$,
%
% Before we jump into the details, the supremum of $h_{rg_r}$ in $S_{rg_r}$ hinders the further analysis. The following proposition analytically gives the optimal form of $h_{rg_r}$.
% \begin{prop}[Optimal test function given $\boldsymbol{r}, \boldsymbol{g}_r$]
% Assume assumptions 1 - 4 and given directions $\boldsymbol{r}, \boldsymbol{g}_r$. Assume an arbitrary orthogonal matrix $\bm{G}_r=$ $\left[\boldsymbol{a}_{1}, \ldots, \boldsymbol{a}_{D}\right]^{T}$ where $\boldsymbol{a}_{i} \in \mathbb{S}^{D \times 1}$ and $\boldsymbol{a}_{d}=\boldsymbol{g}_r .$ Denote $\boldsymbol{x} \sim q$ and $\boldsymbol{y}=\boldsymbol{G}_r \boldsymbol{x}$ which is also a random variable with the induced distribution $q_{G_r}$. Then, the optimal test function for $S_{rg_r}$ is 
% \begin{equation}
% h_{rg_r}^{*}\left(\boldsymbol{x}^{T} \boldsymbol{g}_r\right) \propto \mathbb{E}_{q_ {G_r}\left(\boldsymbol{y}_{-d} \mid y_{d}\right)}\left[\left(s_{p}^{r}\left(\boldsymbol{G}_r^{-1} \boldsymbol{y}\right)-s_{q}^{r}\left(\boldsymbol{G}_r^{-1} \boldsymbol{y}\right)\right)\right]    
% \label{eq: optimal test function maxSSD-g}
% \end{equation}
% where $y_{d}=\boldsymbol{x}^{T} \boldsymbol{g}_r$ and $\boldsymbol{y}_{-d}$ contains other $\boldsymbol{y}$ elements.
% \label{prop: optimality of maxSSD-g}
% \end{prop}
%
we show that $SK_{rg_r}$ approximates $S_{rg_r}$ arbitrarily well if the corresponding RKHS of the chosen kernel is dense in continuous function space. Similar results for $SK_g \approx S_g$ can be derived accordingly as the only difference between $S_{g_r}$ and $S_{rg_r}$ is the summation over orthogonal basis $O_r$. \wg{However, $S_{rg_r}$ still involves a $\sup$ operator over test functions $h_{rg_r}$, which hinders further analysis. To deal with this, we give an important proposition that are needed in almost every theoretical claims we made. This proposition characterises the optimal test functions for $S_{rg_r}$ (or $S_{g_r}$).}
\wg{\begin{prop}[Optimal test function given $\boldsymbol{r}, \boldsymbol{g}_r$]
Assume assumptions \hyperref[assumption 1]{1-4} (density regularity) and given directions $\boldsymbol{r}, \boldsymbol{g}_r$. Assume an arbitrary orthogonal matrix $\bm{G}_r=$ $\left[\boldsymbol{a}_{1}, \ldots, \boldsymbol{a}_{D}\right]^{T}$ where $\boldsymbol{a}_{i} \in \mathbb{S}^{D \times 1}$ and $\boldsymbol{a}_{d}=\boldsymbol{g}_r .$ Denote $\boldsymbol{x} \sim q$ and $\boldsymbol{y}=\boldsymbol{G}_r \boldsymbol{x}$ which is also a random variable with the induced distribution $q_{G_r}$. Then, the optimal test function for $S_{rg_r}$ is 
\begin{equation}
h_{rg_r}^{*}\left(\boldsymbol{x}^{T} \boldsymbol{g}_r\right) \propto \mathbb{E}_{q_ {G_r}\left(\boldsymbol{y}_{-d} \mid y_{d}\right)}\left[\left(s_{p}^{r}\left(\boldsymbol{G}_r^{-1} \boldsymbol{y}\right)-s_{q}^{r}\left(\boldsymbol{G}_r^{-1} \boldsymbol{y}\right)\right)\right]    
\label{eq: optimal test function maxSSD-g}
\end{equation}
where $y_{d}=\boldsymbol{x}^{T} \boldsymbol{g}_r$ and $\boldsymbol{y}_{-d}$ contains other $\boldsymbol{y}$ elements.
\label{prop: optimality of maxSSD-g}
\end{prop}}
\wg{Intuitively, assume $\bm{G}_r$ is a rotation matrix. Then $h_{rg_r}^*$ is the conditional expected score difference between two rotated $p$ and $q$. This form is very similar to the optimal test function for SD, which is just the score difference between the original $p$, $q$.}
% \begin{prop}[Optimal test function given $\boldsymbol{r}, \boldsymbol{g}_r$]
% Assume assumptions 1 - 4 and given directions $\boldsymbol{r}, \boldsymbol{g}_r$. Assume an arbitrary orthogonal matrix $\bm{G}_r=$ $\left[\boldsymbol{a}_{1}, \ldots, \boldsymbol{a}_{D}\right]^{T}$ where $\boldsymbol{a}_{i} \in \mathbb{S}^{D \times 1}$ and $\boldsymbol{a}_{d}=\boldsymbol{g}_r .$ Denote $\boldsymbol{x} \sim q$ and $\boldsymbol{y}=\boldsymbol{G}_r \boldsymbol{x}$ which is also a random variable with the induced distribution $q_{G_r}$. Then, the optimal test function for $S_{rg_r}$ is 
% \begin{equation}
% h_{rg_r}^{*}\left(\boldsymbol{x}^{T} \boldsymbol{g}_r\right) \propto \mathbb{E}_{q_ {G_r}\left(\boldsymbol{y}_{-d} \mid y_{d}\right)}\left[\left(s_{p}^{r}\left(\boldsymbol{G}_r^{-1} \boldsymbol{y}\right)-s_{q}^{r}\left(\boldsymbol{G}_r^{-1} \boldsymbol{y}\right)\right)\right]    
% \label{eq: optimal test function maxSSD-g}
% \end{equation}
% where $y_{d}=\boldsymbol{x}^{T} \boldsymbol{g}_r$ and $\boldsymbol{y}_{-d}$ contains other $\boldsymbol{y}$ elements.
% \label{prop: optimality of maxSSD-g}
% \end{prop}
%Now, we can show $SK_{rg_r}$ well approximates $S_{rg_r}$. 
\wg{Knowing the optimal form of $h_{rg_r}^*$, we can show $SK_{rg_r}$ can be well approximated by $S_{rg_r}$.}
\begin{theorem}[$SK_{rg_r} \approx S_{rg_r}$]
Assume assumptions \hyperref[assumption 1]{1-4} (density regularity) and \hyperref[assumption 1]{5} (richness of kernel). Given $\bm{r}$ and $\bm{g}_r$, $\forall \epsilon>0$ there exists a constant $C$ such that 
\[
0\leq S_{{rg_r}}-SK_{{rg_r}}<C\epsilon\,.
\]
\label{thm: maxSKSD approximates maxSSD}
\end{theorem}
% The proof involves the optimal test function $h^*_{rg_r}$ for $S_{rg_r}$, whose analytic form is in Proposition \ref{prop: optimality of maxSSD-g} of appendix \ref{appendix subsec: relationship SSD and SKSD}.

As $S_{rg_r}$ approximates $SK_{rg_r}$ arbitrary well, the hope is that good slices for $S_{rg_r}$ also correspond to good slices for $SK_{rg_r}$ in practice. Therefore in the next section we focus on analyzing $S_{rg_r}$ instead to propose an efficient algorithm for finding good slices. 

\section{Active slice direction $\bm{g}$}
\label{sec: Active slice direction g}
% In the section, we focused on finding active slice $\bm{g}_{r}$ by analyzing $S_{g_r}$ with orthogonal basis $O_r$. We first establish an upper bound for $S_{g_r}$. Thus, finding optimal $\bm{g}_r$ is translated into minimizing the error between the upper bound and $S_{g_r}$. This error also admits an upper bound that can be analytically minimized. We call this alternative problem \textit{controlled approximation} and the active slices $\bm{g}_r$ are the ones that minimizes the error upper bound.

Finding good slices involves alternating maximization of $\bm{r}$ and $\bm{g}_r$. To simplify the analysis, we focus on good directions $\bm{g}_r$ given fixed $\bm{r}$, e.g. the orthogonal basis $\bm{r}\in O_r$ for now. Finding good $\bm{g}_r$ is achieved in two steps: (i) Rewriting the problem of the maximizing $S_{g_r}$ w.r.t $\bm{g}_r$ into an equivalent minimization problem, called \emph{controlled approximation}; (ii) Establish an upper-bound of the controlled approximation objective such that its minimizer is analytic. \wg{This derivation is based on an important inequality: Poincar\'e inequality, which upper bounds the variances of a function by its gradient magnitude. Therefore, we need \textbf{Assumptions 7-8} (Appendix \ref{appendix: definition and assumption}) to make sure this inequality is valid. }We name the resulting $\bm{g}_{r}$ that minimizes the upper bound as \emph{active slices}. All proofs can be found in appendix \ref{Appendix: theory related to g}.

%To start with, we propose a variant of SD (Eq.\ref{eq: SD}) called \textit{projected Stein discrepancy} (PSD):
%\begin{equation}
%    \text{PSD}(q,p;O_r)=\sum_{\bm{r}\in O_r}\sup_{f_r\in\mathcal{F}_q}\mathbb{E}_q[s^r_p(\bm{x})f_r(\bm{x})+\bm{r}^T\nabla_{\bm{x}}f_r(\bm{x})]
%    \label{eq: PSD}
%\end{equation}
%where $f_{r}:\mathcal{X}\subseteq\mathbb{R}^{D}\rightarrow \mathbb{R}$. SD is a special case of PSD by setting $O_r$ as identity matrix $\bm{I}$. We can regard PSD as an upper bound for $S_{g_r}$ (theorem \ref{thm: controlled approximation}) because $S_{g_r}$ is recovered by $f_r(\bm{x})=h_{rg_r}(\bm{x}^T\bm{g})$. All proofs can be found in appendix \ref{Appendix: theory related to g}. 

\subsection{Controlled Approximation}
 \label{subsec: Controlled approximation}
% \textit{PSD} does not have a nice form as there is a supremum inside. Fortunately, we can analytically solve this optimization with mild assumptions. 

To start with, we need an upper bound for $S_{g_r}$ so that we can transform the maximization of $S_{g_r}$ into the minimization of their gap. Hence, we propose a generalization of SD (Eq.\ref{eq: SD}) called \textit{projected Stein discrepancy} (PSD):
\begin{equation}
    \text{PSD}(q,p;O_r)=\sum_{\bm{r}\in O_r}\sup_{f_r\in\mathcal{F}_q}\mathbb{E}_q[s^r_p(\bm{x})f_r(\bm{x})+\bm{r}^T\nabla_{\bm{x}}f_r(\bm{x})]
    \label{eq: PSD}
\end{equation}
where $f_{r}:\mathcal{X}\subseteq\mathbb{R}^{D}\rightarrow \mathbb{R}$. SD is a special case of PSD by setting $O_r$ as identity matrix $\bm{I}$.
In proposition \ref{prop: optimality of PSD} of appendix \ref{appendix subsec: optimal test function for PSD}, we show that {if $\mathcal{F}_q$ contains all bounded continuous functions}, then the optimal test function in PSD is
\begin{equation}
    f_{r}^{*}(\boldsymbol{x}) \propto\left(s_{p}^{r}(\boldsymbol{x})-s_{q}^{r}(\boldsymbol{x})\right)\,.
    \label{eq: optimal test function PSD}
\end{equation}
It can also be shown that  PSD is equivalent to the \textit{Fisher divergence}, which has been extensively used in training energy based models \cite{song2020sliced,song2019generative} and fitting kernel exponential families \cite{sriperumbudur2017density,sutherland2018efficient,wenliang2019learning}. 

We now prove that PSD upper-bounds $S_{g_r}$, with the gap as the expected square error between their optimal test functions $f_r^*$ and $h^*_{rg_r}$ (Proposition \ref{prop: optimality of maxSSD-g}). Since PSD is constant w.r.t.~$\bm{g}_r$, maximization of $S_{g_r}$ is equivalent to {a minimization task, called \emph{controlled approximation}}.
% \textcolor{red}{Without loss of generality, we assume the coefficient for $f_{r}^*$, $h_{rg_r}^*$ to be 1}. 

% Next, we show \textit{PSD} is indeed an upper bound for $S_{g_r}$, and the error is the expected square loss between $f_r^*$ and $h^*_{rg_r}$. 
% We can now introduce the \textit{controlled approximation}, which is the error between optimal \textit{PSD} and \textit{SSD-g}. In the following, we denote \textit{PSD} with optimal test functions as $\text{PSD}^*$, and same for \textit{SSD-g} ($S^*_{{g_r}}$). By showing that $\text{PSD}^*$ is indeed an upper bound for $S^*_{{g_r}}$, finding the good slice $\bm{g}_r$ is equivalent to minimizing this controlled approximation. 
\begin{theorem}[Controlled Approximation]
Assume assumptions \hyperref[assumption 1]{1-4} (density regularity), and the coefficient for the optimal test functions to be $1$ w.l.o.g., %
% Let's call the \textit{PSD} (\textit{SSD-g} resp.) with corresponding optimal test function as $\text{PSD}^*$ ($S_{g_r}^*$ resp.). 
then $\text{PSD}\geq S_{g_r}$ and
\begin{equation}
    \text{PSD}-S_{g_r}=\sum_{\bm{r}\in O_r}\mathbb{E}_q[(f_{r}^*(\bm{x})-h_{rg_r}^*(\bm{x}^T\bm{g}_r))^2],
    \label{eq: controlled approximation}
\end{equation}\label{thm: controlled approximation}with $f_r^*$ and $h^*_{rg_r}$ are optimal test functions for PSD and $S_{g_r}$ defined in Eq.\ref{eq: optimal test function PSD} and Eq.\ref{eq: optimal test function maxSSD-g} respectively.
\end{theorem}
Intuitively, minimizing the above gap can be regarded as a function approximation problem, where we want to approximate a multivariate function $f_r^*:\mathbb{R}^{D}\rightarrow \mathbb{R}$ by a univariate function $h^*_{rg_r}:\mathbb{R}\rightarrow\mathbb{R}$ with optimal parameters $\bm{g}_r$. 
% The above theorem shows that $\text{PSD}$ upper bounds $S_{g_r}$, and their error can be re-interpreted as expected square approximation error between $f_r^*$ and $h^*_{rg_r}$. 
% From the above theorem, the maximization of \textit{SSD-g} w.r.t. $\bm{g}_r$ is replaced by minimizing the expected squared error between two optimal test functions. Intuitively, one can also interpret it as a function approximation problem where we want to approximate multivariate functions $f_{r}^*$ by uni-variate function $h_{rg_r}^*$ with parameter $\bm{g}_r$. 

\subsection{Upper-bounding the error}
\label{subsec: subspace poincare}
Solving the controlled approximation task directly may be difficult in practice. Instead, we propose an upper-bound of the approximation error, such that this upper-bound's minimizer $\bm{g}_r$ is analytic. The inspiration comes from the \textit{active subspace method} for dimensionality reduction \citep{constantine2014active,zahm2020gradient}, therefore we name the corresponding minimizers as \emph{active slices}.

\begin{theorem}[Error upper-bound and active slices $\bm{g}_r$]
Assume assumptions \hyperref[assumption 2]{2}, \hyperref[assumption 4]{4} (density regularity) and \hyperref[assumption 4]{7-8} (Poincar\'e inequality conditions), we can upper bound the inner part of the controlled approximation error (Eq.\ref{eq: controlled approximation}) by
\begin{equation}
\begin{split}
    \mathbb{E}_{q}\left[\left(f_{r}^{*}(\boldsymbol{x})-h_{rg_r}^{*}\left(\boldsymbol{x}^{T} \boldsymbol{g}_r\right)\right)^{2}\right] \leq C_{\text {sup}} \operatorname{tr}\left(\boldsymbol{G}_{r\backslash d} \boldsymbol{H}_r \boldsymbol{G}_{r\backslash d}^{T}\right),
\end{split}
\label{eq: subspace Poincare upper bound}
\end{equation}
%where 
\begin{equation}
    \begin{split}
        \boldsymbol{H}_r=\int q(\boldsymbol{x}) \nabla_{x} f_{r}^{*}(\boldsymbol{x}) \nabla_{\boldsymbol{x}} f_{r}^{*}(\boldsymbol{x})^{T} d \boldsymbol{x}.
    \end{split}
    \label{eq: sentivity matrix}
\end{equation}
Here $C_{sup}$ is the Poincar\'e constant defined in assumption \hyperref[assumption 8]{8} and $\bm{G}_{r\backslash d}\in\mathbb{R}^{(D-1)\times D}$ is an arbitrary orthogonal matrix $\bm{G}_r$ excluding the $d^{th}$ row $\bm{g}_{r}$. The orthogonal matrix has the form $\bm{G}_r=[\bm{a}_1,\ldots,\bm{a}_D]^T$ where $\bm{a}_i\in\mathbb{S}^{D-1}$ and $\bm{a}_d=\bm{g}_r$.

The above upper-bound is minimized when the row space of $\bm{G}_{r\backslash d}$ is the span of the first $D-1$ eigenvectors of $\bm{H}_r$ (arranging eigenvalues in ascending order). One possible choice for active slice $\bm{g}_r$ is $\bm{v}_D$,
where $(\lambda_i,\bm{v}_i)$ is the eigenpair of $\bm{H}_r$ and $\lambda_1\leq\lambda_2\leq\ldots\leq\lambda_D$.
\label{thm: Upper bound controlled approximation}
\end{theorem}
% In the following corollary, we show how to analytically select the active slice $\bm{g}_r$ given $\bm{r}$ to minimize this bound.
% \begin{corollary}[Active slice $\bm{g}_r$]
% With the conditions of theorem \ref{thm: Upper bound controlled approximation}, the upper bound is minimized when the columns of $\bm{G}^T_{r\backslash d}$ spans the eigenspace corresponding to the $D-1$ smallest eigenvalues of $\bm{H}_r$. Namely,
% \begin{equation}
%     \begin{split}
%         \min_{\bm{G}_{r\backslash d}}tr(\bm{G}_{r\backslash d}\bm{H}_{r}\bm{G}_{r\backslash d}^T)=\sum_{i=1}^{D-1}{\lambda_{i}}
%     \end{split}
%     \label{eq: minimization H}
% \end{equation}
% where $(\lambda_i,\bm{v}_i)$ is the eigenpair of $\bm{H}_r$ and $\lambda_1\leq\lambda_2\leq\ldots\leq\lambda_D$. One possible choice for active slice is $\bm{g}_r=\bm{v}_{D}$. 
% \label{coro: minimization of upper bound }
% \end{corollary}
% Thus, one can obtain the active slice direction $\bm{g}_r$ for each $\bm{r}\in O_r$ by solving an eigenvalue problem of $\bm{H}_{r}$ estimated by Monte Carlo method.
Intuitively, the active slices $\bm{g}_r=\bm{v}_D$ are the directions where the test function $f_r^*$ varies the most. Indeed,  we have $\bm{v}_D^T\bm{H}_r\bm{v}_D=\mathbb{E}_q[||\nabla_{\bm{x}}f_r^*(\bm{x})^T\bm{v}_D||^2]=\lambda_D$, where the eigenvalue $\lambda_D$ measures the averaged gradient variation in the direction defined by $\bm{v}_D$.
% If a $f_{r}^*$ does not vary in a particular direction, those directions can be regarded as unimportant ones. 

% In practice, we could use fast algorithm for finding largest eigenpairs when the dimension is high. We should also recognize that active slices $\bm{g}_r$ are not guaranteed to be optimal in general, as it only minimize an upper bound instead of the actual error itself.
\section{Active slice direction $\bm{r}$}
\label{sec: active slice direction r}
% From above theorem \ref{thm: validity of maxSKSD} and corollary 1.1 in \cite{gong2021sliced}, either random samples $\bm{r}$ or orthogonal basis $O_r$ preserves the discrepancy validity with suitable $\bm{g}_r$.
% However, there are two major limitations in practice with the naive choice of $\bm{r}$ or $O_r$: (i) $\bm{r}$ or $O_r$ may be unimportant directions, which gives weak discriminating power and high variance in practice. (ii) the computational cost with $O_r$ are $O(D^2)$. 

The dependence of active slice $\bm{g}_r$ on $\bm{r}$ motivate us to consider the possible choices of $\bm{r}$.
Although finite random slices $\bm{r}$ are sufficient for obtaining a valid discrepancy, in practice using sub-optimal $\bm{r}$ can result in weak discriminative power and poor active slices $\bm{g}_r$. We address this issue by proposing an efficient algorithm to search for good $\bm{r}$.
% However, from (i), we should only care about the few important ones. Thus, the cost can be potentially reduced to $O(mD)$ where $m$ is the number of important directions.
% To tackle these issues, we propose an method to find active slices $\bm{r}$.
% These practical drawbacks motivate us to efficiently search for good $\bm{r}$.  
Again all the proofs can be found in appendix \ref{Appendix: theory related to r}.

% \textit{maxSSD-g} with the optimal slice $\bm{r}$ is called \textit{maxSSD-rg} (Eq.\ref{eq: maxSSD-rg}). 
\subsection{PSD Maximization for searching $\bm{r}$}
Directly optimizing $S_{rg_r}$ w.r.t.~$\bm{r}$ is particularly difficult due to the alternated updates of $\bm{r}$ and $\bm{g}_r$. To simplify the analysis, we start from the task of finding a single direction $\bm{r}$. Our key idea to sidestep such alternation is based on the intuition that $S_{rg_r}$ with active slices $\bm{g}_r$ should well approximate $\text{PSD}_{r}$ (PSD with given $\bm{r}$) from theorem \ref{thm: Upper bound controlled approximation}. The independence of $\text{PSD}_r$ to $\bm{g}_r$ allows us to avoid the alternated update and the accurate approximation validates the direct usage of the resulting active slices in $S_{rg_r}$. 
% The key idea of our approach is to find an alternative objective for $\bm{r}$ with tractable solutions. Inspired by the analysis of active slice $\bm{g}_r$, minimization of controlled approximation means $S_{g_r}$ is a good approximation to $\text{PSD}$. Thus, \textit{PSD} seems to be a good objective.
Indeed, we will prove that maximizing $\text{PSD}_r$ is equivalent to maximizing a lower-bound for $S_{rg_r}$. 

Assume we have two slices $\bm{r}_1$ and $\bm{r}_2$, with given $\bm{g}_{r_1}$, $\bm{g}_{r_2}$. Then finding good $\bm{r}_1$ is equivalent to maximizing the difference $S_{{r_1,g_{r_1}}}-S_{{r_2,g_{r_2}}}$. The following proposition establishes a lower-bound for this difference.
\begin{prop}[Lower-bound for the $S_{rg_r}$ gap]
Assume the conditions in theorem \ref{thm: Upper bound controlled approximation} are satisfied, then for any slices $\bm{r}_1$, $\bm{r}_2$ and $\bm{g}_{r_1}$, $\bm{g}_{r_2}$, we have 
\begin{equation}
    \begin{split}
        S_{{r_1,g_{r_1}}}-S_{{r_2,g_{r_2}}}\geq \text{PSD}_{r_1}-\text{PSD}_{r_2}-C_{\text{sup}}\Omega,
    \end{split}
    \label{eq: lower bound active slice r}
\end{equation}
where $C_{\text{sup}}$ is the Poincar\'e constant defined in assumption \hyperref[assumption 8]{8} and $\Omega=\sum_{i=1}^D{\omega_i}$ where $\{\omega_i\}_{i}^D$ is the eigenvalue of $\mathbb{E}_q[\nabla_{\bm{x}}\bm{f}^*(\bm{x})\nabla_{\bm{x}}\bm{f}^*(\bm{x})^T]$, $\bm{f}^*(\bm{x})=\bm{s}_p(\bm{x})-\bm{s}_q(\bm{x})$.
\label{prop: lower bound active slice r}
\end{prop}
Proposition \ref{prop: lower bound active slice r} justifies the maximization of $\text{PSD}_{r_1}$ w.r.t.~$\bm{r}_1$ as a valid surrogate. But more importantly, this alternative objective admits an analytic maximizer of $\bm{r}$, which is then used as the active slice direction:
\begin{theorem}[Active slice $\bm{r}$]
Assuming assumptions \hyperref[assumption 1]{1-4} (density regularity), then the maximum of the $\text{PSD}_{r}$ is achieved at $\boldsymbol{r}^{*}=\boldsymbol{v}_{\text{max}}$:
$$
\max_{\boldsymbol{r} \in \mathbb{S}^{D-1}} \mathbb{E}_{q}\left[s_{p}^{r}(\boldsymbol{x}) f^*_{r}(\boldsymbol{x})+\boldsymbol{r}^{T} \nabla_{\boldsymbol{x}} f^*_{r}(\boldsymbol{x})\right]=\lambda_{\text{max} }.
$$
Here $\left(\lambda_{\text{max} }, \boldsymbol{v}_{\text{max} }\right)$ is the largest eigenpair of the matrix $\boldsymbol{S}=\mathbb{E}_{q}\left[\bm{f}^*(\bm{x})\bm{f}^*(\bm{x})^T\right]
$
\label{thm: active slice r}
\end{theorem}
% Thus, for a single slice direction $\bm{r}$, we only need to find the largest eigenvector of $\bm{S}$. One should note that the active slice $\bm{r}$ is not the optimal slice in general.
\subsection{Constructing the orthogonal basis $O_r$}
\label{subsec: greedy algorithm}
Under certain scenarios, e.g. model learning, we want to train the model to perform well in every directions instead of a particular one. Thus, using a good orthogonal basis is preferred over a single active slice $\bm{r}$.
Here gradient-based optimization is less suited as it breaks the orthogonality constraint.
Also proposition \ref{prop: lower bound active slice r} is less useful here as well, as PSD is invariant to the choice of $O_r$, i.e. $\text{PSD}(q,p;O_{r_1})=\text{PSD}(q,p;O_{r_2})$ and $O_{r_1}\neq O_{r_2}$. 

Inspired by the analysis of single active $\bm{r}$, we propose to use the eigendecomposition of $\bm{S}$ to obtain a good orthogonal basis $O_r$. Theoretically, this operation also corresponds to a greedy algorithm, where in step $i$ it searches for the optimal direction $\bm{r}_i$ that is orthogonal to $\{\bm{r}_{<i} \}$ and maximizes $\text{PSD}_{\bm{r}_i}$ (see Corollary \ref{coro: greedy algorithm eigendecomposition} in appendix \ref{subsec: greedy algorithm}).
% Unfortunately, there is no good alternative objective for this task.
% $\text{PSD}$ (Eq.\ref{eq: PSD}) is no longer useful as it is invariant to orthogonal basis $O_r$. Namely, for two different orthogonal basis $O_{r_1}$, $O_{r_2}$, it is easy to verify $\text{PSD}(q,p;O_{r_1})=\text{PSD}(q,p;O_{r_2})$. 
% In the following, we propose a heuristic algorithm to find a orthogonal basis $O_r$. 
%\begin{algorithm}[t]
%  \caption{Greedy algorithm for orthogonal basis $O_r$}
%  \label{alg: greedy algorithm}
%\begin{algorithmic}
%  \STATE {\bfseries Result:} $O_r$
%  \STATE {Obtain active slice $\bm{r}$ from theorem \ref{thm: active slice %r}}
%  \REPEAT
%  \STATE {Find $\bm{r}'$ that is orthogonal to all existing $\bm{r}$ and maximize $\text{PSD}_{r'}$}
%  \UNTIL{$O_r$ is obtained}
%\end{algorithmic}
%\end{algorithm}
Although there is no guarantee for finding the \emph{optimal} $O_r$ due to its myopic behavior, in practice this greedy algorithm at least finds some good directions with high discriminative power (eigenvectors with large eigenvalues). 
% The next corollary shows that the above greedy algorithm is exactly the same as finding all the eigenvectors of $\bm{S}$.

% our proposition 3 shows that choosing the top eigenvector is good. the greedy algo is just a reasonable way to find a orthogonal basis
% \begin{corollary}[Greedy algorithm is eigen-decomposition]
% Assume the conditions in theorem \ref{thm: active slice r} are satisfied, then the $O_r$ from the greedy algorithm is equivalent to the group of all eigenvectors of $\bm{S}$. 
% \label{coro: greedy algorithm eigendecomposition}
% \end{corollary}

\section{Practical algorithm}
\label{sec: In practice}

\begin{algorithm}[t]
   \caption{Active slice algorithm}
   \label{alg: active slice algorithm}
\begin{algorithmic}
   \STATE {\bfseries Input:} Samples $\bm{x}\sim q$, density $p$, kernel $k:\mathcal{X}\times\mathcal{X}\rightarrow \mathbb{R}$, Gaussian noise $\gamma$, pruning factor $m$ (optional)
   \STATE {\bfseries Result:} $\widetilde{O_r}$, $\bm{G}$
   \STATE {Estimate $\bm{s}_p(\bm{x})-\bm{s}_q(\bm{x})$ using KE or GE with kernel $k$ and samples $\bm{x}$.}
   \IF {Pruning}
    \STATE Top $m$ eigenvectors of $\bm{S}$ to form $\widetilde{O_r}$ (Theorem \ref{thm: active slice r})
   \ELSE
   \STATE Getting all eigenvectors of $\bm{S}$ to form $\widetilde{O_r}$
   \ENDIF
    \STATE {Add noise $\gamma$ to $\widetilde{O_r}$, then normalize. (Section \ref{subsec: validity of active slices})}
   \FOR {$\bm{r}\in\widetilde{O_r}$}
   \STATE {$\bm{g}_r$ is the top 1 eigenvector of $\bm{H}_r$ (Theorem \ref{thm: Upper bound controlled approximation})}
   \STATE Add noise $\gamma$ to $\bm{g}_r$ then normalize (Section \ref{subsec: validity of active slices})
   \STATE Concatenate $\bm{g}_r$ to $\bm{G}$
   \ENDFOR
   \STATE Further optimize $\widetilde{O_r}$, $\bm{G}$ with \textit{SKSD-g} ($SK_{g_r}$) using gradient-based optimization (Optional)
   \STATE {{\bfseries Return:} $\widetilde{O_r}$, $\bm{G}$}
\end{algorithmic}
\end{algorithm}

The proposed active slice method is summarized in Algorithm \ref{alg: active slice algorithm}, which requires the intractable score difference $\bm{s}_p(\bm{x})-\bm{s}_q(\bm{x})$. Two types of approximations can be used. The first approach applies \textit{gradient estimators} (GE) to estimate $\bm{s}_q(\bm{x})$ from $\bm{x}$ samples. We use the Stein gradient estimator \cite{li2017gradient} for the GE approach, although other estimators \cite{sriperumbudur2017density,sutherland2018efficient,shi2018spectral,zhou2020nonparametric} can also be employed. 
The second method directly estimates the score difference using a \textit{kernel-smoothed estimator} (KE):
\begin{equation}
\hspace{-1em}
\begin{aligned}
    \bm{s}_p(\bm{y})-\bm{s}_q(\bm{y}) \approx& \mathbb{E}_{\bm{x}\sim q}[(\bm{s}_p(\bm{x})-\bm{s}_q(\bm{x}))k(\bm{x},\bm{y})]\\
=&\mathbb{E}_{\bm{x}\sim q}[\bm{s}_p(\bm{x})k(\bm{x},\bm{y})+\nabla_{\bm{x}}k(\bm{x},\bm{y})],
\end{aligned}
\label{eq: kernel smooth}
\end{equation}
where the second expression comes from integration by part, and it can be computed in practice. Figure \ref{fig: divergence relationship} summarizes the relationships between different SSD discrepancies and highlights our contributions. For GOF test specifically, we also derive the asymptotic distribution and propose an practical GOF algorithm in appendix \ref{appendix: goodness-of-fit test}. 

\subsection{Computational cost}
\wg{The overall complexity includes the cost for (1) finding active slices (algorithm \ref{alg: active slice algorithm}) (2) applying the downstream test. For finding the active slices $\bm{r}$, one important fact is that we only need the $m$ ($m\ll D$) most important $\bm{r}$ (importance characterised by eigenvalues). Luckily, fast eigenvalue-decomposition algorithm, e.g. randomized SVD from \citet{saibaba2021randomized}, requires $O(m)$ matrix-vector product. For $\bm{g}$, from algorithm \ref{alg: active slice algorithm}, we only need to solve $m$ eigenvalue-decomposition, each only cares about the most important eigenvector. Therefore, $O(m\times 1)$ matrix-vector product are needed. So the overall complexity for finding slices is $O(mD^2)$, where $D^2$ comes from matrix-vector product. For gradient-based optimization (GO), the complexity is $O(l(D^2+C_{\text{grad}}))$ ($l$ is optimization step and $C_{\text{grad}}$ is the back-prop cost, $D^2$ coms from evaluating $SK_{g_{r}}$ or $SK_{rg_r}$). Our algorithm in general has lower training cost as $l\gg m$ and $C_{\text{grad}}$ can be expensive. For (2), our method has $O(mD)$ cost compared to $O(D^2)$ for GO. As $m\ll D$, active slices have less complexity compared to pure GO based method proposed in \citet{gong2021sliced}. For memory cost, our method costs $O(mD)$ to store $\bm{r},\bm{g}$ whereas GO uses $O(D^2)$. Overall, our method requires nearly an order of magnitude less complexity in terms of computation and memory consumption. }

%\subsection{Practical Algorithm}

% \begin{algorithm}[H]
%   \caption{Active slice algorithm}
%   \label{alg: active slice algorithm}
% \begin{algorithmic}
%   \STATE {\bfseries Input:} Samples $\bm{x}\sim q$, density $p$, kernel $k:\mathcal{X}\times\mathcal{X}\rightarrow \mathbb{R}$, Gaussian noise $\gamma$, pruning factor $m$ (optional)
%   \STATE {\bfseries Result:} $\widetilde{O_r}$, $\bm{G}$
%   \STATE {Estimate $\bm{s}_p(\bm{x})-\bm{s}_q(\bm{x})$ using KE (Eq.\ref{eq: kernel smooth}) or GE with kernel $k$ and samples $\bm{x}$.}
%   \IF {Pruning}
%     \STATE Top $m$ eigenvectors of $\bm{S}$ to form $\widetilde{O_r}$ (Theorem \ref{thm: active slice r})
%   \ELSE
%   \STATE Getting all eigenvectors of $\bm{S}$ to form $\widetilde{O_r}$
%   \ENDIF
%     \STATE {Add noise $\gamma$ to $\widetilde{O_r}$, then normalize. (Section \ref{subsec: validity of active slices})}
%   \FOR {$\bm{r}\in\widetilde{O_r}$}
%   \STATE {$\bm{g}_r$ is the top 1 eigenvector of $\bm{H}_r$ (Theorem \ref{thm: Upper bound controlled approximation})}
%   \STATE Add noise $\gamma$ to $\bm{g}_r$ then normalize (Section \ref{subsec: validity of active slices})
%   \STATE Concatenate $\bm{g}_r$ to $\bm{G}$
%   \ENDFOR
%   \STATE Further optimize $\widetilde{O_r}$, $\bm{G}$ with \textit{SKSD-g} ($SK_{g_r}$) using gradient-based optimization (Optional)
%   \STATE {{\bfseries Return:} $\widetilde{O_r}$, $\bm{G}$}
% \end{algorithmic}
% \end{algorithm}

\section{Experiments}
\label{sec: Experiments}
\wg{GOF test aims to test the fitness of the model to the target data. The test procedure roughly proceeds as: (1) Define null hypothesis (model matches the data distribution) and alternative hypothesis (model does not match the data distribution); (2) Compute test statistic (e.g. KSD) and threshold (e.g. bootstrap method); (3) Reject null hypothesis (statistic $>$ threshold) or not (statistic $\leq$ threshold).  Refer to appendix \ref{appendix: goodness-of-fit test} for more details. }
\subsection{Benchmark GOF tests}
\begin{figure*}[t]
    \centering
    \includegraphics[scale=0.14]{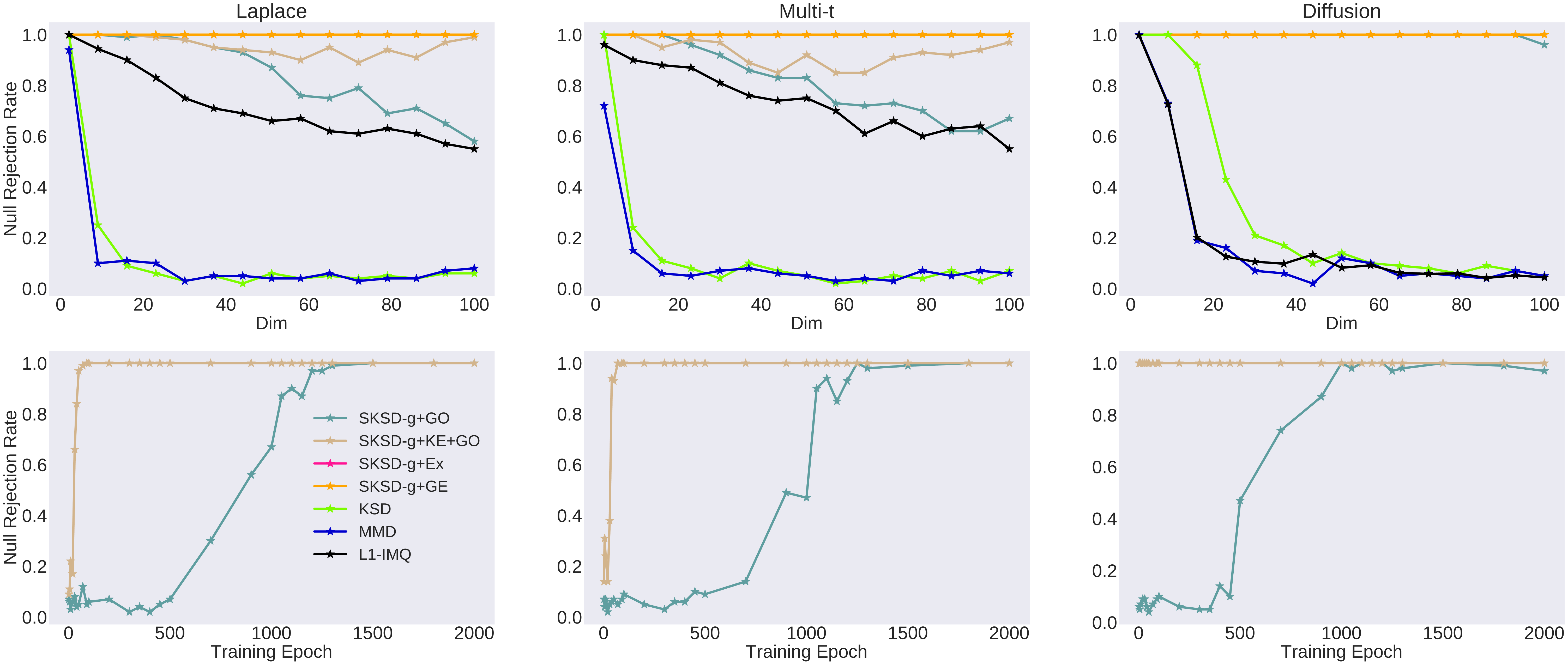}
    \caption{(\textbf{Upper panel}): The null rejection rate w.r.t. different dimensional benchmark problems. \textit{SKSD-g+Ex} and \textit{SKSD-g+GE} coincide at the optimal rejection rate (\textbf{Lower panel}): Null rejection rate with different number of gradient optimization epochs.}
    \label{fig: Benchmark performance}
\end{figure*}
We demonstrate the improved  test power results (in terms of null rejection rates) and significant speed-ups of the proposed active slice algorithm on 3 benchmark tasks, which have been extensively used for measuring GOF test performances \citep{jitkrittum2017linear,huggins2018random,chwialkowski2016kernel,gong2021sliced}. Here the test statistic is based on \textit{SKSD-g} ($SK_{g_r}$) with fixed basis $O_r=\bm{I}$. Two practical approaches are considered for computing the active slice $\bm{g}_r$: (i) gradient estimation with the Stein gradient estimator (\emph{SKSD-g+GE}), and (ii) gradient estimation with the kernel-smoothed estimator (\emph{KE}), plus further gradient-based optimization (\emph{SKSD-g+KE+GO}). For reference, we include a version of the algorithm with exact score difference (\emph{SKSD-g+Ex}) as an ablation for the gradient estimation approaches. 

In comparison, we include the following strong baselines: \emph{KSD} with RBF kernel \citep{liu2016kernelized,chwialkowski2016kernel}, maximum mean discrepancy \citep[MMD,][]{gretton2012kernel} with RBF kernel, random feature Stein discrepancy with L1 IMQ kernel \citep[L1-IMQ,][]{huggins2018random}, and the current state-of-the-art --- \emph{maxSKSD-g} with random initialized $\bm{g}_r$ followed by gradient optimization \citep[SKSD-g+GO,][]{gong2021sliced}. For all \textit{}methods requiring \textit{GO} or active slices, we split the $1000$ test samples from $q$ into $800$ test and $200$ training data, where we run GO or active slice method on the training set. 

The 3 GOF test benchmarks, with details in appendix \ref{appendix subsec: benchmark GOF test}, are: (1) \textbf{Laplace}: $p(\bm{x}) = \mathcal{N}(0, \bm{I})$, $q(\bm{x})=\prod_{d=1}^D{\text{Lap}(x_d|0,1/\sqrt{2})}$; (2) \textbf{Multivariate-t}: $p(\bm{x}) = \mathcal{N}(0, \frac{5}{3}\bm{I})$, $q(\bm{x})$ is a fully factorized multivariate-t with $5$ degrees of freedom, $0$ mean and scale $1$; (3) \textbf{Diffusion}: $p(\bm{x}) = \mathcal{N}(0, \bm{I})$,  $q(\bm{x})=\mathcal{N}(\bm{0},\bm{\Sigma}_1)$ where in $q(\bm{x})$ the variance of $1^{\text{st}}$-dim is $0.3$ and the rest is $\bm{I}$.

The upper panels in Figure \ref{fig: Benchmark performance} show the test power results as the dimensions $D$ increase. As expected, \textit{KSD} and \textit{MMD} with RBF kernel suffer from the curse-of-dimensionality. \textit{L1-IMQ} performs relatively well in \textbf{Laplace} and \textbf{multivariate-t} but still fails in \textbf{diffusion}. For SKSD based approaches, \textit{SKSD-g+GO} with $1000$ {training epochs} still exhibits a decreasing test power in \textbf{Laplace} and \textbf{multivariate-t}. On the other hand, \textit{SKSD-g+KE+GO} with 50 training epochs has nearly optimal performance. \textit{SKSD-g+Ex} and \textit{SKSD-g+GE} achieve the true optimal rejection rate without any \textit{GO}. Specifically, Table \ref{tab: benchmark GOF} shows that the active slice method achieves significant computational savings with \textbf{14x}-\textbf{80x} speed-up over \textit{SKSD-g+GO}. 

For approaches that require gradient optimization, the lower panels in Figure \ref{fig: Benchmark performance} show the test power as the number of {training epochs} increases. \emph{SKSD-g+GO} with random slice initialization requires a huge number of gradient updates to obtain reasonable test power, and $1000$ {epochs} achieves the best balance between run-time and performance. On the other hand, \emph{SKSD-g+KE+GO} with active slice achieves significant speed-ups with near-optimal test power using around \textbf{50} {epochs} on \textbf{Laplace} and \textbf{Multivariate-t}. Remarkably, on \textbf{Diffusion} test, $\bm{g}_r$ initialized by the active slices achieves near-optimal results already, so that the later gradient refinements are not required.

% \begin{table*}[t]
% \centering
% \caption{Test power for 100 dimensional benchmarks and time consumption. \textcolor{red}{The run-time for SKSD-g+KE-GO include both the active slice computation and the later gradient-based refinement steps.}\wg{TOD: Change test power to NRR and add multivariate-t results}}
% \begin{tabular}{l|lll|lll}
% \hline
%              & \multicolumn{3}{c|}{Laplace}                  & \multicolumn{3}{c}{Diffusion}                              \\ \hline
% Method       & Test Power    & sec/trial     & Speed-up      & Test Power & \multicolumn{1}{r}{sec/trial} & Speed-up      \\ \hline
% SKSD-g+Ex    & \textbf{1}    & \textbf{0.38} & \textbf{103x} & \textbf{1} & \textbf{0.34}                  & \textbf{102x} \\ \hline
% SKSD-g+GO    & 0.58          & 39.39         & 1x            & 0.96       & 34.73                          & 1x            \\
% SKSD-g+KE+GO & \textbf{0.99} & 2.72          & 14x           & \textbf{1} & \textbf{0.43}                  & \textbf{81x}  \\
% SKSD-g+GE    & \textbf{1}    & \textbf{0.66} & \textbf{60x}  & \textbf{1} & \textbf{0.78}                  & 44x           \\\hline

% \end{tabular}
% \label{tab: benchmark GOF}
% \end{table*}

\begin{table*}[t]
\centering
\caption{Test power for 100 dimensional benchmarks and time consumption. {The run-time for SKSD-g+KE+GO include both the active slice computation and the later gradient-based refinement steps.} NRR stands for null rejection rate.}
\begin{tabular}{l|lll|lll|lll}
\hline
             & \multicolumn{3}{c|}{Laplace}                  & \multicolumn{3}{c|}{Multi-t}                 & \multicolumn{3}{c}{Diffusion}             \\ \hline
Method       & NRR           & sec/trial     & Speed-up      & NRR           & sec/trial     & Speed-up     & NRR        & sec/trial     & Speed-up      \\ \hline
SKSD-g+Ex    & \textbf{1}    & \textbf{0.38} & \textbf{103x} & \textbf{1 }            & \textbf{0.49} & \textbf{90x} & \textbf{1} & \textbf{0.34} & \textbf{102x} \\ \hline
SKSD-g+GO    & 0.58          & 39.39         & 1x            & 0.67          & 44.24         & 1x           & 0.96       & 34.73         & 1x            \\
SKSD-g+KE+GO & \textbf{0.99} & 2.72          & 14x           & \textbf{0.97} & 2.38          & 19x          & \textbf{1} & \textbf{0.43} & \textbf{81x}  \\
SKSD-g+GE    & \textbf{1}    & \textbf{0.66} & \textbf{60x}  & \textbf{1}    & \textbf{0.67} & \textbf{66x} & \textbf{1} & \textbf{0.78} & \textbf{44x}  \\ \hline
\end{tabular}
\label{tab: benchmark GOF}

\end{table*}

\subsection{RBM GOF test}
\begin{figure*}[t]
    \centering
    \includegraphics[scale=0.14]{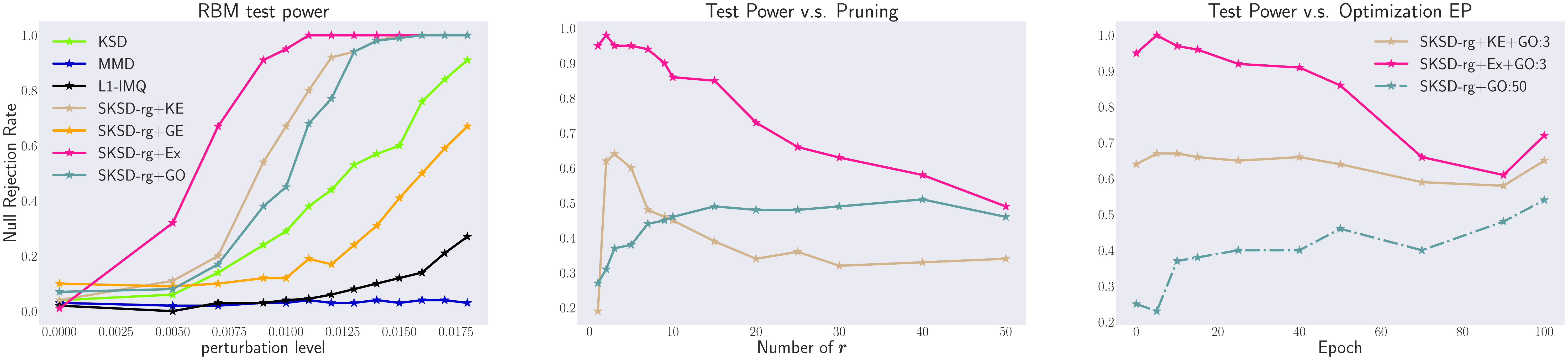}
    \caption{(\textbf{Left}): The GOF test power of each method with different level of noise perturbations (\textbf{Mid}): The effect of different pruning level towards the test power   (\textbf{Right}): The effect of gradient based optimization epoch to the test power. 3 and 50 indicates the pruning level. \textit{KE+GO} or \textit{Ex+GO} means active slices with further gradient refinement steps. }
    \label{fig: RBM GOF}
\end{figure*}
Following \citet{gong2021sliced}, we conduct a more complex GOF test using restrict Boltzman machines (RBMs, \citep{hinton2006reducing,welling2004exponential}). Here the $p$ distribution is an RBM: $p(\bm{x})=\frac{1}{Z} \exp \left(\mathbf{x}^{\top} \mathbf{B} \mathbf{h}+\mathbf{b}^{\top} \mathbf{x}+\mathbf{c}^{\top} \mathbf{x}-\frac{1}{2}\|\mathbf{x}\|^{2}\right)$, where $\bm{x}\in\mathbb{R}^D$ and $\bm{h}\in\{\pm 1\}^{d_{h}}$ denotes the hidden variables. The $q$ distribution is also an RBM with the same $\mathbf{b}, \mathbf{c}$ parameters as $p$ but a different $\bm{B}$ matrix perturbed by different levels of Gaussian noise. We use $D=50$ and $d_h=40$, and block Gibbs sampler with $2000$ burn-in steps. The test statistics for all the approaches are computed on a test set containing 1000 samples from $q$.

The test statistic is constructed using \textit{SKSD-rg} ($SK_{rg_r}$) with $\bm{r}$, $\bm{g}_r$ obtained either by gradient-based optimization (\textit{SKSD-rg+GO}) or the active slice algorithms (\textit{+KE}, \textit{+GE} and \textit{+Ex}) without the gradient refinements. Specifically, \textit{SKSD-rg+GO} runs 50 {training epochs} with $\bm{r}$ and $\bm{g}_r$ initialized to $\bm{I}$. For the active slice methods, we also prune away most slices and only keep the top-$3$ most important $\bm{r}$ slices.    

The left panel of Figure \ref{fig: RBM GOF} shows that \textit{SKSD-rg+KE} achieves the best null rejection rates among all baselines, except for \textit{SKSD-rg+Ex} whose performance is expected to upper-bound all other active slice methods. This shows the potential of our approach with an accurate score difference estimator. Although \textit{SKSD-rg+GO} performs reasonably well, its run-time is \textbf{53x} longer than \textit{SKSD-rg+KE} as shown in Table \ref{tab: RBM GOF}.
Interestingly, \textit{SKSD-rg+GE} performs worse than KSD due to the significant under-estimation of the magnitude of $\bm{s}_q(\bm{x})$. Therefore, we omit this approach in the following ablation studies.
% On the other hand, the kernel smoothing approach (+KE) estimates the score difference as a whole, which suffers less from score magnitude mismatches. 
% \begin{table}[t]
% \caption{Test power and time consumption at 0.01 perturbation}
% \begin{tabular}{l|lll}
% \hline
%           & Test Power    & Opt. Time      & Speed-up     \\ \hline
% SKSD-rg+Ex & 0.95          & 0.04s          & 254x         \\ \hline
% SKSD-rg+KE & \textbf{0.67} & \textbf{0.19s} & \textbf{53x} \\
% SKSD-rg+GO & 0.45          & 10.15s         & 1x           \\ \hline
% \end{tabular}
% \label{tab: RBM GOF}
% \end{table}
\paragraph{Ablation studies}
%We further conduct 2 ablation studies to investigate (i) the quality of the active slices, and (ii) the effect of pruning/removing active slices. 
The first ablation study, with results shown in the middle panel in Figure \ref{fig: RBM GOF}, considers pruning the active slices at different pruning levels, where the {horizontal axis indicates the number of $\bm{r}$ slices used to construct the test statistic}. We observe that the null rejection rates of active slice methods peak with {pruning level 3}, indicating their ability to select the most important directions. Their performances decrease when more $\bm{r}$ are considered since, in practice, those less important directions introduce extra noise to the test statistic. On the other hand, \textit{SKSD-rg+GO} shows no pruning abilities due to its sensitivity to slice initialization. Remarkably, the final performance of \textit{SKSD-rg+GO} without pruning is still worse than \textit{SKSD-rg+KE} with pruning, showing the importance of finding 'good' instead of many 'average-quality' directions. Another advantage of pruning is to reduce the computational and memory costs from $O(MD)$ to $O(mD)$, where $m$ and $M$ are the number of pruned $\bm{r}$ and slice initializations, respectively ($m\ll M$).  

The second ablation study investigates the quality of the obtained slices either by gradient-based optimization or by the active slice approaches. Results are shown in the right panel of Figure \ref{fig: RBM GOF}, where the horizontal axis indicates the number of {training epochs}, and the numbers annotated in the legend ($3$ and $50$) indicate the pruning. We observe that the null rejection rate of \textit{SKSD-rg+KE+GO} starts to improve only after $100$ epochs, meaning that short run of \textit{GO} refinements are redundant due to the good quality of active slices. {The performance decrease of \textit{SKSD-rg+Ex+GO} is due to the over-fitting of GO to the training set}. The null rejection rate of \textit{SKSD-rg+GO} gradually increases with larger training epochs as expected. However, even after 100 epochs, the test power is still lower than active slices without any GO. 

In appendix \ref{appendix subsec: RBM experiment}, another ablation study also shows the advantages of good $\bm{r}$ compared to using random slices.

\subsection{Model learning: ICA}
\label{subsec: Training ICA}

\begin{table*}[t]
\centering
\caption{The test NLL of different dimensional ICA model}
\begin{tabular}{l|llllll}
\hline
Dimensions & SKSD-g+KE+GO      & SKSD-g+Ex+GO      & SKSD-g+GO      & SKSD-rg+GO      & LSD            & KSD             \\ \hline
10         & {7.93$\pm$0.31}  & {7.95$\pm$0.31}  & 8.06$\pm$0.33  & 10.03$\pm$0.61  & \textbf{7.42$\pm$0.31}  & \textbf{7.82$\pm$0.31}   \\
80         & \textbf{7.88$\pm$0.77}  & 15.17$\pm$0.97 & 19.03$\pm$1.06 & 62.53$\pm$0.92  & \textbf{6.26$\pm$1.49}  & 80.75$\pm$1.22  \\
100        & \textbf{6.93$\pm$1.36}  & 21.50$\pm$1.41 & 22.22$\pm$1.08 & 75.28$\pm$1.63  & 17.55$\pm$1.60 & 110.78$\pm$1.19 \\
150        & \textbf{11.67$\pm$2.46} & 27.37$\pm$3.04 & 21.63$\pm$3.27 & 107.25$\pm$1.93 & 32.15$\pm$3.75 & 180.47$\pm$1.91 \\ \hline
\end{tabular}
\label{tab: ICA NLL}
\end{table*}

We evaluate the performance of the active slice methods in model learning by training an independent component analysis (ICA) model, which has been extensively used to evaluate algorithms for training energy-based models \citep{gutmann2010noise,hyvarinen2005estimation,ceylan2018conditional,grathwohl2020cutting}. ICA follows a simple generative process: it first samples a $D$-dimensional random variable $\bm{z}$ from a non-Gaussian $p_z$ (we use multivariate-t), then transforms $\bm{z}$ to $\bm{x}=\bm{W}\bm{z}$ with a non-singular matrix $\bm{W}\in\mathbb{R}^{D\times D}$. The log-likelihood is {$\log p(\bm{x})=\log p_z(\bm{W}^{-1}\bm{x})+C$} where $C$ can be ignored if trained by minimizing Stein discrepancies. 
We follow \citet{grathwohl2020cutting,gong2021sliced} to sample $20000$ training and $5000$ test datapoints from a randomly initialized ICA model. The baselines considered include \textit{KSD}, \textit{SKSD-g+GO}, \textit{SKSD-rg+GO} and the state-of-the-art \textit{learned Stein discrepancy} (\textit{LSD}) \citep{grathwohl2020cutting}, where the test function is parametrized by a neural network. For active slice approaches, {one optimization epoch include the following two steps}: (i) finding active slices for both orthogonal basis $O_r$ and $\bm{g}_r$ at the beginning of the epoch, and (ii) {refining the $\bm{g}_r$ directions and the $\bm{W}$ parameters in an adversarial manner with $O_r$ fixed}.
\begin{table}
\caption{Test power and time consumption at 0.01 perturbation}
\begin{tabular}{l|lll}
\hline
          & Test Power    & Opt. Time      & Speed-up     \\ \hline
SKSD-rg+Ex & 0.95          & 0.04s          & 254x         \\ \hline
SKSD-rg+KE & \textbf{0.67} & \textbf{0.19s} & \textbf{53x} \\
SKSD-rg+GO & 0.45          & 10.15s         & 1x           \\ \hline
\end{tabular}
\label{tab: RBM GOF}
\end{table}
For \textit{SKSD-g+GO}, we fix basis $O_r=\bm{I}$ and only update $\bm{g}_r$ with \textit{GO}.
We refer to appendix \ref{appendix subsec: model learning ICA} for details on the setup and training procedure. 
\begin{figure}
    \centering
    \includegraphics[scale=0.19]{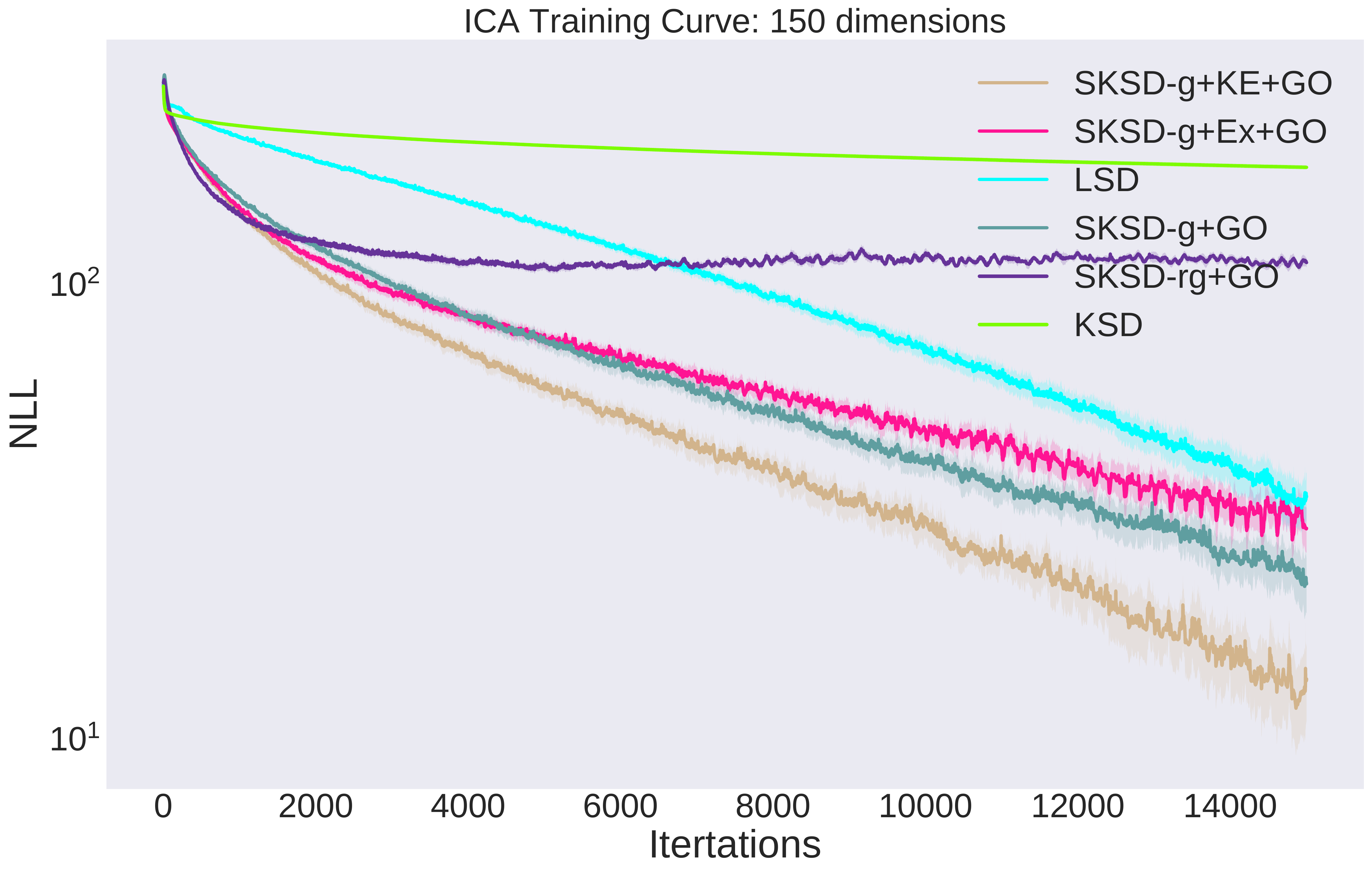}
    \caption{Training Curve of ICA model, where y-axes indicates the test NLL.}
    \label{fig: ICA training curve}
\end{figure}
We see from Figure \ref{fig: ICA training curve} that \textit{SKSD-g+KE+GO} converges significantly faster at $150$ dimensions than all baselines; moreover, it has much better NLL (Table \ref{tab: ICA NLL}). We argue this performance gain is due to the use of the better orthogonal basis $O_r$ found by the greedy algorithm, showing the advantages of better $O_r$ in model learning. On the other hand, the importance of orthogonality in $O_r$ is indicated by the poor performance of \textit{SKSD-rg+GO}, as gradient updates for $\bm{r}$ violate the orthogonality constraint. The goal of learning is to train the model to match the data distribution along every slicing direction, and the orthogonality constraint can help prevent the model from ignoring important slicing directions. 

Interestingly, \textit{SKSD-g+Ex+GO} performs worse than \textit{+KE+GO}. We hypothesize that this is because the \textit{+Ex+GO} approach often focuses on directions {with large discriminative power but with less useful learning signal } (see appendix \ref{appendix subsec: model learning ICA}). $LSD$ performs well in low dimensional problems. However, in high dimensional learning tasks it spends too much time on finding good test functions, which slows down the convergence significantly.

\section{Related Work}
\label{sec: Related work}
\paragraph{Active subspace method (ASM):} ASM is initially proposed as a dimensionality reduction method, which constructs a subspace with low-rank projectors \cite{constantine2014active} according to the subspace Poincar\'e inequality. \citet{zahm2020gradient} showed promising results on the application of ASM to approximating multivariate functions with lower dimensional ones. However, they only considered the subspace Poincar\'e inequality under Gaussian measures, and a generalization to a broader family of famlity is proposed by \citet{parente2020generalized}. Another closely related approach uses logarithmic Sobolev inequality instead to construct the active subspace \citep{zahm2018certified}, which can be interpreted as finding the optimal subspace to minimize a KL-divergence. It has shown successes in Bayesian inverse problems and particle inference \cite{chen2019projected}. \wg{However, as the ASM method is based on the eigen-decomposition of the sensitivity matrix, there is a potential limitation when the sensitivity matrix is estimated by Monte-Carlo method. We prove this limitation in appendix \ref{Appendix: Perturb}.}

\paragraph{Sliced discrepancies:} Existing examples of sliced discrepancies can be roughly divided into two groups. Most of them belong to the first group, and they use the slicing idea to improve computational efficiency. For example, sliced Wasserstein distance projects distributions onto one dimensional slices so that the corresponding distance has an analytic form \citep{kolouri2019generalized,deshpande2019max}. Sliced score matching uses Hutchinson's trick to avoid the expensive computation of the Hessian matrix \citep{song2020sliced}. The second group focuses on the curse-of-dimensionality issue which remains to be addressed. To the best of our knowledge, existing integral probability metrics in this category include \textit{SSD} \citep{gong2021sliced} and \textit{kernelized complete conditional Stein discrepancy} \citep[KCC-SD,][]{singhal2019kernelized}. The former is more general and requires less restrictive assumptions, while the latter requires samples from complete conditional distributions. Recent work has also investigated the statistical properties of sliced discrepancies \citep{nadjahi2020statistical}.

\section{Conclusion}
We have proposed the active slice method as a practical solution for searching good slices for \textit{SKSD}. We first prove that the validity of the kernelized discrepancy only requires finite number of random slices instead of optimal ones, giving us huge freedom to select slice directions. Then by analyzing the approximation quality of \textit{SSD} to \textit{SKSD}, we proposed to find active slices by optimizing surrogate optimization tasks. Experiments on high-dimensional GOF tests and ICA training showed the active slice method performed the best across a number of competitive baselines in terms of both test performance and run-time. Future research directions include better score difference estimation methods, non-linear generalizations of slice projections, and the application of the active slice method to other discrepancies. 
\bibliography{example_paper}

\begin{thebibliography}{8}
\providecommand{\natexlab}[1]{#1}
\providecommand{\url}[1]{\texttt{#1}}
\expandafter\ifx\csname urlstyle\endcsname\relax
  \providecommand{\doi}[1]{doi: #1}\else
  \providecommand{\doi}{doi: \begingroup \urlstyle{rm}\Url}\fi

\bibitem[Author(2021)]{anonymous}
Author, N.~N.
\newblock Suppressed for anonymity, 2021.

\bibitem[Duda et~al.(2000)Duda, Hart, and Stork]{DudaHart2nd}
Duda, R.~O., Hart, P.~E., and Stork, D.~G.
\newblock \emph{Pattern Classification}.
\newblock John Wiley and Sons, 2nd edition, 2000.

\bibitem[Kearns(1989)]{kearns89}
Kearns, M.~J.
\newblock \emph{Computational Complexity of Machine Learning}.
\newblock PhD thesis, Department of Computer Science, Harvard University, 1989.

\bibitem[Langley(2000)]{langley00}
Langley, P.
\newblock Crafting papers on machine learning.
\newblock In Langley, P. (ed.), \emph{Proceedings of the 17th International
  Conference on Machine Learning (ICML 2000)}, pp.\  1207--1216, Stanford, CA,
  2000. Morgan Kaufmann.

\bibitem[Michalski et~al.(1983)Michalski, Carbonell, and
  Mitchell]{MachineLearningI}
Michalski, R.~S., Carbonell, J.~G., and Mitchell, T.~M. (eds.).
\newblock \emph{Machine Learning: An Artificial Intelligence Approach, Vol. I}.
\newblock Tioga, Palo Alto, CA, 1983.

\bibitem[Mitchell(1980)]{mitchell80}
Mitchell, T.~M.
\newblock The need for biases in learning generalizations.
\newblock Technical report, Computer Science Department, Rutgers University,
  New Brunswick, MA, 1980.

\bibitem[Newell \& Rosenbloom(1981)Newell and Rosenbloom]{Newell81}
Newell, A. and Rosenbloom, P.~S.
\newblock Mechanisms of skill acquisition and the law of practice.
\newblock In Anderson, J.~R. (ed.), \emph{Cognitive Skills and Their
  Acquisition}, chapter~1, pp.\  1--51. Lawrence Erlbaum Associates, Inc.,
  Hillsdale, NJ, 1981.

\bibitem[Samuel(1959)]{Samuel59}
Samuel, A.~L.
\newblock Some studies in machine learning using the game of checkers.
\newblock \emph{IBM Journal of Research and Development}, 3\penalty0
  (3):\penalty0 211--229, 1959.

\end{thebibliography}


\begin{thebibliography}{45}
\providecommand{\natexlab}[1]{#1}
\providecommand{\url}[1]{\texttt{#1}}
\expandafter\ifx\csname urlstyle\endcsname\relax
  \providecommand{\doi}[1]{doi: #1}\else
  \providecommand{\doi}{doi: \begingroup \urlstyle{rm}\Url}\fi

\bibitem[Arcones \& Gine(1992)Arcones and Gine]{arcones1992bootstrap}
Arcones, M.~A. and Gine, E.
\newblock On the bootstrap of u and v statistics.
\newblock \emph{The Annals of Statistics}, pp.\  655--674, 1992.

\bibitem[Ceylan \& Gutmann(2018)Ceylan and Gutmann]{ceylan2018conditional}
Ceylan, C. and Gutmann, M.~U.
\newblock Conditional noise-contrastive estimation of unnormalised models.
\newblock In \emph{International Conference on Machine Learning}, pp.\
  726--734. PMLR, 2018.

\bibitem[Chen et~al.(2019)Chen, Wu, Chen, O'Leary-Roseberry, and
  Ghattas]{chen2019projected}
Chen, P., Wu, K., Chen, J., O'Leary-Roseberry, T., and Ghattas, O.
\newblock Projected stein variational newton: A fast and scalable bayesian
  inference method in high dimensions.
\newblock \emph{arXiv preprint arXiv:1901.08659}, 2019.

\bibitem[Chwialkowski et~al.(2016)Chwialkowski, Strathmann, and
  Gretton]{chwialkowski2016kernel}
Chwialkowski, K., Strathmann, H., and Gretton, A.
\newblock A kernel test of goodness of fit.
\newblock JMLR: Workshop and Conference Proceedings, 2016.

\bibitem[Comon(1994)]{comon1994independent}
Comon, P.
\newblock Independent component analysis, a new concept?
\newblock \emph{Signal processing}, 36\penalty0 (3):\penalty0 287--314, 1994.

\bibitem[Constantine et~al.(2014)Constantine, Dow, and
  Wang]{constantine2014active}
Constantine, P.~G., Dow, E., and Wang, Q.
\newblock Active subspace methods in theory and practice: applications to
  kriging surfaces.
\newblock \emph{SIAM Journal on Scientific Computing}, 36\penalty0
  (4):\penalty0 A1500--A1524, 2014.

\bibitem[Deshpande et~al.(2019)Deshpande, Hu, Sun, Pyrros, Siddiqui, Koyejo,
  Zhao, Forsyth, and Schwing]{deshpande2019max}
Deshpande, I., Hu, Y.-T., Sun, R., Pyrros, A., Siddiqui, N., Koyejo, S., Zhao,
  Z., Forsyth, D., and Schwing, A.~G.
\newblock Max-sliced wasserstein distance and its use for gans.
\newblock In \emph{Proceedings of the IEEE/CVF Conference on Computer Vision
  and Pattern Recognition}, pp.\  10648--10656, 2019.

\bibitem[Gong et~al.(2021)Gong, Li, and Hern{\'a}ndez-Lobato]{gong2021sliced}
Gong, W., Li, Y., and Hern{\'a}ndez-Lobato, J.~M.
\newblock Sliced kernelized stein discrepancy.
\newblock In \emph{International Conference on Learning Representations}, 2021.
\newblock URL \url{https://openreview.net/forum?id=t0TaKv0Gx6Z}.

\bibitem[Gorham \& Mackey(2015)Gorham and Mackey]{gorham2015sd}
Gorham, J. and Mackey, L.
\newblock Measuring sample quality with stein’s method.
\newblock In \emph{Advances in Neural Information Processing Systems}, pp.\
  226--234, 2015.

\bibitem[Gorham \& Mackey(2017)Gorham and Mackey]{gorham2017measuring}
Gorham, J. and Mackey, L.
\newblock Measuring sample quality with kernels.
\newblock In \emph{International Conference on Machine Learning}, pp.\
  1292--1301. PMLR, 2017.

\bibitem[Grathwohl et~al.(2020)Grathwohl, Wang, Jacobsen, Duvenaud, and
  Zemel]{grathwohl2020cutting}
Grathwohl, W., Wang, K.-C., Jacobsen, J.-H., Duvenaud, D., and Zemel, R.
\newblock Cutting out the middle-man: Training and evaluating energy-based
  models without sampling.
\newblock \emph{arXiv preprint arXiv:2002.05616}, 2020.

\bibitem[Gretton et~al.(2012)Gretton, Borgwardt, Rasch, Sch{\"o}lkopf, and
  Smola]{gretton2012kernel}
Gretton, A., Borgwardt, K.~M., Rasch, M.~J., Sch{\"o}lkopf, B., and Smola, A.
\newblock A kernel two-sample test.
\newblock \emph{The Journal of Machine Learning Research}, 13\penalty0
  (1):\penalty0 723--773, 2012.

\bibitem[Gutmann \& Hyv{\"a}rinen(2010)Gutmann and
  Hyv{\"a}rinen]{gutmann2010noise}
Gutmann, M. and Hyv{\"a}rinen, A.
\newblock Noise-contrastive estimation: A new estimation principle for
  unnormalized statistical models.
\newblock In \emph{Proceedings of the Thirteenth International Conference on
  Artificial Intelligence and Statistics}, pp.\  297--304. JMLR Workshop and
  Conference Proceedings, 2010.

\bibitem[Hinton \& Salakhutdinov(2006)Hinton and
  Salakhutdinov]{hinton2006reducing}
Hinton, G.~E. and Salakhutdinov, R.~R.
\newblock Reducing the dimensionality of data with neural networks.
\newblock \emph{science}, 313\penalty0 (5786):\penalty0 504--507, 2006.

\bibitem[Hoeffding(1992)]{hoeffding1992class}
Hoeffding, W.
\newblock A class of statistics with asymptotically normal distribution.
\newblock In \emph{Breakthroughs in statistics}, pp.\  308--334. Springer,
  1992.

\bibitem[Hu et~al.(2018)Hu, Chen, Sun, Bai, Ye, and Cheng]{hu2018stein}
Hu, T., Chen, Z., Sun, H., Bai, J., Ye, M., and Cheng, G.
\newblock Stein neural sampler.
\newblock \emph{arXiv preprint arXiv:1810.03545}, 2018.

\bibitem[Huggins \& Mackey(2018)Huggins and Mackey]{huggins2018random}
Huggins, J. and Mackey, L.
\newblock Random feature stein discrepancies.
\newblock In \emph{Advances in Neural Information Processing Systems}, pp.\
  1899--1909, 2018.

\bibitem[Huskova \& Janssen(1993)Huskova and Janssen]{huskova1993consistency}
Huskova, M. and Janssen, P.
\newblock Consistency of the generalized bootstrap for degenerate u-statistics.
\newblock \emph{The Annals of Statistics}, pp.\  1811--1823, 1993.

\bibitem[Hyv{\"a}rinen \& Dayan(2005)Hyv{\"a}rinen and
  Dayan]{hyvarinen2005estimation}
Hyv{\"a}rinen, A. and Dayan, P.
\newblock Estimation of non-normalized statistical models by score matching.
\newblock \emph{Journal of Machine Learning Research}, 6\penalty0 (4), 2005.

\bibitem[Jitkrittum et~al.(2017)Jitkrittum, Xu, Szab{\'o}, Fukumizu, and
  Gretton]{jitkrittum2017linear}
Jitkrittum, W., Xu, W., Szab{\'o}, Z., Fukumizu, K., and Gretton, A.
\newblock A linear-time kernel goodness-of-fit test.
\newblock In \emph{Advances in Neural Information Processing Systems}, pp.\
  262--271, 2017.

\bibitem[Kingma \& Ba(2014)Kingma and Ba]{kingma2014adam}
Kingma, D.~P. and Ba, J.
\newblock Adam: A method for stochastic optimization.
\newblock \emph{arXiv preprint arXiv:1412.6980}, 2014.

\bibitem[Kolouri et~al.(2019)Kolouri, Nadjahi, Simsekli, Badeau, and
  Rohde]{kolouri2019generalized}
Kolouri, S., Nadjahi, K., Simsekli, U., Badeau, R., and Rohde, G.~K.
\newblock Generalized sliced wasserstein distances.
\newblock \emph{arXiv preprint arXiv:1902.00434}, 2019.

\bibitem[Li \& Turner(2017)Li and Turner]{li2017gradient}
Li, Y. and Turner, R.~E.
\newblock Gradient estimators for implicit models.
\newblock \emph{arXiv preprint arXiv:1705.07107}, 2017.

\bibitem[Liu \& Wang(2016)Liu and Wang]{liu2016stein}
Liu, Q. and Wang, D.
\newblock Stein variational gradient descent: A general purpose bayesian
  inference algorithm.
\newblock \emph{arXiv preprint arXiv:1608.04471}, 2016.

\bibitem[Liu et~al.(2016)Liu, Lee, and Jordan]{liu2016kernelized}
Liu, Q., Lee, J., and Jordan, M.
\newblock A kernelized stein discrepancy for goodness-of-fit tests.
\newblock In \emph{International conference on machine learning}, pp.\
  276--284, 2016.

\bibitem[Mityagin(2015)]{mityagin2015zero}
Mityagin, B.
\newblock The zero set of a real analytic function.
\newblock \emph{arXiv preprint arXiv:1512.07276}, 2015.

\bibitem[Nadjahi et~al.(2020)Nadjahi, Durmus, Chizat, Kolouri, Shahrampour, and
  {\c{S}}im{\c{s}}ekli]{nadjahi2020statistical}
Nadjahi, K., Durmus, A., Chizat, L., Kolouri, S., Shahrampour, S., and
  {\c{S}}im{\c{s}}ekli, U.
\newblock Statistical and topological properties of sliced probability
  divergences.
\newblock \emph{arXiv preprint arXiv:2003.05783}, 2020.

\bibitem[Parente et~al.(2020)Parente, Wallin, Wohlmuth,
  et~al.]{parente2020generalized}
Parente, M.~T., Wallin, J., Wohlmuth, B., et~al.
\newblock Generalized bounds for active subspaces.
\newblock \emph{Electronic Journal of Statistics}, 14\penalty0 (1):\penalty0
  917--943, 2020.

\bibitem[Pu et~al.(2017)Pu, Gan, Henao, Li, Han, and Carin]{pu2017vae}
Pu, Y., Gan, Z., Henao, R., Li, C., Han, S., and Carin, L.
\newblock Vae learning via stein variational gradient descent.
\newblock \emph{arXiv preprint arXiv:1704.05155}, 2017.

\bibitem[Saibaba et~al.(2021)Saibaba, Hart, and van
  Bloemen~Waanders]{saibaba2021randomized}
Saibaba, A.~K., Hart, J., and van Bloemen~Waanders, B.
\newblock Randomized algorithms for generalized singular value decomposition
  with application to sensitivity analysis.
\newblock \emph{Numerical Linear Algebra with Applications}, pp.\  e2364, 2021.

\bibitem[Sameh \& Tong(2000)Sameh and Tong]{sameh2000trace}
Sameh, A. and Tong, Z.
\newblock The trace minimization method for the symmetric generalized
  eigenvalue problem.
\newblock \emph{Journal of computational and applied mathematics}, 123\penalty0
  (1-2):\penalty0 155--175, 2000.

\bibitem[Serfling(2009)]{serfling2009approximation}
Serfling, R.~J.
\newblock \emph{Approximation theorems of mathematical statistics}, volume 162.
\newblock John Wiley \& Sons, 2009.

\bibitem[Shi et~al.(2018)Shi, Sun, and Zhu]{shi2018spectral}
Shi, J., Sun, S., and Zhu, J.
\newblock A spectral approach to gradient estimation for implicit
  distributions.
\newblock \emph{arXiv preprint arXiv:1806.02925}, 2018.

\bibitem[Singhal et~al.(2019)Singhal, Han, Lahlou, and
  Ranganath]{singhal2019kernelized}
Singhal, R., Han, X., Lahlou, S., and Ranganath, R.
\newblock Kernelized complete conditional stein discrepancy.
\newblock \emph{arXiv preprint arXiv:1904.04478}, 2019.

\bibitem[Song \& Ermon(2019)Song and Ermon]{song2019generative}
Song, Y. and Ermon, S.
\newblock Generative modeling by estimating gradients of the data distribution.
\newblock In \emph{Advances in Neural Information Processing Systems}, pp.\
  11918--11930, 2019.

\bibitem[Song et~al.(2020)Song, Garg, Shi, and Ermon]{song2020sliced}
Song, Y., Garg, S., Shi, J., and Ermon, S.
\newblock Sliced score matching: A scalable approach to density and score
  estimation.
\newblock In \emph{Uncertainty in Artificial Intelligence}, pp.\  574--584.
  PMLR, 2020.

\bibitem[Sriperumbudur et~al.(2017)Sriperumbudur, Fukumizu, Gretton,
  Hyv{\"a}rinen, and Kumar]{sriperumbudur2017density}
Sriperumbudur, B., Fukumizu, K., Gretton, A., Hyv{\"a}rinen, A., and Kumar, R.
\newblock Density estimation in infinite dimensional exponential families.
\newblock \emph{The Journal of Machine Learning Research}, 18\penalty0
  (1):\penalty0 1830--1888, 2017.

\bibitem[Sriperumbudur et~al.(2011)Sriperumbudur, Fukumizu, and
  Lanckriet]{sriperumbudur2011universality}
Sriperumbudur, B.~K., Fukumizu, K., and Lanckriet, G.~R.
\newblock Universality, characteristic kernels and rkhs embedding of measures.
\newblock \emph{Journal of Machine Learning Research}, 12\penalty0 (7), 2011.

\bibitem[Sutherland et~al.(2018)Sutherland, Strathmann, Arbel, and
  Gretton]{sutherland2018efficient}
Sutherland, D., Strathmann, H., Arbel, M., and Gretton, A.
\newblock Efficient and principled score estimation with nystr{\"o}m kernel
  exponential families.
\newblock In \emph{International Conference on Artificial Intelligence and
  Statistics}, pp.\  652--660. PMLR, 2018.

\bibitem[Welling et~al.()Welling, Rosen-Zvi, and
  Hinton]{welling2004exponential}
Welling, M., Rosen-Zvi, M., and Hinton, G.~E.
\newblock Exponential family harmoniums with an application to information
  retrieval.
\newblock Citeseer.

\bibitem[Wenliang et~al.(2019)Wenliang, Sutherland, Strathmann, and
  Gretton]{wenliang2019learning}
Wenliang, L., Sutherland, D., Strathmann, H., and Gretton, A.
\newblock Learning deep kernels for exponential family densities.
\newblock In \emph{International Conference on Machine Learning}, pp.\
  6737--6746. PMLR, 2019.

\bibitem[Yu et~al.(2015)Yu, Wang, and Samworth]{yu2015useful}
Yu, Y., Wang, T., and Samworth, R.~J.
\newblock A useful variant of the davis--kahan theorem for statisticians.
\newblock \emph{Biometrika}, 102\penalty0 (2):\penalty0 315--323, 2015.

\bibitem[Zahm et~al.(2018)Zahm, Cui, Law, Spantini, and
  Marzouk]{zahm2018certified}
Zahm, O., Cui, T., Law, K., Spantini, A., and Marzouk, Y.
\newblock Certified dimension reduction in nonlinear bayesian inverse problems.
\newblock \emph{arXiv preprint arXiv:1807.03712}, 2018.

\bibitem[Zahm et~al.(2020)Zahm, Constantine, Prieur, and
  Marzouk]{zahm2020gradient}
Zahm, O., Constantine, P.~G., Prieur, C., and Marzouk, Y.~M.
\newblock Gradient-based dimension reduction of multivariate vector-valued
  functions.
\newblock \emph{SIAM Journal on Scientific Computing}, 42\penalty0
  (1):\penalty0 A534--A558, 2020.

\bibitem[Zhou et~al.(2020)Zhou, Shi, and Zhu]{zhou2020nonparametric}
Zhou, Y., Shi, J., and Zhu, J.
\newblock Nonparametric score estimators.
\newblock \emph{arXiv preprint arXiv:2005.10099}, 2020.

\end{thebibliography}
\bibliographystyle{icml2021}

\clearpage
\appendix
\section{Terms and Notations}
\label{appendix: Notations}
For the clarity of the paper, we give a summary of the commonly used notations in the main text and proof. \\
\textbf{Symbols:}\\
\begin{tabular}{p{1.75cm}p{5.75cm}}
$\bm{s}_p(\bm{x})$ & $\nabla_{\bm{x}}\log p(\bm{x})$\\
$s_p^r(\bm{x})$ & Projected score function $\nabla_{\bm{x}}\log p(\bm{x})^T\bm{r}$\\
$\mathcal{X}$ & A subset of $\mathbb{R}^D$ \\
$\mathcal{K}$ & A subset of $\mathbb{R}$.\\
$k_{rg_r}$ & kernel function $k:\mathcal{K}\times\mathcal{K}\rightarrow \mathcal{R}$\\
$\mathcal{H}_{rg_r}$ & Induced RKHS by the kernel $k_{rg_r}$.\\
$||\cdot||_{\mathcal{H}_{rg_r}}$ & RKHS norm of $\mathcal{H}_{rg_r}$\\
$\bm{g}_r$ & Input projection direction (e.g. $\bm{x}^T\bm{g}_r$) for corresponding $\bm{r}$.\\
$\bm{r}$ & Score projection direction (e.g. $s_p^r(\bm{x})=\bm{s}_p(\bm{x})^T\bm{r}$)\\
$S_{\text{max}_{g_r}}$ & \textit{maxSSD-g} (Eq.\ref{eq: maxSSD-g}).\\
$S_{{g_r}}$ & \textit{SSD-g}, i.e. $S_{\text{max}_{g_r}}$ (Eq.\ref{eq: maxSSD-g}) without $\sup_{\bm{g}_r}$. But with summation of $O_r$.\\
$S_{{rg_r}}$ & \textit{SSD-rg}, i.e. $S_{\text{max}_{g_r}}$ (Eq.\ref{eq: maxSSD-g}) without $\sup_{\bm{g}_r}$ and summation of $O_r$. Instead, we use specific $\bm{r}$.\\
$SK_{\text{max}_{g_r}}$ &\textit{maxSKSD-g}. The kernelized verison of $S_{\text{max}_{g_r}}$\\
$SK_{{g_r}}$ &\textit{SKSD-g}. The kernelized verison of $S_{{g_r}}$\\
$SK_{{rg_r}}$ &\textit{SKSD-rg}. The kernelized verison of $S_{{rg_r}}$\\
$\text{PSD}$ & Projected Stein discrepancy (Eq.\ref{eq: PSD})\\
$\text{PSD}_{r}$ & Projected Stein discrepancy (Eq.\ref{eq: PSD}) without summation $O_r$ and  use specific $\bm{r}$ instead.\\
$f_r^*$ & Optimal test function for $\text{PSD}$. $f_r^*(\bm{x})\propto s_p^r(\bm{x})-s_q^r(\bm{x})$\\
$h_{rg_r}^*$ & Optimal test function for $S_{g_r}$ with specific $\bm{r}$ and $\bm{g}_r$, defined in Eq.\ref{eq: optimal test function maxSSD-g}.\\
$^*$& This indicates the optimal test function (e.g. $f_r^*$)\\
$C_{\text{sup}}$ & Supremum of Poincar\'e constant defined in assumption 6. 
\end{tabular}
% \textbf{Equations:}\\
% Stein discrepancy (Eq.\ref{eq: SD}):
% \[
% D_{SD}(q,p)=\sup_{\bm{f}\in\mathcal{F}_q}\mathbb{E}_q[\mathcal{A}_p\bm{f}(\bm{x})]
% \]
% Optimal test functions for Stein discrepancy:
% \[
% \bm{f}^*(\bm{x})\propto (\bm{s}_p(\bm{x})-\bm{s}_q(\bm{x}))
% \]
% Max sliced Stein discrepancy (requires the supremum of $\bm{g}_r$ and summation of $O_r$):
% \[
% \begin{split}
%     S_{\text{max}_{g_r}}(q,p)&=\sum_{\bm{r}\in O_r}{\sup_{\substack{h_{rg_r}\in\mathcal{F}_q\\\bm{g}_r\in\mathbb{S}^{D-1}}}{\mathbb{E}_q[{s}^r_p(\bm{x})h_{rg_r}(\bm{x}^T\bm{g}_r)+}}\\
%     &{{\bm{r}^T\bm{g}_r\nabla_{\bm{x}^T\bm{g}_r}h_{rg_r}(\bm{x}^T\bm{g}_r)]}}.
%     \label{eq: maxSSD-g}
% \end{split}
% \]
\subsection{``Sub-optimal" variants of SSD}
\wg{
For the ease of the analysis, we want to define the notations without the $\sup$ opeartor over the slice directions $\bm{r}$, $\bm{g}_r$. Here, we define \emph{SSD-g} ($S_{g_r}$) as the \emph{maxSSD-g} ($S_{\text{max}_{g_r}}$ in Eq.\ref{eq: maxSSD-g}) without the $\sup_{\bm{g}_r}$.
\begin{equation}
    \begin{split}
        S_{g_r} &= \sum_{\bm{r}\in O_r}\sup_{h_{rg_r}\in\mathcal{F}_q}\mathbb{E}_{q}[s_p^r(\bm{x})h_{rg_r}(\bm{x}^T\bm{g}_r)+\\
        &\bm{r}^T\bm{g}_r\nabla_{\bm{x}^T\bm{g}_r}h_{rg_r}(\bm{x}^T\bm{g}_r)]
    \end{split}
\end{equation}
Similarly, we define \emph{SSD-rg} ($S_{rg_r}$) as \emph{maxSSD-rg} ($S_{\text{max}_{rg_r}}$ in Eq.\ref{eq: maxSSD-rg}) without $\sup_{\bm{r},\bm{g}_r}$:
\begin{equation}
    \begin{split}
        S_{rg_r} &= \sup_{h_{rg_r}\in\mathcal{F}_q}\mathbb{E}_{q}[s_p^r(\bm{x})h_{rg_r}(\bm{x}^T\bm{g}_r)+\\
        &\bm{r}^T\bm{g}_r\nabla_{\bm{x}^T\bm{g}_r}h_{rg_r}(\bm{x}^T\bm{g}_r)]
    \end{split}
    \label{eq: SSD-rg}
\end{equation}
As for each of the above "optimal" discrepancies, it has the corresponding kernelized version. Therefore, we need to define their "un-optimal" version as well. We define \emph{SKSD-g} ($SK_{g_r}$) as \emph{maxSKSD-g} ($SK_{\text{max}_{g_r}}$ in Eq.\ref{eq: maxSKSD-g}) as 
\begin{equation}
    SK_{{g_r}}=\sum_{\bm{r}\in O_r}{{||\mathbb{E}_q[\xi_{p,r,g_r}(\bm{x})]||^2_{\mathcal{H}_{rg_r}}}}\,,
    \label{eq: SKSD-g}
\end{equation}
Similarly, we define \emph{SKSD-rg} ($SK_{rg_r}$) as \emph{maxSKSD-rg} ($SK_{\text{max}_{rg_r}}$ in Eq.\ref{eq: maxSKSD-rg}) as 
\begin{equation}
    SK_{rg_r} = ||\mathbb{E}_q[\xi_{p,r,g_r}(\bm{x})]||_{\mathcal{H}_{rg_r}}^2
    \label{eq: SKSD-rg}
\end{equation}
}

\section{Assumptions and Definitions}
\label{appendix: definition and assumption}
\begin{definition}[Inner product in Hilbert space]
We denote the algebraic space $\mathbb{R}^D$ refers to a parameter space of dimension $D$. The Borel sets of $\mathbb{R}^D$ is denoted as $\mathcal{B}(\mathbb{R}^D)$, and we let $\mu(x)$ be a probability measure on $\bm{x}$. We define
\begin{equation}
    \mathcal{H}_{\mu}=L^2(\mathbb{R}^D,\mathcal{B}(\mathbb{R}^D),\mu)
\end{equation}
as the Hilbert space which contains all the measurable functions $f:\mathbb{R}^D\rightarrow \mathbb{R}$, such that $||f||_{\mathcal{H}_{\mu}}\leq \infty$, where we define inner product $\langle\cdot,\cdot\rangle_{\mathcal{H}_{\mu}}$ to be 
\begin{equation}
    \langle f,g\rangle_{\mathcal{H}_{\mu}}=\int {f(\bm{x})g(\bm{x})d\mu(x)}
\end{equation}
for all $f,g\in \mathcal{H}_{\mu}$
\end{definition}

\begin{definition}{(\textbf{Stein Class} \citep{liu2016kernelized})}
Assume distribution $q$ has continuous and differentiable density $q(\bm{x})$. A function $f$ defined on the domain $\mathcal{X}\subseteq \mathbb{R}^D$, $f:\mathcal{X}\rightarrow \mathbb{R}$ is in the \textbf{{Stein class} of $q$} if {$f$ is smooth} and satisfies 
\begin{equation}
    \int_{\mathcal{X}}{\nabla_{x}(f(\bm{x})q(\bm{x}))d\bm{x}}=0
\end{equation}
\label{def: Stein class}
\end{definition}
We call a function $f(\bm{x})\in \mathcal{F}_q$ if $f$ belongs to the Stein class of $q$. We say vector-valued function $\bm{f}(\bm{x}):\mathcal{X}\subseteq \mathbb{R}^D\rightarrow\mathbb{R}^m\in\mathcal{F}_q$ if each component of $\bm{f}$ belongs to the Stein class of $q$.  
\begin{definition}[Stein Identity]
Assume $q$ is a smooth density satisfied assumption 1 , then we have
\begin{equation}
    \mathbb{E}_{q}\left[s_{q}(x) f(x)^{T}+\nabla f(x)\right]=0
\label{eq: Stein Identity}
\end{equation}
for any functions $f: \mathcal{X}\subseteq\mathbb{R}^{D} \rightarrow \mathbb{R}^{D}$ in Stein class of $q$.
\end{definition}

We can easily see that the above holds true for $\mathcal{X}=\mathbb{R}^D$ if 
\begin{equation}
    \lim_{||\bm{x}||\rightarrow\infty}{q(\bm{x})f(\bm{x})=0}
\end{equation}

\paragraph{Assumption 1}\label{assumption 1}(Properties of densities) Assume the two probability distributions $p$, $q$ has continuous differentiable density $p(\bm{x})$, $q(\bm{x})$ supported on $\mathcal{X}\subseteq \mathbb{R}^D$, such that the induced set $\mathcal{K}=\{y\in\mathbb{R}|y=\bm{x}^T\bm{g},||\bm{g}||^2=1, \bm{x}\in\mathcal{X}\}$ is \textit{locally compact Hausdorff} (LCH) for all possible $\bm{g}\in\mathbb{S}^{D-1}$.  If $\mathcal{X}=\mathbb{R}^D$, then the density $q$ satisfies: %$\lim_{||\bm{x}||\rightarrow \infty}{p(\bm{x})= 0}$ and 
$\lim_{||\bm{x}||\rightarrow \infty}{q(\bm{x})= 0}$. If $\mathcal{X}\subset\mathbb{R}^D$ is compact, then $q(\bm{x})=0$ at boundary $\partial \mathcal{X}$. 

\paragraph{Assumption 2}\label{assumption 2}(Regularity of score functions) Denote the score function of $p(\bm{x})$ as $\bm{s}_p(\bm{x})=\nabla_{\bm{x}}\log p(\bm{x}) \in \mathbb{R}^D$ and score function of $q(\bm{x})$ accordingly. Assume the score functions are bounded continuous differentiable functions and satisfying
\begin{equation}
    \begin{split}
        &\int_{\mathcal{X}}{q(\bm{x})|(s_p(\bm{x})-s_q(\bm{x}))^T\bm{r}|d\bm{x}}<\infty\\
        &\int_{\mathcal{X}}{q(\bm{x})||(s_p(\bm{x})-s_q(\bm{x}))^T\bm{r}||^2d\bm{x}}<\infty
    \end{split}
    \label{eq: assumption 2}
\end{equation}
for all $\bm{r}$
where $\bm{r}\in \mathbb{S}^{D-1}$.

\paragraph{Assumption 3}{\label{assumption 3}}(Test functions) Assume the test function $h_{rg_r}:\mathcal{K}\subseteq\mathbb{R}\rightarrow \mathbb{R}$ is smooth and belongs to the Stein class of $q$. Specifically, if with assumption \hyperref[assumption 1]{1}, we only requires $h_{rg_r}$ to be a bounded continuous function. Similarly, we assume this also holds for PSD (eq.\ref{eq: PSD}) test function $f_r(\bm{x})$.
% \paragraph{Assumption 4}{\label{assumption 4}}(Bounded Bonditional Expectation) Define 
% \begin{equation}
%     I_{q,p}=q(\bm{x})(s_p(\bm{x})-s_q(\bm{x}))^T\bm{r}
% \end{equation}
% We assume the Radon transformation of $I_{q,p}$, $\mathcal{R}[I_{q,p}](l,\bm{g})$ is bounded for all $\bm{g}$, where $\bm{g}$ is sampled from a uniform distribution over a unit ball $\mathbb{S}^{D-1}$. Namely, $||\mathcal{R}[I_{q,p}](l,\bm{g})||_{\infty}<\infty$

\paragraph{Assumption 4}{\label{assumption 4}}(Bounded Conditional Expectation) Define 
\begin{equation}
    h^*_{rg_r}(y_d)=\mathbb{E}_{q_{G_r}(\bm{y}_{-d}|y_d)}[(s_p^r(\bm{G}^{-1}_{r}\bm{y})-s_q^r(\bm{G}^{-1}_r\bm{y}))]
\end{equation}
as in proposition \ref{prop: optimality of maxSSD-g}. We assume $h^*_{rg_r}$ is uniformly bounded for all possible $\bm{g}_r\in\mathbb{S}^{D-1}$.

\paragraph{Assumption 5}{\label{assumption 5}}(universal kernel): We assume the kernel $k_{rg}:\mathcal{K}\times\mathcal{K}\rightarrow\mathbb{R}$ is bounded and $c_0-$universal.
\paragraph{Assumption 6}{\label{assumption 6}}(Real analytic translation invariant kernel): We assume the kernel is translation invariant $k(x,y)=\phi(x-y):\mathcal{K}\rightarrow \mathbb{R}$ and $\phi$ is a real analytic function. Additionally, we assume if $k(cx,cy)=k'(x,y)$ for a constant $c>0$ where $k'$ is also a $c_0-$universal kernel. For example, \textit{radial basis kernel function} (RBF) and \textit{inverse multiquadric} (IMQ) kernel satisfy these assumptions.

\paragraph{Assumption 7}{\label{assumption 7}}(Log-concave probabilities) Assume a probability distribution $q$ with density function such that $q(\bm{x})=\exp(-V(\bm{x}))$, where $V(\bm{x})$ is a convex function.
% \paragraph{Assumption 6}{\label{assumption 6}}(Locally Lipschitz) Assume the projected score function difference $S_p^r(\bm{x})-S_q^r(\bm{x})$ is locally Lipschitz.
\paragraph{Assumption 8}{\label{assumption 8}}(Existence of supremum of Poincar\'e constant). For the Poincar\'e constant defined in lemma \ref{lemma: subspace poincare inequality}, the essential supremum exists $C_{ess,\bm{G}}=\text{ess sup}_{y_d}C_{y_d}<\infty$ and also the $C_{sup}=\sup_{\bm{G}}C_{ess,\bm{G}}<\infty$ exists over all possible orthogonal matrix $\bm{G}$.

% The assumption 1-4 are used for defining the \textit{SSD} family. Assumptions 5-6 are used for subspace Poincar\'e inequality. Assumption 7 is used for the validity of \textit{maxSKSD-g} and \textit{maxSKSD-rg}. In addition, if the kernel $k_{rg}$ is defined in the \textit{locally compact Hausdorff} (LCH) space (e.g. $\mathbb{R}$), the resulting RKHS induced by $k_{rg}$ is $c_0-$universal and $L_p-$universal\citep{sriperumbudur2011universality}. Assumption 8 is used to analyze the real analyticity of \textit{maxSKSD-g} w.r.t. $\bm{g}_r$. 

% kernel is universal, RKHS is dense w.r.t some norm. $c_0-\text{universal}\Rightarrow \text{induced RKHS dense in } C_{0}(\mathcal{K})$ w.r.t uniform norm
\section{Detailed Background}
\label{appendix: Detailed Background}
\subsection{Stein Discrepancy}
\wg{Assume we have two differentiable probability density functions $q(\bm{x})$ and $p(\bm{x})$ where $\bm{x}\in\mathcal{X}\subset \mathbb{R}^D$. We further define a test function $\bm{f}:\mathcal{X}\rightarrow \mathbb{R}^D$ and a suitable test function family $\mathcal{F}_q$ called \emph{Stein's class of q}. Recall the Stein operator (Eq.\ref{eq: Stein operator}) is defined as 
\begin{equation}
    \mathcal{A}_p\bm{f}(\bm{x}) = \bm{s}_p(\bm{x})^T\bm{f}(\bm{x})+\nabla_{\bm{x}}^T]\bm{f}(\bm{x})
\end{equation}
The function family $\mathcal{F}_q$ is defined as 
\begin{equation}
    \mathcal{F}_{q}=\left\{\bm{f}:\mathcal{X}\rightarrow \mathbb{R}^D\;|\;\mathbb{E}_q[\mathcal{A}_q\bm{f}] = 0 \right\}
\end{equation}
This function space can be quite general. For example, if $\mathcal{X}=\mathbb{R}^D$, we only require $\bm{f}$ to be differentiable and vanishing at infinity. With all the notations, \emph{Stein discrepancy} is defined as follows:
\begin{equation}
    D_{SD}(q,p) = \sup_{\bm{f}\in\mathcal{F}_q}\mathbb{E}_q[\mathcal{A}_p\bm{f}(\bm{x})]
\end{equation}
which can be proved to be a valid discrepancy \citep{gorham2017measuring}. Stein discrepancy has been shown to be closely related to \emph{Fisher discrepancy} defined as 
\begin{equation}
    D_{F}(q,p) = \mathbb{E}_{q}{||\bm{s}_p(\bm{x})-\bm{s}_q(\bm{x})||^2_2}
\end{equation}
Indeed, \citet{hu2018stein} shows that the optimal test function for \emph{Stein discrepancy} has the form $\bm{f}^*(\bm{x})\propto \bm{s}_p(\bm{x})-\bm{s}_q(\bm{x})$. By substitution, we can show \emph{Stein discrepancy} is equivalent to \emph{Fisher divergence} up to a multiplicative constant.

Unfortunately, the score difference $\bm{s}_p(\bm{x})-\bm{s}_q(\bm{x})$ may be intractable in practice, making SD intractable as a consequence. Thus, \citet{liu2016kernelized,chwialkowski2016kernel} propose an variant of SD by restricting $\mathcal{F}_q$ to be a unit ball inside an RKHS $\mathcal{H}_k$ induced by a $c_0-$universal kernel $k:\mathcal{X}\times\mathcal{X}\rightarrow \mathbb{R}$. By using the reproducing properties, they propose \emph{kernelized Stein discrepancy} as 
\begin{equation}
\begin{split}
        D^2(q,p) &= ||\mathbb{E}_q[\bm{s}_p(\bm{x})k(\bm{x},\cdot)+\nabla_{\bm{x}}k(\bm{x},\cdot)]||^2_{\mathcal{H}_k}\\
        &=\mathbb{E}_{\bm{x},\bm{x'}\sim q}[u_p(\bm{x},\bm{x}')]
\end{split}
\end{equation}
where $u_p(\bm{x},\bm{x}')$ is 
\begin{equation}
    \begin{split}
        &u_p(\bm{x},\bm{x}') = \bm{s}_p(\bm{x})^Tk(\bm{x},\bm{x}')\bm{s}_p(\bm{x}')+\bm{s}_p(\bm{x})^T\nabla_{\bm{x}'}k(\bm{x},\bm{x}')\\
        +&\bm{s}_p(\bm{x}')^T\nabla_{\bm{x}}k(\bm{x},\bm{x}')+\nabla^2_{\bm{x},\bm{x}'}k(\bm{x},\bm{x}')
    \end{split}
\end{equation}
and $\bm{x}$, $\bm{x}'$ are i.i.d. samples from $q$.

Due to its tractability, it has been extensively used in statistical test e.g. GOF test \citet{liu2016kernelized,chwialkowski2016kernel,huggins2018random,jitkrittum2017linear}. However, recent work demonstrate KSD suffers from the curse-of-dimensionality problem \citet{gong2021sliced,huggins2018random,chwialkowski2016kernel}. One potential fix is to use another variant called \emph{sliced kernelized Stein discrepancy}.
}
\subsection{Sliced Kernelized Stein Discrepancy}
In this section, we give a more detailed introduction to sliced kernelized Stein discrepancy (SKSD). 
Recall the definition of Stein discrepancy:
\begin{equation}
    D_{SD}(q,p)=\sup_{\bm{f}\in\mathcal{F}_q}\mathbb{E}_q[\bm{s}_p^T(\bm{x})\bm{f}(\bm{x})+\nabla_{\bm{x}}^T\bm{f}(\bm{x})]
\end{equation}
In the original paper of \cite{gong2021sliced}, they argue that the curse of dimensionality comes from two sources: (i) the high dimensionality of the score function $\bm{s}_p:\mathcal{X}\subseteq\mathbb{R}^D\rightarrow \mathbb{R}^D$ and (ii) the test function input $\bm{x}\in\mathcal{X}\subseteq \mathbb{R}^D$. Therefore, authors proposed two slice directions $\bm{r}$, $\bm{g}$ to project $\bm{s}_p$ and $\bm{x}$ respectively. However, this projection is equivalent to throwing away most of the information possessed by $\bm{s}_p$ and $\bm{x}$. To tackle this problem, authors proposed the first member of the SSD family by considering over all possible directions of $\bm{r}$ and $\bm{g}$ (a distribution over $\bm{r}\sim p_r$, $\bm{g}\sim p_g$), called \textit{integrated sliced Stein discrepancy}:
\begin{equation}
\begin{split}
    S(q,p)&=\mathbb{E}_{p_r,p_g}\left[
    \sup_{h_{rg}\in\mathcal{F}_q}{\mathbb{E}_q[s^r_p(\bm{x})h_{rg}(\bm{x}^T\bm{g})+}\right.\\
    &\left.{\bm{r}^T\bm{g}\nabla_{\bm{x}^T\bm{g}}h_{rg}(\bm{x}^T\bm{g})]}
    \vphantom{\sup_{asd}}\right].
\end{split}
    \label{eq: integrated SSD}
\end{equation}
where $h_{rg_r}$ is the test function. 
Although it is theoretically valid (Theorem 1 in\cite{gong2021sliced}), its practical useage is limited by the intractability of the integral over $p_r$, $p_g$ and the optimal test function $h_{rg}$. Surprisingly, authors show that the integral over $\bm{r}$, $\bm{g}$ is not necessary for discrepancy validity. \wg{They achieved this in two steps.

The first step is to replace the expectation w.r.t. $\bm{r}$ by a finite summation over orthogonal basis. The author showed that this is a valid discrepancy, called \emph{orthogonal sliced Stein discrepancy} defined as
\begin{equation}
    \begin{split}
    S_O(q,p)&=\sum_{\bm{r}\in O_r}\mathbb{E}_{p_g}\left[
    \sup_{h_{rg}\in\mathcal{F}_q}{\mathbb{E}_q[s^r_p(\bm{x})h_{rg}(\bm{x}^T\bm{g})+}\right.\\
    &\left.{\bm{r}^T\bm{g}\nabla_{\bm{x}^T\bm{g}}h_{rg}(\bm{x}^T\bm{g})]}
    \vphantom{\sup_{asd}}\right].
\end{split}
    \label{eq: orthogonal SSD}
\end{equation} 
where $O_r$ is an orthogonal basis (e.g. one-hot vectors). 
The next step is to get rid of the expectation w.r.t. $\bm{g}$ by a supremum operator. This is called \emph{maxSSD-g}, which is defined as Eq.\ref{eq: maxSSD-g} in the main text. For a quick recall, we include \textit{maxSSD-g} in here:
\begin{equation}
    \begin{split}
    S_{\text{max}_{g_r}}(q,p)&=\sum_{\bm{r}\in O_r}{\sup_{\substack{h_{rg_r}\in\mathcal{F}_q\\\bm{g}_r\in\mathbb{S}^{D-1}}}{\mathbb{E}_q[{s}^r_p(\bm{x})h_{rg_r}(\bm{x}^T\bm{g}_r)+}}\\
    &{{\bm{r}^T\bm{g}_r\nabla_{\bm{x}^T\bm{g}_r}h_{rg_r}(\bm{x}^T\bm{g}_r)]}}
\end{split}
\end{equation}
}

Further, one can also use single optimal direction $\bm{r}$ to replace the summation over the orthogonal basis $O_r$, resulting in \textit{maxSSD-rg}($S_{\text{max}_{rg_r}}$):
\begin{equation}
\begin{split}
    S_{max_{rg_r}}(q,p)&=\sup_{h_{rg}\in\mathcal{F}_q,\bm{g}_r,\bm{r}\in\mathbb{S}^{D-1}}\mathbb{E}_q\left[s^r_p(\bm{x})h_{rg_r}(\bm{x}^T\bm{g}_r)+\right.\\
    &\left.\bm{r}^T\bm{g}_r\nabla_{\bm{x}^T\bm{g}_r}h_{rg_r}(\bm{x}^T\bm{g}_r)\right]
\end{split}
\label{eq: maxSSD-rg}
\end{equation}
% It is worth noting that there is one more valid member of the \textit{SSD} family, called \textit{orthogonal sliced Stein discrepancy}, which only replace the integral of $p_r(\bm{r})$ with summation over $O_r$. Namely, it is defined as \textit{integrated SSD} (Eq.\ref{eq: integrated SSD}) where the integral of $p_r$ is replaced by summation over orthogonal basis $O_r$.
% \begin{equation}
%     \begin{split}
%     S_O(q,p)&=\sum_{\bm{r}\in O_r}\mathbb{E}_{p_g}\left[
%     \sup_{h_{rg}\in\mathcal{F}_q}{\mathbb{E}_q[s^r_p(\bm{x})h_{rg}(\bm{x}^T\bm{g})+}\right.\\
%     &\left.{\bm{r}^T\bm{g}\nabla_{\bm{x}^T\bm{g}}h_{rg}(\bm{x}^T\bm{g})]}
%     \vphantom{\sup_{asd}}\right].
% \end{split}
%     \label{eq: orthogonal SSD}
% \end{equation}

Similar to KSD, authors addressed tractability issue of the optimal $h_{rg_r}$ by restricting the $\mathcal{F}_q$ to be a one-dimensional RKHS induced by a $c_0-$universal kernel $k_{rg}:\mathcal{K}\times\mathcal{K}\rightarrow \mathbb{R}$ where $\mathcal{K}\subseteq \mathbb{R}$. Thus, for each member of the above SSD family, we have a corresponding kernelized version. \wg{They are called \textit{integrated sliced kernelized Stein discrepancy}, \textit{orthogonal SKSD}, and \textit{max sliced kernelized Stein discrepancy} (including \textit{maxSKSD-g} and \textit{maxSKSD-rg}). In practice, \emph{maxSKSD-g} or \emph{maxSKSD-rg} is often preferred over the others due to its computational tractability, where their optimal slices for $\bm{r}$ and $\bm{g}_r$ are obtained by gradient-based optimization.} 

\wg{By reproducing properties of RKHS, one can define $\xi_{p,r,g_r}(\bm{x},\cdot)$ as in Eq.\ref{eq: SKSD test function}, and further define $\mu_{p,r,g_r}=\langle \xi_{p,r,g_r}(\bm{x},\cdot),\xi_{p,r,g_r}(\bm{y},\cdot)\rangle_{\mathcal{H}_{{rg_r}}}$}
% \begin{equation}
% \begin{split}
%     \xi_{p,r,g}(\bm{x},\cdot)&=s_p^r(\bm{x})k_{rg}(\bm{x}^T\bm{g},\cdot)+\bm{r}^T\bm{g}\nabla_{\bm{x}^T\bm{g}}k_{rg}(\bm{x}^T\bm{g},\cdot),
% \end{split}
%     \label{eq: SKSD test function}
% \end{equation}
\begin{equation}
    \begin{aligned}
    \mu_{p,r,g_r}(\bm{x},\bm{y})&=s^r_p(\bm{x})k_{rg_r}(\bm{x}^T\bm{g}_r,\bm{y}^T\bm{g}_r)s^r_p(\bm{y})\\
      &+\bm{r}^T\bm{g}_rs^r_p(\bm{y})\nabla_{\bm{x}^T\bm{g}_r}k_{rg_r}(\bm{x}^T\bm{g}_r,\bm{y}^T\bm{g}_r)\\
      &+\bm{r}^T\bm{g}_rs^r_p(\bm{x})\nabla_{\bm{y}^T\bm{g}_r}k_{rg}(\bm{x}^T\bm{g}_r,\bm{y}^T\bm{g}_r)\\
      &+(\bm{r}^T\bm{g}_r)^2\nabla^2_{\bm{x}^T\bm{g}_r,\bm{y}^T\bm{g}_r}k_{rg}(\bm{x}^T\bm{g}_r,\bm{y}^T\bm{g}_r).
    \end{aligned}
    \label{eq: SKSD inner}
\end{equation}
\wg{Then, by simple algebra, one can show that given $\bm{r}$, $\bm{g}_r$, the optimality w.r.t. test functions can be computed analytically}:
\begin{equation}
    \begin{split}
        &D^2_{rg_r}(q,p)\\
        =&(\sup_{\substack{h_{rg_r}\in\mathcal{H}_{rg_r}\\||h_{rg_r}||_{\mathcal{H}_{rg_r}}\leq 1}}\mathbb{E}_q[{s}^r_p(\bm{x})h_{rg_r}(\bm{x}^T\bm{g}_r)\\
        &+\bm{r}^T\bm{g}_r\nabla_{\bm{x}^T\bm{g}_r}h_{rg_r}(\bm{x}^T\bm{g}_r)])^2\\
        =&||\mathbb{E}_q[\xi_{p,r,g_r}(\bm{x})]||^2_{\mathcal{H}_{rg_r}}%\\
        =\mathbb{E}_{q(\bm{x})q(\bm{x}')}[\mu_{p,r,g_r}(\bm{x},\bm{x}')].
    \end{split}
    \label{eq: D_rg}
\end{equation}
where $\mathcal{H}_{rg_r}$ is the RKHS induced by the kernel $k_{rg_r}$. Therefore, the \textit{maxSSD-g} and \textit{maxSSD-rg} can be computed as
\begin{equation}
    SK_{\text{max}_{g_r}}(q,p)=\sum_{\bm{r}\in O_r}{\sup_{\bm{g}_r\in\mathbb{S}^{D-1}}{D^2_{rg_r}(q,p)}}
\end{equation}
and
\begin{equation}
    SK_{\text{max}_{rg_r}}(q,p)={\sup_{\substack{\bm{g}_r\in\mathbb{S}^{D-1}\\\bm{r}\in\mathbb{S}^{D-1}}}{D^2_{rg_r}(q,p)}}
    \label{eq: maxSKSD-rg}
\end{equation}

\section{Goodness-of-fit test}
\label{appendix: goodness-of-fit test}
In this section, we give an introduction to the GOF test. To be general, we focus on the \textit{SKSD-rg} ($SK_{rg_r}=||\mathbb{E}_q[\xi_{p,r,g_r}(\bm{x})]||^2_{\mathcal{H}_{rg_r}}$) as other related discrepancy can be easily derived from it. Assuming we have active slices $\bm{r}$ and $\bm{g}_r$ from algorithm \ref{alg: active slice algorithm}. Thus, we can estimate $SK_{rg_r}$ using the minimum variance U-staistics \citep{hoeffding1992class,serfling2009approximation}:
\begin{equation}
    \widehat{SK}_{rg_r}(q,p)=\frac{1}{N(N-1)}{\sum_{1\leq i\neq j\leq N}{\mu_{p,r,g_r}(\bm{x}_i,\bm{x}_j)}}.
    \label{eq: U SKSD-rg}
\end{equation}
where $\mu_{\bm{x},\bm{y}}$ is defined in Eq.\ref{eq: SKSD inner} which satisfies $\mathbb{E}_{q(\bm{x})q(\bm{x}')}[\mu_{p,r,g_r}(\bm{x},\bm{x}')]=||\mathbb{E}_q[\xi_{p,r,g_r}(\bm{x})]||^2_{\mathcal{H}_{rg_r}}$, and $\bm{x}$, $\bm{x}'$ are i.i.d. samples from $q$. With the help of the U-statistics, we characterize its asymptotic distribution.
\begin{theorem}
Assume the conditions in theorem \ref{thm: validity of maxSKSD} are satisfied, we have the following:
\begin{enumerate}
    \item If $q\neq p$, then $\widehat{SK}_{rg_r}$ is asymptotically normal. Particularly,
    \begin{equation}
        \sqrt{N}(\widehat{SK}_{rg_r}-SK_{rg_r})\stackrel{d}{\rightarrow}\mathcal{N}(0,\sigma_{h}^2)
        \label{eq: Asymptotic maxSKSD q neq q}
    \end{equation}
    where $\sigma_{h}^2=\text{var}_{\bm{x}\sim q}({\mathbb{E}_{\bm{x}'\sim q}[\mu_{p,r,g_r}(\bm{x},\bm{x}')]})$ and $\sigma_h\neq 0$
    \item If $q=p$, we have a degenerated U-statistics with $\sigma_h=0$ and 
    \begin{equation}
        N\widehat{SK}_{rg_r}\stackrel{d}{\rightarrow}\sum_{j=1}^\infty{c_j(Z_j^2-1)}
        \label{eq: Asymptotic maxSKSD q=p}
    \end{equation}
    where $\{Z_j\}$ are i.i.d standard Gaussian variables, and $\{c_j\}$ are the eigenvalues of the kernel ${\mu_{p,r,g_r}(\bm{x},\bm{x}')}$ under $q(\bm{x})$. In other words, they are the solutions of $c_j\phi_j(\bm{x})=\int_{\bm{x}'}{{\mu_{p,r,g_r}(\bm{x},\bm{x}')}\phi_j(\bm{x}')q(\bm{x}')d\bm{x}'}$.
\end{enumerate}
\label{thm: asymptotic distribution of SKSD-rg}
\end{theorem}
\begin{proof}
As the $\widehat{SK_{rg_r}}$ is the second order U-statistic of $SK_{rg_r}$, thus, we can directly use the results from section 5.5.1 and 5.5.2 in \cite{serfling2009approximation}.
\end{proof}
The above theorem indicates a well-defined asymptotic distribution for $SK_{rg_r}$, which allows us to use the following bootstrap method to estimate the rejection threshold \citep{huskova1993consistency,arcones1992bootstrap,liu2016kernelized}. The bootstrap samples can be computed as 
\begin{equation}
    \widehat{SK}_m^*=\sum_{1\leq i\neq j\leq N}{(w^m_i-\frac{1}{N})(w^m_j-\frac{1}{N}){\mu_{p,r,g_r}(\bm{x}_i,\bm{x}_j)}}
    \label{eq: Bootstrap samples}
\end{equation}
where $(w^m_1,\ldots,w^m_N)_{m=1}^M$ are random weights drawn from multinomial distributions $\text{Multi}(N,\frac{1}{N},\ldots,\frac{1}{N})$. 
Now, we give the detailed algorithm for GOF test.
\begin{algorithm}[H]
   \caption{GOF test with active slices}
   \label{alg: GOF with active slices}
\begin{algorithmic}
   \STATE {\bfseries Input:} Samples $\bm{x}\sim q$, density $p$, kernel $k_{rg_r}$, active slices $\bm{r}$, $\bm{g}_r$, significance level $\alpha$, and bootstrap sample size $M$.
   \STATE {\bfseries Hypothesis:} $H_0$: $p=q$ v.s. $H_1:$ $p\neq q$
   \STATE {Computing U-statistics $\widehat{SK}_{rg_r}$ using Eq.\ref{eq: U SKSD-rg}}
   \STATE {Generate $M$ bootstrap samples $\{\widehat{SK}^*_m\}_{m=1}^M$ using Eq.\ref{eq: Bootstrap samples}}
   \STATE {Reject null hypothesis $H_0$ if the proportion of $\widehat{SK}_{m}^*>\widehat{SK}_{rg_r}$ is less than $\alpha$}
\end{algorithmic}
\end{algorithm}
\section{Relaxing constraints for kernelized SSD family}
\label{appendix: validity of maxSKSD with active slices}
\subsection{Validity w.r.t $\bm{r}$, $\bm{g}_r$}
\label{appendix subsec: validity w.r.t g_r}
The key to this proof is to prove the real analyticity of $SK_{g_r}$ (or $S_{rg_r}$) to slices $\bm{r}$ and $\bm{g}_r$. Therefore, let's first give a definition of multivariate real analytic function.
\begin{definition}[Real analytic function]
A function $f:\mathcal{U}\rightarrow \mathbb{R}$ is real analytic if for each $\bm{c}\in\mathcal{U}$, there is a power series as in the form 
\[
f(\bm{x})=\sum_{\kappa\in\mathbb{N}_0^n}\alpha_{\kappa}(\bm{x}-\bm{c})^\kappa
\]
for some choice of $(\alpha_\kappa)_{\kappa\in\mathbb{N}_0^n}\subset \mathbb{R}$ and all $\bm{x}$ in a neighbourhood of $\bm{c}$, and this power series converges absolutely. Namely,
\[
\sum_{\kappa\in\mathbb{N}_0^n}|\alpha_{\kappa}||(\bm{x}-\bm{c})^\kappa|<\infty
\]
where $\mathbb{N}_0=\{0,1,\ldots\}$ denotes non-negative integers, $\kappa=(\kappa_1,\ldots,\kappa_n)$ are called multiindex, and we define $\bm{x}^\kappa=x_1^{\kappa_1}\ldots x_n^{\kappa_n}$.
\label{def: real analytic function}
\end{definition}
Now, we introduce a useful lemma showing that composition of real analytic function is also real analytic. 
\begin{lemma}[Composition of real analytic function]
Let $\mathcal{U}\subset\mathbb{R}^n$ and $\mathcal{V}\subset\mathbb{R}^m$ be open, and let $\bm{f}:\mathcal{U}\rightarrow \mathcal{V}$ and $\bm{g}:\mathcal{V}\rightarrow \mathbb{R}^p$ be real analytic. Then $\bm{g}\circ \bm{f}:\mathcal{U}\rightarrow \mathbb{R}^p$ is real analytic.
\label{lemma: composition of real analytic function}
\end{lemma}
Especially, the real analyticity is not only preserved by function composition, it is also closed under most of the simple operations: addition, multiplication, division (assuming denominator is non-zero), etc. Now we can prove the main proposition to show that the \textit{SKSD-rg} ($SK_{{rg_r}}$) is real analytic w.r.t both $\bm{g}_r$ and $\bm{r}$.  In the following, we assume the $\bm{r},\bm{g}_r\in\mathbb{R}^D$. 
\begin{prop}[\textit{SKSD-g} is real analytic]
Assume assumption 1-4 (density regularity), 5-6 (kernel richness and real analyticity) are satisfied, further we let $\bm{g}_r\in \mathbb{R}^D$, then \textit{SKSD-g} ($SK_{g_r}$) is real analytic w.r.t $\bm{g}_r$ and $SK_{rg_r}$ is real analytic to both $\bm{r}\in\mathbb{R}^D$ and $\bm{g}_r$. 
\label{prop: maxSKSD-g real analytic}
\end{prop}
\begin{proof}
First, let's focus on the real analyticity w.r.t. $\bm{g}_r$. We re-write the \textit{SKSD-g} as the following:
\[
\begin{split}
    &SK_{{g_r}}=\sum_{\bm{r}\in O_r}||\xi_{p,r,g_r}(\bm{x})||^2_{\mathcal{H}_{rg_r}}\\
    &=\sum_{\bm{r}\in O_r}\langle\mathbb{E}_{q}[\underbrace{(s_p^r(\bm{x})-s_q^r(\bm{x}))}_{f_r^*(\bm{x})}k_{rg_r}(\bm{x}^T\bm{g}_r,\cdot)],\\
    &\mathbb{E}_{q}[(s_p^r(\bm{x})-s_q^r(\bm{x}))k_{rg_r}(\bm{x}^T\bm{g}_r,\cdot)]\rangle_{\mathcal{H}_{rg_r}}\\
    &=\sum_{\bm{r}\in O_r}\mathbb{E}_{\bm{x},\bm{x}'}[f_r^*(\bm{x})k_{rg_r}(\bm{x}^T\bm{g}_r,\bm{x}'^T\bm{g}_r)f_r^*(\bm{x}')]
\end{split}
\]
The second equality is from the definition of RKHS norm $||\cdot||_{\mathcal{H}_{rg_r}}$ and Stein identity. We can observe that $\bm{g}_r$ appears inside the kernel $k_{rg_r}$ in the form of $\bm{x}^T\bm{g}_r$. So in order to use the function composition lemma (lemma \ref{lemma: composition of real analytic function}), we need to first show that for any given $\bm{x}$, $\bm{x}^T\bm{g}_r$ is real analytic.
By definition of real analytic function, we need a center point $\bm{c}\in\mathbb{R}^D$, and $\bm{g}_r$ in the neighborhood of $c$ (i.e. $|\bm{g}_r-\bm{c}|<R_c$). Then, we define the power series as 
\[
h_x(\bm{g}_r)=\sum_{\kappa_1=0}^\infty\ldots\sum_{\kappa_D=0}^\infty\frac{(g_{r1}-c_1)^{\kappa_1}\ldots(g_{rD}-c_D)^{\kappa_D}}{\kappa_1!\ldots\kappa_D!}\alpha_{\{\kappa_i\}_i^D}
\]
with the following coefficient
\[
\begin{cases}
\alpha_{\{\kappa_i\}_{i}^D}=0\;\;\;\;\;\text{if}\;\;\;\;\; \sum_{i}\kappa_i>1\\
\alpha_{\{\kappa_i\}_{i}^D}=x_i\;\;\;\;\;\text{if}\;\;\;\;\; \kappa_i=1,\sum_{i}\kappa_i=1\\
\alpha_{\{\kappa_i\}_{i}^D}=\pmb{c}^T\pmb{x}\;\;\;\;\;\text{if}\;\;\;\;\; \sum_{i}\kappa_i=0\\
\end{cases}
\]
Then, by substitution, we have
\begin{align}
h_x(\pmb{g})&=\sum_{d=1}^D{(g_d-c_d)x_d}+\pmb{c}^T\pmb{x}\\
&=\pmb{x}^T\pmb{g}
\end{align}
which converges with radius of convergence $R_c=\infty$. 
From assumption 6, we know the kernel $k_{rg_r}(\bm{x}^T\bm{g}_r,\bm{x}'^T\bm{g}_r)=\phi((\bm{x}-\bm{x}')^T\bm{g}_r)$ is translation invariant and real analytic. Thus, from lemma \ref{lemma: composition of real analytic function}, we know $k_{rg_r}(\bm{x}^T\bm{g}_r,\bm{x}'^T\bm{g}_r)$ is real analytic to $\bm{g}_r$ with radius of convergence $R_k$ ($R_k$ is determined by the form of the kernel function). This means we can use a power series to represents this kernel w.r.t. $\bm{g}_r$ inside some neighborhood define around center point. Specifically, for a central point $\bm{c}\in\mathbb{R}^D$ and any $\bm{g}_r$ satisfying $|\bm{g}_r-\bm{c}|<R_k$, we have
\[
k_{rg_r}(\bm{x}^T\bm{g}_r,\bm{x}'^T\bm{g}_r)=\sum_{\kappa\in\mathbb{N}_0^D}\alpha_{\kappa}(\bm{x},\bm{x}')(\bm{g}_r-\bm{c})^\kappa
\]
where this series converges absolutely. 
We substitute it into $SK_{{g_r}}$
\[
\begin{split}
    &SK_{{g_r}}=\sum_{\bm{r}\in O_r}\mathbb{E}_{\bm{x},\bm{x}'}[f_r^*(\bm{x})k_{rg_r}(\bm{x}^T\bm{g}_r,\bm{x}'^T\bm{g}_r)f_r^*(\bm{x}')]\\
    =&\sum_{\bm{r}\in O_r}\mathbb{E}_{\bm{x},\bm{x}'}[f_r^*(\bm{x})\sum_{\kappa\in\mathbb{N}_0^D}\alpha_{\kappa}(\bm{x},\bm{x}')(\bm{g}_r-\bm{c})^\kappa f_r^*(\bm{x}')]\\
    =&\sum_{\bm{r}\in O_r}\sum_{\kappa\in\mathbb{N}_0^D}\mathbb{E}_{\bm{x},\bm{x}'}[\alpha_\kappa(\bm{x},\bm{x}')f_r^*(\bm{x})f_r^*(\bm{x}')](\bm{g}_r-\bm{c})^\kappa
\end{split}
\]
which also converges absolutely with radius of convergence $R_k$. The third equality is from the Fubini's theorem. The conditions of Fubini's theorem can be verified by fact that $f_r^*$ is square integrable (assumption 2), and the power series of $k_{rg_r}$ converges absolutely. Thus, by definition of real analytic function, \textit{SKSD-g} is real analytic w.r.t each $\bm{g}_r$. This also implies \textit{SKSD-rg} ($SK_{rg_r}$) is real analytic w.r.t. $\bm{g}_r$ (because $SK_{rg_r}$ is just $SK_{g_r}$ without summation over $O_r$).  

For the real analyticity w.r.t $\bm{r}$, the proof is almost the same. The inner product $s_p^r(\bm{x})-s_q^r(\bm{x})$ is real analytic w.r.t $\bm{r}$ obviously for given $\bm{x}$. We also use the fact that real analyticity is preserved under multiplication of two real analytic functions. In addition, note that $k_{rg_r}(\bm{x}^T\bm{g}_r,\bm{x}^T\bm{g}_r)$ act as a constant w.r.t. $\bm{r}$, we can directly apply the Fubini's theorem again to form a power series w.r.t. $\bm{r}$ with absolute convergence. Thus, $SK_{rg_r}$ is real analytic w.r.t. $\bm{r}$ for any $\bm{g}_r$. Thus, $SK_{rg_r}$ is real analytic to both $\bm{r}$ and $\bm{g}_r$. 
\end{proof}
Next, we introduce an important property of real analytic function:
\begin{lemma}[Zero Set Theorem \citep{mityagin2015zero}]
Let $f(\bm{x})$ be a real analytic function on (a connected open domain $\mathcal{U}$ of)$\mathbb{R}^d$. If $f$ is not identically $0$, then its zero set
\[
S(f):=\{\bm{x}\in\mathcal{U}|f(x)=0\}
\]
has a measure $0$, i.e. $\text{mes}_dS(f)=0$
\label{lemma: zero set theorem}
\end{lemma}
% Now, we can prove that with $\bm{g}_r\sim \eta$ supported on $\mathbb{R}^D$, the corresponding \textit{SKSD-g} is a valid discrepancy almost $\eta$ surely. 
With the help from the zero-set theorem, we can prove the validity of $SK_{g_r}$ (or $SK_{rg_r}$) with finite random slices $\bm{g_r}$ (and $\bm{r}$). 

\textbf{Proof of theorem \ref{thm: validity of maxSKSD}}
\begin{proof}
We first deal with the validity of $\bm{g}_r$ with fixed orthogonal basis $O_r$. 
It is trivial that when $p=q$, $SK_{{g_r}}=0$ identically. Now, assume $p\neq q$, then, from the theorem 3 in \cite{gong2021sliced}, the orthogonal SKSD (Eq.\ref{eq: integrated SKSD}) is a valid discrepancy. Namely, we have
\begin{equation}
    \sum_{\bm{r}\in O_r}\int q_{g_r}(\bm{g}_r)||\mathbb{E}_q[\xi_{p,r,g_r}(\bm{x})]||^2_{\mathcal{H}_{rg_r}}> 0
\label{eq: integrated SKSD}
\end{equation}
We should note that the distribution $q_{g_r}$ is originally defined on $\mathbb{S}^{D-1}$. But, we can easily generalize it to larger spaces. As for $\bm{g}_r\in\mathbb{R}^D$, we can always write $\bm{g}_r=c\bm{g}_r'$, where $\bm{g}_r'\in\mathbb{S}^{D-1}$, and $c\geq 0$. As the domain for $\bm{g}_r$ is $\mathbb{R}^D$, the $\bm{g}_r$ can represents all possible directions. Thus, we can follow the same proof logic as theorem 3 in \cite{gong2021sliced} to show the corresponding discrepancy is greater than 0 when $p\neq q$. 

Therefore, Eq.\ref{eq: integrated SKSD} represents there exists a $\bm{r}\in O_r$ such that $||\mathbb{E}_q[\xi_{p,r,g_r}(\bm{x})]||^2_{\mathcal{H}_{rg_r}}>0$ for a set of $\bm{g}_r$ with non-zero measure. Namely, $||\mathbb{E}_q[\xi_{p,r,g_r}(\bm{x})]||^2_{\mathcal{H}_{rg_r}}$ is not $0$ identically. Thus, from the propositon \ref{prop: maxSKSD-g real analytic} and lemma \ref{lemma: zero set theorem}, the set of $\bm{g}_r$ that make $||\mathbb{E}_q[\xi_{p,r,g_r}(\bm{x})]||^2_{\mathcal{H}_{rg_r}}=0$ has a $0$ measure. Then, if $\bm{g}_r$ is sampled from some distribution $\eta_g$ with density supported on $\mathbb{R}^D$ (e.g. Gaussian distribution), we have
\[
SK_{{g_r}}=\sum_{\bm{r}\in O_r}||\mathbb{E}_q[\xi_{p,r,g_r}(\bm{x})]||^2_{\mathcal{H}_{rg_r}}>0
\]
almost surely. 

Now, we show that $SK_{rg_r}$ is also a valid discrepancy with $\bm{r}\sim \eta_r$. First, due to the validity of \textit{integrated SKSD}, we have
\begin{equation}
    \int q_{r}(\bm{r})\int q_{g_r}(\bm{g}_r)||\mathbb{E}_q[\xi_{p,r,g_r}(\bm{x})]||^2_{\mathcal{H}_{rg_r}}d\bm{g}_rd\bm{r}> 0
\label{eq: integrated rg SKSD}
\end{equation}
Due to the real analyticity of $SK_{rg_r}$ ($||\mathbb{E}_q[\xi_{p,r,g_r}(\bm{x})]||^2_{\mathcal{H}_{rg_r}}$) w.r.t $\bm{r}$, we can easily show that 
\[
\int q_{g_r}(\bm{g}_r)||\mathbb{E}_q[\xi_{p,r,g_r}(\bm{x})]||^2_{\mathcal{H}_{rg_r}}d\bm{g}_r
\]
is real analytic to $\bm{r}$ and it is not 0 identically. Thus, by lemma \ref{lemma: zero set theorem}, for $\bm{r}\sim \eta_r$, we have
\[
\int q_{g_r}(\bm{g}_r)||\mathbb{E}_q[\xi_{p,r,g_r}(\bm{x})]||^2_{\mathcal{H}_{rg_r}}d\bm{g}_r> 0
\]
Namely, $||\mathbb{E}_q[\xi_{p,r,g_r}(\bm{x})]||^2_{\mathcal{H}_{rg_r}}> 0$ for a set of $\bm{g}_r$ with non-zero measure. In the beginning of the proof, we show that this set of $\bm{g}_r$ is almost everywhere in $\mathbb{R}^D$ due to its real analyticity. Namely, $||\mathbb{E}_q[\xi_{p,r,g_r}(\bm{x})]||^2_{\mathcal{H}_{rg_r}}> 0$ for $\bm{r}\sim \eta_r$ and $\bm{g}_r\sim \eta_g$ if $p\neq q$. Thus, we can conclude that for $SK_{rg_r}=0$ if and only if $p=q$ almost surely for $\bm{r}\sim\eta_r$ and $\bm{g}_r\sim\eta_g$.
\end{proof}
\begin{corollary}[Normalizing $\bm{g}_r$]
Assume the conditions in theorem \ref{thm: validity of maxSKSD} are satisfied, then the following operations do not violate the validity of \textit{SKSD-rg} $SK_{rg_r}$. (1) For $\bm{g}'_r,\bm{r}'\in\mathbb{S}^{D-1}$, we define $\bm{g}_r=\bm{g}'_r+\bm{\gamma}_g$ and $\bm{r}=\bm{r}'+\bm{\gamma}_r$, where $\gamma_r$, $\gamma_g$ are the noise from Gaussian distribution. (2) Define $\tilde{\bm{g}}_r=c_g\times \bm{g}_r$ and $\tilde{\bm{r}}=c_r\times \bm{r}$, where $\tilde{\bm{g}}_r,\tilde{\bm{r}}$ are unit vectors and $c_r,c_g>0$. The resulting active slices $\tilde{\bm{r}}$ and $\tilde{\bm{g}}_r$ do not violate the validity of $SK_{rg_r}$. 
\label{coro: validity of maxSKSD normalized}
\end{corollary}
\begin{proof}
From the theorem \ref{thm: validity of maxSKSD} with ${\bm{g}}_r,{\bm{r}}$, when $p\neq q$, we have 
\[
\begin{split}
    &SK_{rg_r}=||\mathbb{E}_q[\xi_{p,r,g_r}(\bm{x})]||^2_{\mathcal{H}_{rg_r}}\\
    =&\mathbb{E}_{\bm{x},\bm{x}'}[f_r^*(\bm{x})k_{rg_r}(\bm{x}^T\bm{g}_r,\bm{x}'^T\bm{g}_r)f_r^*(\bm{x}')]\\
    =&\mathbb{E}_{\bm{x},\bm{x}'}[c^2_rf_{\tilde{r}}^*(\bm{x})k_{rg_r}(c\bm{x}^T\tilde{\bm{g}}_r,c\bm{x}'^T\tilde{\bm{g}}_r)f_{\tilde{r}}^*(\bm{x}')]>0
\end{split}
\]
From the assumption 6 that $k_{rg_r}(c\bm{x}^T\tilde{\bm{g}}_r,c\bm{x}'^T\tilde{\bm{g}}_r)=k'_{rg_r}(\bm{x}^T\tilde{\bm{g}}_r,\bm{x}'^T\tilde{\bm{g}}_r)$. So this is equivalent to the \textit{SKSD-rg} defined with a new $c_0-$universal kernel $k'_{rg_r}$ and $\tilde{\bm{g}}_r,\tilde{\bm{r}}\in\mathbb{S}^{D-1}$. Thus, the corresponding \textit{maxSKSD-rg} with $\tilde{\bm{g}}_r,\tilde{\bm{r}}\in\mathbb{S}^{D-1}$ is a valid discrepancy almost surely.
\end{proof}
% \subsection{Validity w.r.t $\bm{r}$}
% With active slice $\bm{r}$ (instead of orthogonal $O_r$), we consider \textit{SKSD-rg} instead of \textit{SKSD-g}.
% To prove the validity of \textit{SKSD-rg}, we can directly follow the proof strategy of appendix \ref{appendix subsec: validity w.r.t g_r}. In fact, the proof for $\bm{r}$ is even simpler. Because $\bm{r}$ only appears in the form of vector inner product, $(\bm{s}_p(\bm{x})-\bm{s}_q(\bm{x}))^T\bm{r}$, which is real analytic w.r.t. $\bm{r}$. Then instead of using Eq.\ref{eq: integrated SKSD}, we use \textit{SKSD-rg} with both $\bm{r}$ and $\bm{g}_r$ integrated out, which is exactly the kernelized version of \textit{integrated SSD}(Eq.\ref{eq: integrated SSD}). Its validity guarantees that when $p\neq q$, it is greater than $0$. Then we obtain $\tilde{\bm{r}}$ by performing the similar modification of $\tilde{\bm{g}}_{r}$. Following the similar argument as theorem \ref{thm: validity of maxSKSD} and corollary \ref{coro: validity of maxSKSD normalized}, we can show the validity of \textit{SKSD-rg} is preserved by $\tilde{\bm{r}}$.
\subsection{Relationship beetween SSD and SKSD} 
\label{appendix subsec: relationship SSD and SKSD}
% Before we jump into the details, the supremum of $h_{rg_r}$ in $S_{rg_r}$ hinders the further analysis. The following proposition analytically gives the optimal form of $h_{rg_r}$.
% \begin{prop}[Optimal test function given $\boldsymbol{r}, \boldsymbol{g}_r$]
% Assume assumptions \hyperref[assumption 1]{1-4} and given directions $\boldsymbol{r}, \boldsymbol{g}_r$. Assume an arbitrary orthogonal matrix $\bm{G}_r=$ $\left[\boldsymbol{a}_{1}, \ldots, \boldsymbol{a}_{D}\right]^{T}$ where $\boldsymbol{a}_{i} \in \mathbb{S}^{D \times 1}$ and $\boldsymbol{a}_{d}=\boldsymbol{g}_r .$ Denote $\boldsymbol{x} \sim q$ and $\boldsymbol{y}=\boldsymbol{G}_r \boldsymbol{x}$ which is also a random variable with the induced distribution $q_{G_r}$. Then, the optimal test function for $S_{rg_r}$ is 
% \begin{equation}
% h_{rg_r}^{*}\left(\boldsymbol{x}^{T} \boldsymbol{g}_r\right) \propto \mathbb{E}_{q_ {G_r}\left(\boldsymbol{y}_{-d} \mid y_{d}\right)}\left[\left(s_{p}^{r}\left(\boldsymbol{G}_r^{-1} \boldsymbol{y}\right)-s_{q}^{r}\left(\boldsymbol{G}_r^{-1} \boldsymbol{y}\right)\right)\right]    
% \label{eq: optimal test function maxSSD-g}
% \end{equation}
% where $y_{d}=\boldsymbol{x}^{T} \boldsymbol{g}_r$ and $\boldsymbol{y}_{-d}$ contains other $\boldsymbol{y}$ elements.
% \label{prop: optimality of maxSSD-g}
% \end{prop}

\textbf{Proof of proposition \ref{prop: optimality of maxSSD-g}}
\begin{proof}
We consider the \textit{SSD-rg} ($S_{rg_r}$) without the optimal test function:
\begin{equation}
    \mathbb{E}_q[s_p^r(\bm{x})h_{rg_r}(\bm{x}^T\bm{g}_r)+\bm{r}^T\bm{g}_r\nabla_{\bm{x}^T\bm{g}_r}h_{rg_r}(\bm{x}^T\bm{g}_r)]
    \label{eq: inner maxSSD-g}
\end{equation}
From the Stein identity (Eq.\ref{eq: Stein Identity}), we can let $\boldsymbol{f}(\boldsymbol{x})=\left[r_{1} h_{rg_r}\left(\boldsymbol{x}^{T} \boldsymbol{g}_r\right), r_{2} h_{rg_r}\left(\boldsymbol{x}^{T} \boldsymbol{g}_r\right), \ldots, r_{D} h_{rg_r}\left(\boldsymbol{x}^{T} \boldsymbol{g}_r\right)\right]^{T}$ and then take the trace. Thus, we have
\[
\mathbb{E}_q[s_q^r(\bm{x})h_{rg_r}(\bm{x}^T\bm{g}_r)]=\mathbb{E}_q[\bm{r}^T\bm{g}_r\nabla_{\bm{x}^T\bm{g}_r}h_{rg_r}(\bm{x}^T\bm{g}_r)]
\]
Substitute it into Eq.\ref{eq: inner maxSSD-g} and change the variable to $\bm{y}=\bm{G}_r\bm{x}$, we have
\[
\begin{split}
    &\mathbb{E}_q[(s_p^r(\bm{x})-s_q^r(\bm{x}))h_{rg_r}(\bm{x}^T\bm{g}_r)]\\
    =&\int q_{G_r}(y_d,\bm{y}_{-d})\underbrace{(s_p^r(\bm{G}_r^{-1}\bm{y})-s_q^r(\bm{G}_r^{-1}\bm{y}))}_{f_r^*(\bm{G}_r^{-1}\bm{y})}h_{rg_r}(y_d)d\bm{y}\\
    =&\int q_{G_r}(y_d)\underbrace{\int q_{G_r}(\bm{y}_{-d}|y_d)f_{r}^*(\bm{G}_r^{-1}\bm{y})d\bm{y}_{-d}}_{h_{rg_r}^*(y_d)}h_{rg_r}(y_d)dy_d\\
    \leq&\sqrt{\mathbb{E}_{q_{G_r}(y_d)}[h_{rg_r}^*(y_d)^2]}\sqrt{\mathbb{E}_{q_{G_r}(y_d)}[h_{rg_r}(y_d)^2]}
\end{split}
\]
where the last inequality is from Cauchy-Schwarz inequality, where the equality holds when 
\[
\begin{split}
&h_{rg_r}(y_d)\propto h_{rg_r}^*(y_d)\\
    &=\mathbb{E}_{q_ {G_r}\left(\boldsymbol{y}_{-d} \mid y_{d}\right)}\left[\left(s_{p}^{r}\left(\boldsymbol{G}_r^{-1} \boldsymbol{y}\right)-s_{q}^{r}\left(\boldsymbol{G}_r^{-1} \boldsymbol{y}\right)\right)\right]
\end{split}
\]
where $y_d=\bm{x}^T\bm{g}_r$. 
\end{proof}

\textbf{Proof of theorem \ref{thm: maxSKSD approximates maxSSD}}
\begin{proof}
Let's first re-write of $S_{{rg_r}}$ and $SK_{{rg_r}}$.
\[
\begin{split}
    &S_{{rg_r}}=\sup_{h_{rg_r}\in\mathcal{F}_q}\mathbb{E}_q[(s_p^r(\bm{y})-s_q^r(\bm{x}))h_{rg_r}(\bm{x}^T\bm{g}_r)]\\
    =&\mathbb{E}_{q_{G_r}(y_d)}[\underbrace{\int q_{G_r}(\bm{y}_{-d}|y_d)(s_p^r(\bm{G}^{-1}_r\bm{y})-s_q^r(\bm{G}^{-1}_r\bm{y}))d\bm{y}_{-d}}_{h^*_{rg_r}(y_d)}\\
    \times&h^*_{rg_r}(y_d)]\\
    =&\mathbb{E}_{q_{G_r}(y_d)}[h^{*}_{rg_r}(y_d)^2]
\end{split}
\]
where the second equality is from proposition \ref{prop: optimality of maxSSD-g}. 
\[
\begin{split}
    SK_{{rg_r}}=\langle \mathbb{E}_q[\xi_{p,r,g_r}(\bm{x})],\mathbb{E}_q[\xi_{p,r,g_r}(\bm{x})]\rangle_{\mathcal{H}_k}
\end{split}
\]
where $\xi_{p,r,g_r}(\bm{x},\cdot)$ is defined in Eq.\ref{eq: SKSD test function}, and $\langle\cdot,\cdot\rangle_{\mathcal{H}_{rg_r}}$ is the RKHS inner product induced by kernel $k_{rg_r}$. 
By simple algebraic manipulation and Stein identity (Eq.\ref{eq: Stein Identity}), we have
\[
\begin{split}
    &\mathbb{E}_q[\xi_{p,r,g_r}(\bm{x},\cdot)]\\
    =&\mathbb{E}_q[(s_p^r(\bm{x})-s_q^r(\bm{x}))k_{rg_r}(\bm{x}^T\bm{g}_r,\cdot)]\\
    =&\mathbb{E}_{q_{G_r}(y_d)}[\underbrace{\int q_{G_r}(\bm{y}_{-d}|y_d)(s_p^r(\bm{G}^{-1}_r\bm{y})-s_q^r(\bm{G}^{-1}_r\bm{y}))d\bm{y}_{-d}}_{h^*_{rg_r}(y_d)}\\
    \times& k_{rg_r}(y_d,\cdot)]\\
    =&\mathbb{E}_{q_{G_r}(y_d)}[h^*_{rg_r}(y_d)k_{rg_r}(y_d,\cdot)]
\end{split}
\]
Thus, we have
\[
\begin{split}
    &SK_{{rg_r}}\\
    =&\mathbb{E}_{y_d,y_d'\sim q_{G_r}(y_d)}[h^*_{rg_r}(y_d)k_{rg_r}(y_d,y_d')h^*_{rg_r}(y_d')]\\
    \leq&\sqrt{\mathbb{E}_{y_d,y'_d}[k_{rg_r}(y_d,y_d')^2]}\sqrt{\mathbb{E}_{y_d}[h^*_{rg_r}(y_d)^2]}\sqrt{\mathbb{E}_{y_d'}[h^*_{rg_r}(y_d')^2]}\\
    =&MS^*_{{rg_r}}
\end{split}
\]
where constant $M$ is from the bounded kernel assumption, and the inequality is from Cauchy-Schwarz inequality.
Without the loss of generality, we can set $M=1$. For other value of $M>0$, one can always set the optimal test function ($h_{rg_r}^*$) for \textit{SSD-rg} with coefficient $M$. The the new \textit{SSD-g} will be $M$ multiplied by the original \textit{SSD-rg} with $M=1$.

Thus, \textit{SSD-rg} is an upper bound for \textit{SKSD-rg}. 
From the assumption 1, we know that the induced set $\mathcal{K}=\{y\in\mathbb{R}|y=\bm{x}^T\bm{g},||g||=1,\bm{x}\in\mathcal{X}\}$ is LCH, and the kernel $k_{rg_r}:\mathcal{K}\times\mathcal{K}\rightarrow \mathbb{R}$ is $c_0-$universal. Then, from \cite{sriperumbudur2011universality}, $c_0-$universal implies $L_p-$universal. Namely, the induced RKHS $\mathcal{H}_{rg_r}$ is dense in $L^p(\mathcal{K};\mu)$ with all Borel probability measure $\mu$ w.r.t. \textit{p-norm}, defined as 
\[
||f||_p=\left(\int |f(\bm{x})|^pd\mu(\bm{x})\right)^\frac{1}{p}
\]
Now, from the assumption 4, we know $h_{rg_r}^*(y_d)$ is bounded for all possible $\bm{g}_r$, we have
\[
\int q_{G_r}(y_d)|h_{rg_r}^*(y_d)|^2dy_d<\infty
\]
This means $h_{rg_r}^*\in L^2(\mathcal{K},\mu_{G_r})$, where $\mu_{G_r}$ is the probability measure with density $q_{G_r}(y_d)$

From the $L_p-$universality, there exists a function $\widetilde{h^*_{rg_r}}\in\mathcal{H}_{rg_r}$, such that for any given $\epsilon>0$,
\[
||h^*_{rg_r}-\widetilde{h^*_{rg_r}}||_{2}<\epsilon
\]
Let's define $\widetilde{SK_{{rg_r}}}$ is the \textit{SKSD-rg} with the specific kernelized test function $\widetilde{h^*_{rg_r}}$, and from the optimality of \textit{SKSD-rg}, we have
\[
SK_{{rg_r}}\geq\widetilde{SK_{{rg_r}}}
\]
Therefore, we have
\[
\begin{split}
    0&\leq S_{{rg_r}}-SK_{{rg_r}}\\
    &\leq S_{{rg_r}}-\widetilde{SK_{{rg_r}}}\\
    &=\mathbb{E}_q[(s_p^r(\bm{x})-s_q^r(\bm{x}))(h^*_{rg_r}(\bm{x}^T\bm{g}_r)-\widetilde{h^*_{rg_r}}(\bm{x}^T\bm{g}_r))]\\
    &\leq\underbrace{\sqrt{\mathbb{E}_{q}[(s_p^r(\bm{x})-s_q^r(\bm{x}))^2]}}_{C_r}\\
    &\times \sqrt{\mathbb{E}_q[(h^*_{rg_r}(\bm{x}^T\bm{g}_r)-\widetilde{h^*_{rg_r}}(\bm{x}^T\bm{g}_r))^2]}\\
    &=C_r\sqrt{\int q_{G_r}(y_d,\bm{y}_{-d})(h^*_{rg_r}(y_d)-\widetilde{h^*_{rg_r}}(y_d))^2d\bm{y}}\\
    &=C_r||h^*_{rg_r}-\widetilde{h^*_{rg_r}}||_2<C_r\epsilon
\end{split}
\]
From assumption 2, we know $s_p^r(\bm{x})-s_q^r(\bm{x})$ is square integrable for all possible $\bm{r}$. Therefore, let's define $C=\max_{\bm{r}\in\mathbb{S}^{D-1}}C_r$, then,
\[
0\leq S_{{rg_r}}-SK_{{rg_r}}<C\epsilon
\]
\end{proof}

\section{Theory related to active slice $\bm{g}$}
\label{Appendix: theory related to g}
\subsection{Optimal test function for PSD}
\label{appendix subsec: optimal test function for PSD}
\begin{prop}[Optimality of PSD]
Assume the assumption $1-3$ (density regularity) are satisfied, then the optimal test function for \textit{PSD} given $O_r$ is proportional to the projected score difference, i.e.
\begin{equation}
    f_{r}^{*}(\boldsymbol{x}) \propto\left(s_{p}^{r}(\boldsymbol{x})-s_{q}^{r}(\boldsymbol{x})\right)
    % \label{eq: optimal test function PSD}
\end{equation}
Thus, 
\begin{equation}
    \text{PSD}(q,p;O_r)=\sum_{\bm{r}\in O_r}\mathbb{E}_{q}[(s_p^r(\bm{x})-s_q^r(\bm{x}))^2]
    \label{eq: optimal PSD}
\end{equation}
if the coefficient of $f^*_r$ to be 1. 
\label{prop: optimality of PSD}
\end{prop}
% It is worth noticing that the \textit{PSD} recovers another divergence called \textit{Fisher divergence} with orthogonal $O_r$. \textit{Fisher divergence} has been extensively used in training energy based model\cite{song2020sliced,song2019generative} and fitting kernel exponential families\cite{sriperumbudur2017density,sutherland2018efficient,wenliang2019learning}. 

\begin{proof}
From the Stein identity (Eq.\ref{eq: Stein Identity}), we can re-write the inner part of the supremum of Eq.\ref{eq: PSD} as 
\[
\begin{split}
&\mathbb{E}_q[s_p^r(\bm{x})f_{r}(\bm{x})+\bm{r}^T\nabla_{\bm{x}}f_r(\bm{x})]\\
=&\mathbb{E}_q[(s_p^r(\bm{x})-s_q^r(\bm{x}))f_{r}(\bm{x})]
\end{split}
\]
Then, we can upper bound the PSD (Eq.\ref{eq: PSD}) as the following
\[
\begin{split}
    &\sum_{\bm{r}\in O_r}\mathbb{E}_q[(s^r_p(\bm{x})-s^r_q(\bm{x}))f_r(\bm{x})]\\
    \leq&\sum_{\bm{r}\in O_r}\sqrt{\mathbb{E}_q[(s_p^r(\bm{x})-s_q^r(\bm{x}))^2]}\sqrt{\mathbb{E}_q[(f_r(\bm{x}))^2]}
\end{split}
\]
by Cauchy-Schwarz inequality. It is well-known that the equality holds when $f_r(\bm{x})\propto (s_p^r(\bm{x})-s_q^r(\bm{x}))$
\end{proof}
\subsection{Proof of Theorem \ref{thm: controlled approximation}}
\begin{proof}
The key to this proof is to notice that $h_{rg_r}^*$ is the conditional mean of $f_{r}^*$ w.r.t. the transformed distribution $q_{G_r}$. By using the similar terminology of proposition \ref{prop: optimality of maxSSD-g}, and let $s_p^r=s_p^r(\bm{x})$ for abbreviation. Then,
\[
\begin{split}
    &\mathbb{E}_q[(s_p^r-s_q^r)f_{r}^*(\bm{x})]-\mathbb{E}_q[(s_p^r-s_q^r)h_{rg_r}^*(\bm{x}^T\bm{g}_r)]\\
    =&\int q(\bm{x})[(s_p^r-s_q^r)^2-(s_p^r-s_q^r)h_{rg_r}^*(\bm{x}^T\bm{g}_r)]d\bm{x}\\
    =&\int q_{G_r}(y_d)[\int q_{G_r}(\bm{y}_{-d}|y_d)(s_p^r(\bm{G}_r^{-1}\bm{y})-s_q^r(\bm{G}_r^{-1}\bm{y}))^2d\bm{y}_{-d}\\
    -&\underbrace{\int q_{G_r}(s_p^r(\bm{G}_r^{-1}\bm{y})-s_q^r(\bm{G}_r^{-1}\bm{y}))d\bm{y}_{-d}}_{h_{rg_r}^*(y_d)}h_{rg_r}^*(y_d)]dy_d\\
    =&\int q_{G_r}(y_d)[\int q_{G_r}(\bm{y}_{-d}|y_d)(f_{r}^*(\bm{G}_{r}^{-1}\bm{y})-h_{rg_r}^*(y_d))^2]d\bm{y}\\
    =&\mathbb{E}_q[(f_{r}^*(\bm{x})-h_{rg_r}^*(\bm{x}^T\bm{g}_r))^2]\geq 0
\end{split}
\]
where the $3^{\text{rd}}$ equality is due to the fact that $h_{rg_r}^*$ is the conditional mean of $f_{r}^*$. Thus, 
\[
\begin{split}
    &PSD-S_{g_r}\\
    =&\sum_{\bm{r}\in O_r}\mathbb{E}_q[(s_p^r-s_q^r)f_{r}^*(\bm{x})]-\mathbb{E}_q[(s_p^r-s_q^r)h_{rg_r}^*(\bm{x}^T\bm{g}_r)]\\
    =&\sum_{\bm{r}\in O_r}\mathbb{E}_q[(f_r^*(\bm{x})-h_{rg_r}^*(\bm{x}^T\bm{g}_r))]\geq 0
\end{split}
\]
\end{proof}
\subsection{Proof of Theorem \ref{thm: Upper bound controlled approximation}}
Before we give the details, we introduce the main inequality and its variant for the proof.
\begin{lemma}[Poincar\'e Inequality]
For a probabilistic distribution $p$ that satisfies assumption $7$, for all locally Lipschitz function $f(\bm{x}): \mathcal{X}\subseteq\mathbb{R}^{D} \rightarrow \mathbb{R},$ we have the following inequality
$$
\operatorname{Var}_{p}(f(\boldsymbol{x})) \leq C_{p} \int p(\boldsymbol{x})|| \nabla_{\boldsymbol{x}} f(\boldsymbol{x}) \|^{2} d \boldsymbol{x}
$$
where $C_{p}$ is called Poincaré constant that is only related to $p$.
\label{lemma: Poincare Inequality}
\end{lemma}
One should note that the assumption of log concavity of $p$ is a sufficient condition for Poincar\'e inequality, which means it may be applied to a broader class of distributions. But it is beyond the scope of this work. 

Due to the form of optimal test functions of \textit{SSD-g}, we need to deal with the transformed distribution $q_{G_r}$ and its conditional expectations (see Eq.\ref{eq: optimal test function maxSSD-g}). Unfortunately, the original form of Poincar\'e inequality cannot be applied. In the following, we introduce its variant called \textit{subspace} Poincar\'e inequality \cite{constantine2014active,zahm2020gradient,parente2020generalized} to deal with the conditional expectation. But before that, we need to make sure the transformed distribution and its conditional density still satisfy the conditions of Poincar\'e inequality, i.e. log concavity.
\begin{lemma}[Preservation of log concavity]
Assume distribution $q(\boldsymbol{x})=$ $\exp (-V(\boldsymbol{x}))$ is log-concave. With arbitrary orthogonal matrix $\boldsymbol{G}$ and corresponding transformed distribution $q_{G},$ the conditional distribution $q_{G}\left(\boldsymbol{y}_{-d} \mid y_{d}\right)$ is also log-concave for all $d=1, \ldots, D .$
\label{lemma: preservation of log concavity}
\end{lemma}
\begin{proof}
Assume we have $\boldsymbol{y}=\boldsymbol{G} \boldsymbol{x}$. Thus, by change of variable formula, $q_{G}(\boldsymbol{y})=$ $q(\boldsymbol{x})=q\left(\boldsymbol{G}^{-1} \boldsymbol{y}\right)=\exp \left(-V\left(\boldsymbol{G}^{-1} \boldsymbol{y}\right)\right) .$ Thus, the log conditional distribution
$$
\log q_{G}\left(\boldsymbol{y}_{-d} \mid y_{d}\right)=-V\left(\boldsymbol{G}^{-1} \boldsymbol{y}\right)-\log q_{G}\left(y_{d}\right)
$$
We inspect its Hessian w.r.t $\boldsymbol{y}_{-d}$
$$
\begin{aligned}
& \nabla_{\boldsymbol{y}_{-d}}^{2}\left(V\left(\boldsymbol{G}^{-1} \boldsymbol{y}\right)+\log q_{G}\left(y_{d}\right)\right) \\
=& \nabla_{\boldsymbol{y}_{-d}}^{2}\left(V\left(\boldsymbol{G}^{-1} \boldsymbol{y}\right)\right) \\
=& \nabla_{\boldsymbol{y}_{-d}}\left(\boldsymbol{G}_{\backslash d} V^{\prime}\left(\boldsymbol{G}^{-1} \boldsymbol{y}\right)\right) \\
=& \boldsymbol{G}_{\backslash d} V^{\prime \prime}\left(\boldsymbol{G}^{-1} \boldsymbol{y}\right) \boldsymbol{G}_{\backslash d}^{T}
\end{aligned}
$$
where $\boldsymbol{G}_{\backslash d}=\left[\boldsymbol{g}_{1}, \ldots, \boldsymbol{g}_{d-1}, \boldsymbol{g}_{d+1}, \ldots, \boldsymbol{g}_{D}\right]^{T}$ and $V^{\prime}\left(\boldsymbol{G}^{-1} \boldsymbol{y}\right)=\nabla_{\boldsymbol{G}^{-1} \boldsymbol{y}} V\left(\boldsymbol{G}^{-1} \boldsymbol{y}\right)$.
We already know that $V(\cdot)$ is a convex function. Thus, for all $\boldsymbol{u} \in \mathbb{R}^{D}$, $\boldsymbol{u}^{T} V^{\prime \prime}(\boldsymbol{x}) \boldsymbol{u} \geq 0,$ therefore,
$$
\boldsymbol{u}^{T} \boldsymbol{G}_{\backslash d} V^{\prime \prime}\left(\boldsymbol{G}^{-1} \boldsymbol{y}\right) \boldsymbol{G}_{\backslash d}^{T} \boldsymbol{u}=\boldsymbol{l}^{T} V^{\prime \prime}\left(\boldsymbol{G}^{-1} \boldsymbol{y}\right) \boldsymbol{l} \geq 0
$$
where $\boldsymbol{l}=\boldsymbol{G}_{\backslash d}^{T} \boldsymbol{u} .$
\end{proof}
Now, we can introduce the subspace Poincar\'e inequality
\begin{lemma}[Poincar\'e inequality for conditional expectation]
Assume the assumption 2,4 (density regularity), 7 (Poincar\'e inequality condition) are satisfied, with arbitrary orthogonal matrix $\boldsymbol{G}, \boldsymbol{y}=\boldsymbol{G} \boldsymbol{x}$ and $y_{d}=\boldsymbol{x}^{T} \boldsymbol{g},$ we have the following inequality
\[
\begin{split}
&\int q_{G}\left(\boldsymbol{y}_{-d} \mid y_{d}\right)\left[f_{r}^{*}\left(\boldsymbol{G}^{-1} \boldsymbol{y}\right)-h_{rg_r}^{*}\left(y_{d}\right)\right]^{2} d \boldsymbol{y}_{-d} \\
\leq& C_{y d} \mathbb{E}_{q_{G}\left(\boldsymbol{y}_{-d} \mid y_{d}\right)}\left[\left\|\boldsymbol{G}_{\backslash d} \nabla f_{r}^{*}\right\|^{2}\right] \end{split}
\]
where $C_{y_{d}}$ is the Poincaré constant, $\boldsymbol{G}_{\backslash d}=\left[\boldsymbol{a}_{1}, \ldots, \boldsymbol{a}_{d-1}, \boldsymbol{a}_{d+1}, \ldots, \boldsymbol{a}_{D}\right]^{T}$ is the orthogonal matrix $\bm{G}$ excluding $\bm{a}_d=\bm{g}$ and $f_{r}^*$, $h_{rg_r}^*$ are the optimal test functions defined in proposition \ref{prop: optimality of PSD}, \ref{prop: optimality of maxSSD-g} respectively with coefficient 1. 
\label{lemma: subspace poincare inequality}
\end{lemma}
\begin{proof}
From lemma \ref{lemma: preservation of log concavity}, we know $q_{G}\left(\boldsymbol{y}_{-d} \mid y_{d}\right)$ is a log-concave distribution. Therefore, it satisfies the Poincaré inequality (lemma.\ref{lemma: Poincare Inequality}). We have
\[
\begin{aligned}
& \int q_{G}\left(\boldsymbol{y}_{-d} \mid y_{d}\right)\left[f_{r}^{*}\left(\boldsymbol{G}^{-1} \boldsymbol{y}\right)-h^*_{rg_r}(y_d)\right]^{2} d \boldsymbol{y}_{-d} \\
=& Var_{q_G(\bm{y}_{-d}|y_d)}(f_{r}^*(\bm{G}^{-1}\bm{y})) \\
\leq & C_{y_{d}} \int q_{G}\left(\boldsymbol{y}_{-d} \mid y_{d}\right)\left\|\nabla_{\boldsymbol{y}-d} f_{r}^{*}\left(\boldsymbol{G}^{-1} \boldsymbol{y}\right)\right\|^{2} d \boldsymbol{y}_{-d} \\
=& C_{y_{d}} \int q_{G}\left(\boldsymbol{y}_{-d} \mid y_{d}\right)\left\|\boldsymbol{G}_{\backslash d} \nabla_{\boldsymbol{G}^{-1} \boldsymbol{y}} f_{r}^{*}\left(\boldsymbol{G}^{-1} \boldsymbol{y}\right)\right\|^{2} d \boldsymbol{y}_{-d}
\end{aligned}
\]
The first equality comes from the fact that $h^*_{rg_r}(y_d)$ is actually a conditional mean of $f_{r}^*(\bm{G}^{-1}\bm{y})$, and the inequality comes from the direct application of Poincar\'e inequality on $q_{G}(\bm{y}_{-d}|y_d)$ and $f_{r}^*(\bm{G}^{-1}\bm{y})$.
\end{proof}
With the above tools, it is now easy to prove theorem \ref{thm: Upper bound controlled approximation}.

\textbf{Theorem \ref{thm: Upper bound controlled approximation}}
\begin{proof}
We can re-write the inner part of controlled approximation (Eq.\ref{eq: controlled approximation}) in the following:
\[
\begin{aligned}
& \int q_{G_r}\left(y_{d}, \boldsymbol{y}_{-d}\right)\left[f_{r}^{*}\left(\boldsymbol{G}_r^{-1} \boldsymbol{y}\right)-h_{rg_r}^{*}\left(y_{d}\right)\right]^{2} d \boldsymbol{y} \\
=&\int q_{G_r}(y_d)\mathbb{E}_{q_{G_r}(\bm{y}_{-d}|y_d)}[(f_r^*(\bm{G}_r^{-1}\bm{y})-h_{rg_r}(y_d))^2]d\bm{y}\\
\leq & \int q_{G_r}\left(y_{d}, \boldsymbol{y}_{-d}\right) C_{y_{d}}\left\|\boldsymbol{G}_{r\backslash d} \nabla f_{r}^{*}\right\|^{2} d \boldsymbol{y} \\
\leq & C_{sup} \int q_{G_r}\left(y_{d}, \boldsymbol{y}_{-d}\right)\left\|\boldsymbol{G}_{r\backslash d} \nabla f_{r}^{*}\right\|^{2} d \boldsymbol{y} \\
=& C_{s u p} \int q(\boldsymbol{x}) \operatorname{tr}\left[\left(\boldsymbol{G}_{r\backslash d} \nabla f_{r}^{*}\right)\left(\boldsymbol{G}_{r\backslash d} \nabla f_{r}^{*}\right)^{T}\right] d \boldsymbol{x} \\
=& C_{s u p} \operatorname{tr}\left[\boldsymbol{G}_{r\backslash d} \boldsymbol{H}_r \boldsymbol{G}_{r\backslash d}^{T}\right]
\end{aligned}
\]
where the first inequality is directly from lemma \ref{lemma: subspace poincare inequality} and the second inequality is from the definition of $C_{sup}$. 

To minimize this upper bound, we can directly use the theorem 2.1 \cite{sameh2000trace} by setting $B=I$ and $\boldsymbol{X}=\boldsymbol{G}_{r\backslash d}^{T}$. Therefore, we only need to check if $\boldsymbol{G}_{r\backslash d} \boldsymbol{G}_{r\backslash d}^{T}=\boldsymbol{I}$. This is trivial
as $\bm{G}_r$ is an orthogonal matrix. Thus, the proof is complete.
\end{proof}

\section{Theory related to active slice $\bm{r}$}
\label{Appendix: theory related to r}
\subsection{Proof of proposition \ref{prop: lower bound active slice r}}
First, from the theorem \ref{thm: controlled approximation}, we have
\[
\text{PSD}_{r}\geq S_{{r,g}}
\]
Thus, we can establish the following lower bound
\[
\begin{split}
    &S_{{r_1,g_{r_1}}}-S_{{r_2,g_{r_2}}}\\
    \geq&S_{{r_1,g_{r_1}}}-\text{PSD}_{r_2}\\
    =&\underbrace{S_{{r_1,g_{r_1}}}-\text{PSD}_{r_1}}_{\text{controlled approximation}}+\text{PSD}_{r_1}-\text{PSD}_{r_2}
\end{split}
\]
Thus, from theorem \ref{thm: Upper bound controlled approximation}, we can obtain
\[
\begin{split}
    &S_{{r_1,g_{r_1}}}-\text{PSD}_{r_1}\\
    =&-\mathbb{E}_{q}\left[(f_{r_1}^*(\bm{x})-h^*_{r_1g_{r_1}}(\bm{x}^T\bm{g}_{r_1})^2)\right]\\
    \geq& -C_{\text{sup}}\text{tr}(\bm{G}_{r_1\backslash d}\bm{H}_{r_1}\bm{G}_{r_1\backslash d}^T)\\
    =&-C_{\text{sup}}\text{tr}(\bm{H})+\underbrace{\bm{g}_{r_1}^T\bm{H}_{r_1}\bm{g}_{r_1}}_{\geq 0}\\
    \geq& -C_{\text{sup}}\text{tr}(\bm{H}_{r_1})
\end{split}
\]
where the first inequality is from the upper bound of controlled approximation (theorem \ref{thm: Upper bound controlled approximation}) and $\bm{g}_{r_1}^T\bm{H}_{r_1}\bm{g}_{r_1}\geq 0$ is due to the positive semi-definiteness of $\bm{H}_{r_1}$. Assume we have an orthogonal basis $O_{r_1}$ that contains $\bm{r}_1$, thus, for each $\bm{r}\in O_{r_1}$, we have $\text{tr}(\bm{H}_{r})\geq 0$. Then, we can show
\[
\begin{split}
    \text{tr}(\bm{H}_{r_1})&\leq\sum_{\bm{r}\in O_{r_1}}\text{tr}(\bm{H}_{r})\\
    &=\sum_{\bm{r}\in O_{r_1}}\text{tr}(\mathbb{E}_{q}[\nabla_{\bm{x}}\bm{f}^*(\bm{x})\bm{r}\bm{r}^T\nabla_{\bm{x}}\bm{f}^*(\bm{x})^T])\\
    &=\text{tr}(\mathbb{E}_q[\nabla_{\bm{x}}\bm{f}^*(\bm{x})\sum_{\bm{r}\in O_{r_1}}{\bm{r}\bm{r}^T}\nabla_{\bm{x}}\bm{f}^*(\bm{x})^T])\\
    &=\text{tr}(\mathbb{E}_{q}[\nabla_{\bm{x}}\bm{f}^*(\bm{x})\nabla_{\bm{x}}\bm{f}^*(\bm{x})^T])\\
    &=\sum_{i=1}^D\omega_i=\Omega
\end{split}
\]
where $\{\omega_i\}_{i}^D$ are the eigenvalues of $\mathbb{E}_{q}[\nabla_{\bm{x}}\bm{f}^*(\bm{x})\nabla_{\bm{x}}\bm{f}^*(\bm{x})^T]$, $\bm{f}^*(\bm{x})=\bm{s}_p(\bm{x})-\bm{s}_q(\bm{x})$ and $\sum_{\bm{r}\in O_{r_1}}\bm{r}\bm{r}^T=\bm{I}$ since $\bm{r}\in O_{r_1}$ are orthogonal to each other.

Thus, we can substitute it back, we have
\[
S_{{r_1,g_{r_1}}}-S_{{r_2,g_{r_2}}}\geq \text{PSD}_{r_1}-\text{PSD}_{r_2}-C_{\text{sup}}\Omega
\]
\subsection{Proof of theorem \ref{thm: active slice r}}
\begin{proof}
From proposition \ref{prop: optimality of PSD} we know $f_{r}^{*}(\bm{x}) =\left(s_{p}^{r}(x)-s_{q}^{r}(x)\right),$ thus, we can substitute into \textit{PSD} (Eq.\ref{eq: PSD}), we get
$$
\text{PSD}_{r}=\max _{\boldsymbol{r} \in \mathbb{S}^{D}-1} \mathbb{E}_{q}\left[\left(\left(\boldsymbol{s}_{p}(\boldsymbol{x})-\boldsymbol{s}_{q}(\boldsymbol{x})\right)^{T} \boldsymbol{r}\right)^{2}\right]
$$
To maximize it, we consider the following constraint optimization problem.
$$
\max _{\bm{r}} \mathbb{E}_{q}\left[\left(\left(\bm{s}_{p}(x)-\bm{s}_{q}(x)\right)^{T} \bm{r}\right)^{2}\right] \quad \text { s.t. } \quad\|\bm{r}\|^{2}=1
$$
We take the derivative of the corresponding Lagrange multiplier w.r.t. $\boldsymbol{r}$,
$$
\begin{aligned}
& \mathbb{E}_{q}\left[\nabla_{\bm{r}}\left(\left(\boldsymbol{s}_{p}(\boldsymbol{x})-\boldsymbol{s}_{q}(\boldsymbol{x})\right)^{T} \boldsymbol{r}\right)^{2}\right]-2 \lambda \boldsymbol{r}=0 \\
\Rightarrow & \mathbb{E}_{q}\left[\left(\boldsymbol{s}_{p}(\boldsymbol{x})-\boldsymbol{s}_{q}(\boldsymbol{x})\right)^{T} \boldsymbol{r}\left(\boldsymbol{s}_{p}(\boldsymbol{x})-\boldsymbol{s}_{q}(\boldsymbol{x})\right)\right]=\lambda \boldsymbol{r} \\
\Rightarrow & \underbrace{\mathbb{E}_{q}\left[\left(\boldsymbol{s}_{p}(\boldsymbol{x})-\boldsymbol{s}_{q}(\boldsymbol{x})\right)\left(\boldsymbol{s}_{p}(\boldsymbol{x})-\boldsymbol{s}_{q}(\boldsymbol{x})\right)^{T}\right]}_{\boldsymbol{S}=\mathbb{E}_q[\bm{f}^*(\bm{x})\bm{f}^*(\bm{x})^T]} \boldsymbol{r}=\lambda \boldsymbol{r}\\
\Rightarrow&\bm{S}\bm{r}=\lambda\bm{r}
\end{aligned}
$$
This exactly the problem of finding eigenpair for matrix $\bm{S}$. Let's assume $\bm{r}=\bm{v}$ which is the eigenvector of $\bm{S}$ with corresponding eigenvalue $\lambda$. Substituting it back to \textit{PSD}, we have
\[
\begin{split}
&\mathbb{E}_{q}\left[\left(\left(\boldsymbol{s}_{p}(\boldsymbol{x})-\boldsymbol{s}_{q}(\boldsymbol{x})\right)^{T} \boldsymbol{r}\right)^{2}\right] \\
=&\mathbb{E}_{q}\left[\left(\boldsymbol{s}_{p}(\boldsymbol{x})-\boldsymbol{s}_{q}(\boldsymbol{x})\right)^{T} \boldsymbol{r}\left(\boldsymbol{s}_{p}(\boldsymbol{x})-\boldsymbol{s}_{q}(\boldsymbol{x})\right)\right]^{T} \boldsymbol{r} \\
=&\bm{r}^T\bm{S}\bm{r} \\
=&\lambda \boldsymbol{v}^{T} \boldsymbol{v}=\lambda
\end{split}
\]
Thus, to obtain the active slice $\bm{r}$, we only need to find the eigenvector of $\bm{S}$ with the largest eigenvalue. 
\end{proof}
\subsection{Greedy algorithm is eigen-decomposition}
\label{subsec: greedy algorithm}
\begin{corollary}[Greedy algorithm is eigen-decomposition]
Assume the conditions in theorem \ref{thm: active slice r} are satisfied, then finding the orthogonal basis $O_r$ from the greedy algorithm is equivalent to the eigen-decomposition of $\bm{S}$. 
\label{coro: greedy algorithm eigendecomposition}
\end{corollary}
\begin{proof}
Assume we have obtained the active slice $\bm{r}$ from theorem \ref{thm: active slice r}, thus, we have $\bm{S}\bm{r}=\lambda\bm{r}$.
The greedy algorithm for $\bm{r}'$ can be translated into the following constrained optimization 
\[
\begin{split}
    &\max_{\bm{r}'}\quad \mathbb{E}_{q}\left[\left(\left(\bm{s}_{p}(x)-\bm{s}_{q}(x)\right)^{T} \bm{r}'\right)^{2}\right]\\
    s.t.\quad& ||\bm{r}'||^2=1\\
    &\bm{r}^T\bm{r}'=0
\end{split}
\]
By using Lagrange multipliers ($\mu$, $\gamma$), and then take derivative w.r.t. $\bm{r}'$,
\[
\bm{S}\bm{r}'=\mu\bm{r}'+\gamma \bm{r}
\]
Then taking the inner product with $\bm{r}$ in both side, and notice $\bm{S}$ is a symmetric matrix, we obtain
\[
\begin{split}
\gamma&=\langle\bm{S}\bm{r}',\bm{r}\rangle\\
&=\langle\bm{r}',\bm{S}^T\bm{r}\rangle\\
&=\langle\bm{r}',\lambda\bm{r}\rangle=0
\end{split}
\]
Therefore, the constrained optimization is the same as the one in theorem \ref{thm: active slice r}, which is to find a eigenvector of $\bm{S}$ that is different from $\bm{r}$. Repeat the above procedure, the final resulting $O_r$ is a group of eigenvectors of $\bm{S}$. 
\end{proof}

\section{Experiment Details}
\label{appendix: Experiment Details}
\begin{figure}
    \centering
    \includegraphics[scale=0.19]{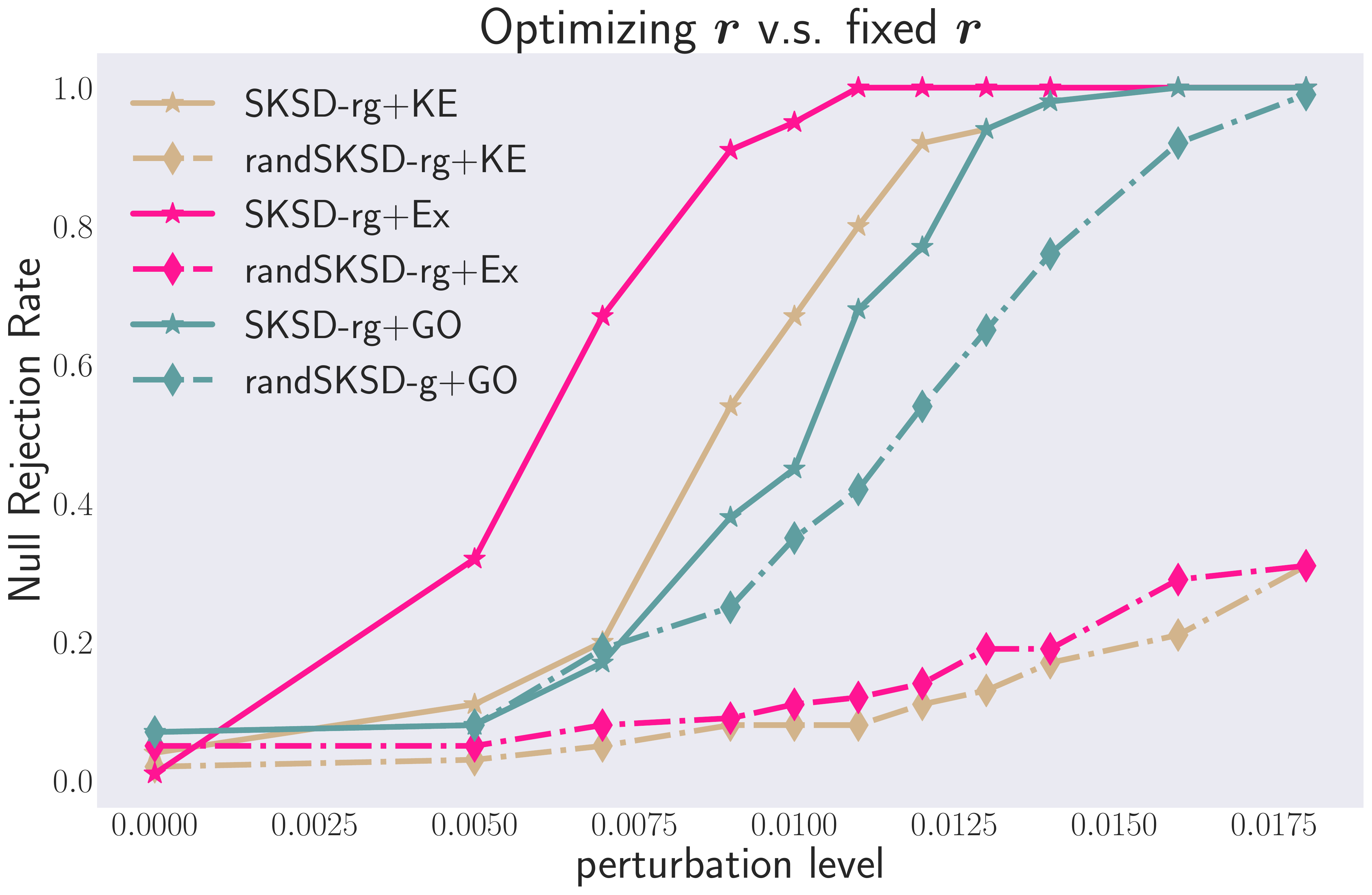}
  \caption{The test power difference with good $\bm{r}$ and fixed $\bm{r}$.}
    \label{fig: optimzie R}
\end{figure}
\begin{figure*}[t]
    \centering
    \includegraphics[scale=0.14]{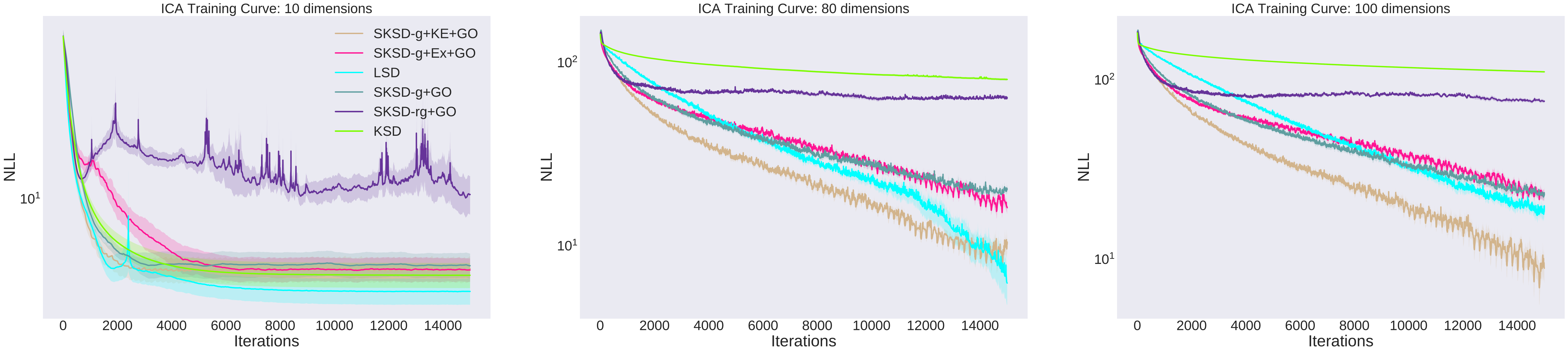}
    \caption{Training curve of ICA model with test NLL for different dimensions. }
    \label{fig: Remaining ICA curve}
\end{figure*}
\begin{figure}
    \centering
    \includegraphics[scale=0.19]{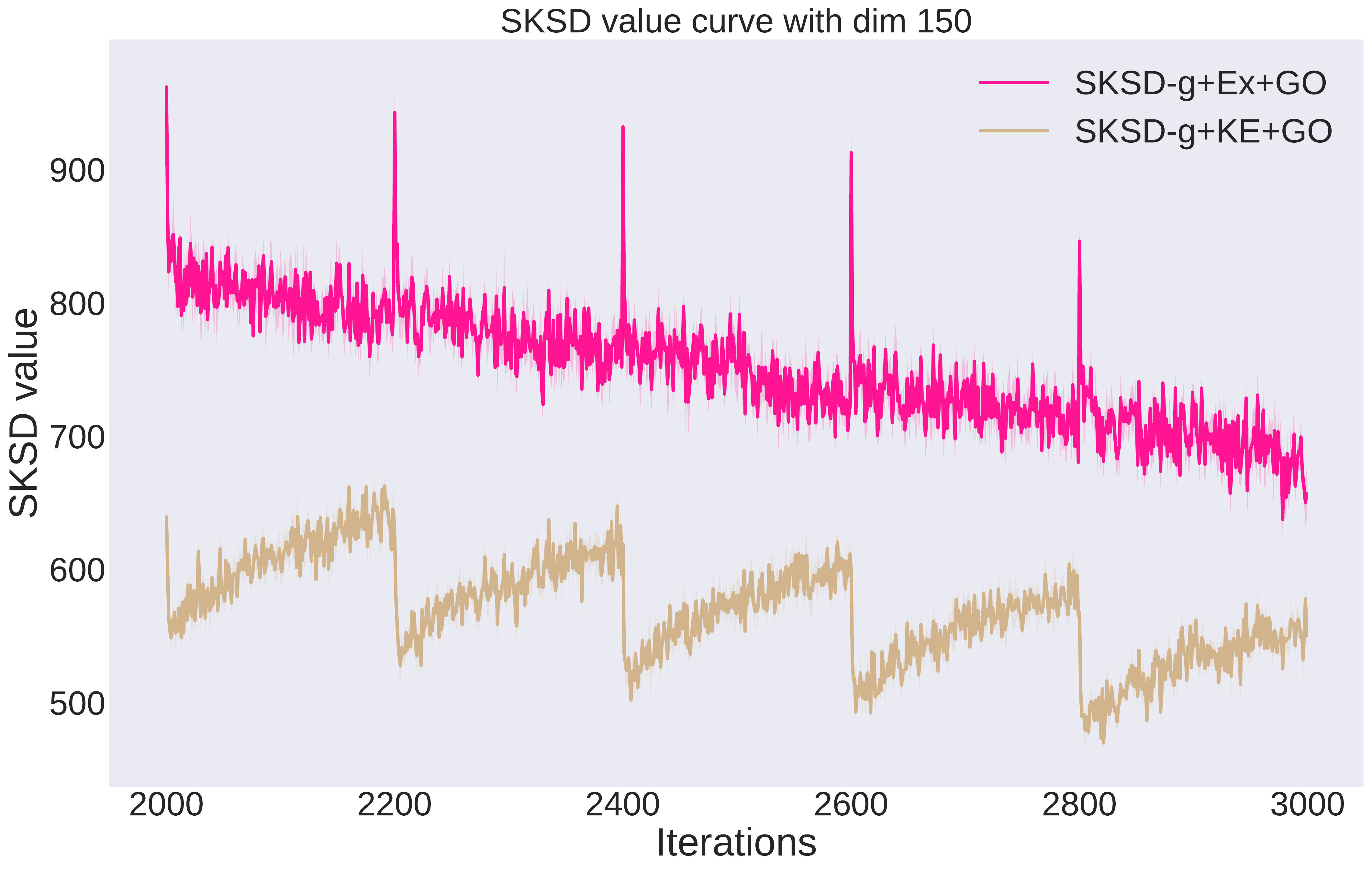}
    \caption{SKSD value curve with 150 dimensional ICA during iteration 5000 to 6000}
    \label{fig:ICA Value Curve}
\end{figure}
For all experiments in this paper, we use RBF kernel with median heuristics. 
\subsection{Benchmark GOF test}
\label{appendix subsec: benchmark GOF test}
% \begin{table}[]
% \centering
% \caption{Test power of 100 dimensional Multivariate-t GOF test}
% \begin{tabular}{l|lll}
% \hline
%              & \multicolumn{3}{c}{Multi-t}      \\ \hline
% Method       & Test Power & sec/trial & Speed-up \\ \hline
% SKSD-g+Ex    & 1          & 0.49      & 90x      \\ \hline
% SKSD-g+GO    & 0.67       & 44.24     & 1x       \\
% SKSD-g+KE+GO & \textbf{0.97}       & 2.38      & 19x      \\
% SKSD-g+GE    & \textbf{1}          & 0.67      & \textbf{66x}     \\ \hline
% \end{tabular}
% \label{tab: Multivarite-t GOF test}
% \end{table}
For gradient based optimization, we use Adam \citep{kingma2014adam} with learning rate $0.001$ and $\beta=(0.9,0.99)$. We use random initialization for SKSD-g+GO by drawing $\bm{g}_r$ from a Gaussian distribution before normalizing them to unit vectors. For kernel smooth and gradient estimator, we use RBF kernel with median heuristics. Although the algorithm \ref{alg: active slice algorithm} states that small Gaussian noise are needed for active slices, in practice, we found that active slices still have the satisfactory performance without the noise. 

The significance level for GOF test $\alpha=0.05$, and the dimensions of the benchmark problems grow from $2$ to 100. We use $1000$ bootstrap samples to estimate the threshold and run $100$ trials for each benchmark problems.
\subsection{RBM GOF test}
\label{appendix subsec: RBM experiment}
We set significance level $\alpha=0.05$ and use $1000$ bootstrap samples to compute the threshold. For methods that require training (\textit{SKSD} based method), we need to collect some training samples. Following the same settings as \cite{gong2021sliced}, to avoid over-fitting to small training set, we collect the pseudo-samples during the early burn-in stage. Note that these pseudo-samples should not be used for testing, as they are not drawn from the $q$. We collect $2000$ samples. For gradient based optimization, we use the same optimizer as benchmark GOF test with the same hyper-parameters. The batch size is 100. For initialization of \textit{SKSD+GO}, we found that if the slices are initialized randomly, the gradient optimization fails to find meaningful slices within a reasonable amount of time, therefore, we have initialize the $\bm{r}$ and $\bm{g}_r$ as one-hot vectors and set $\bm{r}=\bm{g}_r$. For pruning ablation study, if the pruning level is set to $50$, we initialize $\bm{r}$ and $\bm{g}_r$ to be the identity matrix. The default number of gradient optimization for \textit{SKSD+GO} is $50$. For active slice method, we directly use the active slices without any further optimizations. We run 100 trials for GOF test with 1000 test samples per trial.

\cite{gong2021sliced} reports \textit{SKSD-rg+GO} has near optimal test power at perturbation level $0.01$. The performance difference is because they train the \textit{SKSD-rg} with $200$ batch sizes per burn-in step. Namely, the training set size are $200\times 2000=400000$, which is 200 times larger than ours. They also run $2000$ iterations, which is equivalent to $100$ epochs in our settings.

% \begin{figure}
%     \centering
%     \includegraphics[scale=0.19]{Experiments/Fig/OptimizeR.pdf}
%   \caption{The test power difference with good $\bm{r}$ and fixed $\bm{r}$.}
%     \label{fig: optimzie R}
% \end{figure}
% \begin{figure*}[t]
%     \centering
%     \includegraphics[scale=0.14]{Experiments/Fig/Remaining_Training_Curve.pdf}
%     \caption{Training curve of ICA model with test NLL for different dimensions. }
%     \label{fig: Remaining ICA curve}
% \end{figure*}
% \begin{figure}
%     \centering
%     \includegraphics[scale=0.19]{Experiments/Fig/Value_Curve.pdf}
%     \caption{SKSD value curve with 150 dimensional ICA during iteration 5000 to 6000}
%     \label{fig:ICA Value Curve}
% \end{figure}
Figure \ref{fig: optimzie R} shows the test power difference with optimized $\bm{r}$ and fixed $\bm{r}$. The legend with \emph{rand} annotation implies we randomly initialized $\bm{r}$ as one-hot vectors and fix them while updating $\bm{g}_r$ using \textit{GO} or active slice. Without \emph{rand}, it means both $\bm{r}$ and $\bm{g}_r$ are optimized. We only use 3 $\bm{r}$ for active slice method and 50 for gradient-based counterpart. For active slice method with pruning (\textit{randSKSD-g+Ex} or \textit{randSKSD-g+KE}), despite we show that any finite random slices define a valid discrepancy, it is clear that the performance is quite poor with random initialized $\bm{r}$'s. It indicates that using active slices of $\bm{g}_r$ alone cannot compensate the poor discriminating power of the random $\bm{r}$'s. Although \textit{SKSD-rg+GO} demonstrates an advantage compared to \textit{randSKSD-g+GO}, the performance boost is less clear compared to active slices method. This is because we do not use any pruning for \textit{randSKSD-g+GO}, and adopt orthogonal basis $O_r=\bm{I}$. Despite the orthogonal basis may not capture the important directions, they can provide reasonable discriminating power due to their orthogonality from each other. In summary, using good directions for $\bm{r}$ is advantageous compared to fixed $\bm{r}$. 
\subsection{Model learning: Training ICA}
\label{appendix subsec: model learning ICA}
We use Adam optimizer for the model and slice directions with learning rate $0.001$ and $\beta=(0.9,0.99)$. We totally run $15000$ iterations. The batch size is 100. We evaluate our method in dimension 10, 80, 100 and 150. For more stable comparisons, we initialized the weight matrix $\bm{W}$ until its conditional number is smaller than its dimensions. For active slice method, we use randomly sampled $3000$ data from training set to estimate the score difference and the matrices used for eigen-decomposition. 

For \textit{SKSD-rg+GO}, we initialize the $\bm{r}$ to be a group of one-hot vectors to form identity matrix and $\bm{g}_r=\bm{r}$. We use an adversarial training procedure that updates both $\bm{r}$ and $\bm{g}_r$ using Adam once per iteration before we update the model. For \textit{SKSD-g+GO}, we fix the orthogonal basis $O_r$ to be the identity matrix and only update $\bm{g}_r$. Each results are the average of 5 runs of training. 

As for the reason why \textit{SKSD-g+Ex+GO} performs worse than \textit{+KE+GO}, we suspect that \textit{+Ex} only focus on directions with high discriminating power. However, high discriminating power is not necessarily good for model learning. It may focus on very small area that is different from the target but ignore the larger area with small difference. Because our algorithm for finding basis is greedy, this means it can ignore the generally good directions if they are not orthogonal to the directions with high discriminating power. 

From figure \ref{fig:ICA Value Curve}, we can observe there is a spike of \textit{SKSD-g+Ex+GO} value at every 200 iterations due to the new active slices found at the beginning of each training epoch. However, the value drops significantly fast to the one before new active slices. This indicates the \textit{Ex} indeed finds directions with large discriminating power but they do not represents good directions for learning due to the fast drop of SKSD values. On the other hand, the directions provided by KE does not give the highest discriminating power, but it can find generally good directions of $\bm{g}_r$ using \textit{GO} refinement steps within a few iterations. This means the directions found by \textit{KE} indeed represents good directions for learning as the model cannot decrease this value quickly. We guess this is due to the smooth estimation of KE, where very small areas with high discriminating power are smoothed out. 

Figure \ref{fig: Remaining ICA curve} shows the ICA training curve of other dimensions. We can observe the convergence speed of LSD deteriorates as the dimension increases due to the poor test function in early training stage, whereas \textit{SKSD-g+KE+GO} maintains the fastest convergence in high dimensions.

\section{Perturbation of eigenvectors}
\label{Appendix: Perturb}

The active slice method (algorithm \ref{alg: active slice algorithm}) is mainly based on the eigenvalue-decomposition of matrix $\pmb{H}$, where 
\begin{equation*}
    \pmb{H}=\int q(\pmb{x})\nabla_{\pmb{x}}f_r^*(\pmb{x})\nabla_{\pmb{x}}f_r^*(\pmb{x})^Td\pmb{x}
\end{equation*}

Obtaining the analytic form of $\pmb{H}$ involves complicated integration, so Monte Carlo estimation is often used for approximation. We denote it as $\hat{\pmb{H}}$, with M being the number of samples:
\begin{equation}
    \hat{\pmb{H}} = \frac{1}{M} \sum_{i=1}^M [\nabla_{\pmb{x}_i}f_r^*(\pmb{x}_i)\nabla_{\pmb{x}_i}f_r^*(\pmb{x}_i)^T]
\end{equation}

Let $\pmb{g}$ be the top eigenvector of $\pmb{H}$ and $\hat{\pmb{g}}$ be the top eigenvector of $\hat{\pmb{H}}$. Let $\lambda_1$, $\lambda_2$ be the top two eigenvalues of $\pmb{H}$. Assuming the error matrix $\pmb{E}=\hat{\pmb{H}} - \pmb{H}$ is deterministic, \citep{yu2015useful} proved that 
\begin{equation}
    ||\pmb{g}\pmb{g}^T(\pmb{I}-\hat{\pmb{g}}\hat{\pmb{g}}^T)||_F \leq \frac{2||\pmb{E}||_{op}}{\lambda_1 - \lambda_2}
\label{eq: eigenvector bound}
\end{equation}
where we define the operator norm for a given $n \times n$ matrix $\pmb{A}$ as
\begin{equation*}
    ||\pmb{A}||_{op} = \sup \{||\pmb{A}\pmb{x}|| : x \in \mathbb{R}^n \ \text{with} \ ||\pmb{x}|| = 1 \}
\end{equation*}
We also have (with proof below)
\begin{equation}
    \min_{\epsilon \in \{ -1, 1 \}} ||\pmb{g} - \epsilon \hat{\pmb{g}} ||_2 \leq \sqrt{2} ||\pmb{g}\pmb{g}^T(\pmb{I}-\hat{\pmb{g}}\hat{\pmb{g}}^T)||_F
\label{eq: sin dist bound}
\end{equation}

Inequality \ref{eq: eigenvector bound} and \ref{eq: sin dist bound} imply that,
\begin{equation}
    \min_{\epsilon \in \{ -1, 1 \}} ||\pmb{g} - \epsilon \hat{\pmb{g}} ||_2 \leq 2^{3/2} \frac{||\hat{\pmb{H}} - \pmb{H}||_{op}}{\lambda_1 - \lambda_2}
    \label{eq: eigenvector distance}
\end{equation}

\subsection{Proof of inequality \ref{eq: sin dist bound}}
\begin{prop}
    Let $\bm{S}$ and $\bm{U}$ be two matrices with orthonormal columns and equal rank $r$. Let $\bm{\Pi}_S$ ($resp.$ $\bm{\Pi}_U$) indicates the projection matrix to the column space of $\bm{S}$ ($resp.$ $\bm{U}$). Then
    \begin{equation}
        \min_{O \in \mathbb{R}^{r\times r} \operatorname{orthogonal}} ||\bm{S}-\bm{U}\bm{O}||_F \leq \sqrt{2} ||\bm{\Pi}_S(\bm{I}-\bm{\Pi}_U)||_F
    \end{equation}
\end{prop}
When $r=1$, we denote $\pmb{O}$ as $\epsilon$. Following the definition of orthogonal matrix, we have $\epsilon^T \epsilon = \epsilon^2 = 1$, hence $\epsilon \in \{ -1, 1 \}$. Substituting $\bm{S} = \pmb{g}$ and $\pmb{U} = \hat{\pmb{g}}$, we get inequality \ref{eq: sin dist bound}.
\begin{proof}
Let $\bm{W} \bm{\Sigma} \bm{V}^T$ be a singular value decomposition of $\bm{S}^T \bm{U}$, and use $\bm{O}=\bm{V}\bm{W}^T$.Now,
\begin{align*}
    ||\bm{S}-\bm{U}\bm{O}||^2_F &= \operatorname{Tr} ((\bm{S}-\bm{U}\bm{O})^T (\bm{S}-\bm{U}\bm{O})) \\
    &= ||\bm{S}||^2_F + ||\bm{U}||^2_F - 2\operatorname{Tr} (\bm{O}\bm{S}^T\bm{U}) \\
    &= 2r - 2\operatorname{Tr}(\bm{\Sigma})
\end{align*}
where r is the rank of $\bm{S}$ and $\bm{U}$. On the other hand, by Pythagora's theorem
\begin{align*}
    ||\bm{\Pi}_S(\bm{I}-\bm{\Pi}_U)||_F^2 &= ||\bm{\Pi}_S||_F^2 - ||\bm{\Pi}_S\bm{\Pi}_U||_F^2 \\
    &= r - ||\bm{\Pi}_S\bm{\Pi}_U||_F^2 \\
    &= r - ||\bm{SS}^T \bm{UU}^T||_F^2 \\
    &= r - \operatorname{Tr}(\bm{\Sigma}^2)
\end{align*}
We claim that the entries of $\bm{\Sigma}$ are bounded above by 1, such that $\operatorname{Tr}(\bm{\Sigma}) \leq \operatorname{Tr} (\bm{\Sigma}^2)$, then
\begin{align*}
    \min_{O \in \mathbb{R}^{r\times r} \operatorname{orthogonal}} ||\bm{S}-\bm{U}\bm{O}||_F^2 
    &\leq 2r-2\operatorname{Tr}(\bm{\Sigma}) \\
    &\leq 2r - 2\operatorname{Tr}(\bm{\Sigma}^2) \\
    &= 2||\bm{\Pi}_S(\bm{I}-\bm{\Pi}_U)||_F^2
\end{align*}
Taking the square root of both sides yields the desired inequality. To prove the claim, let $\bm{\omega} = [\bm{S}, \bm{S}']$ and $\Tilde{\bm{U}} = [\bm{U}, \bm{U}']$ be orthogonal matrices. Then $\bm{S}^T \bm{U}$ is a diagonal block in $\bm{\omega}^T \Tilde{\bm{U}}$. It follows that $\max_i \Sigma_{i,i} = ||\bm{S}^T \bm{U}||_{op} \leq ||\bm{\omega}^T \Tilde{\bm{U}}||_{op} = 1$
\end{proof}
From Eq.\ref{eq: eigenvector distance}, we can see if the top two eigenvalues are similar, then their corresponding eigenvectors can be arbitrary different. In terms of our active slice algorithm, it means if the most discriminating directions for two distributions $q$, $p$ have similar "magnitude of difference", our algorithm may fail under Monte-carlo approximation. On the other hand, if the eigenvalues are different, Eq.\ref{eq: eigenvector distance} guarantees that eigenvectors from $\hat{\bm{H}}$ are not far-away from the truth. 

% investigate whether choosing g or -g causes a difference, and whether choosing r and -r causes a difference

% We here insert the assumption that $\lambda_1 - \lambda_2$ is large enough (larger than which value? $\lambda_1 /2$?), so that even for moderate $M$, the approximate top eigenvector $\hat{\pmb{g}}$ is close to the true top eigenvector $\pmb{g}$.

\end{document}